\documentclass[12pt]{article}

\usepackage{graphicx}
\usepackage{amsmath,amsthm,amsfonts,amssymb}
\usepackage{url}
\usepackage{epstopdf, epsfig}
\usepackage{natbib}
\usepackage{algorithm}
\usepackage{algorithmic}
\usepackage{caption}
\usepackage{subcaption}
\usepackage{color}
\usepackage{float}
\usepackage[margin=1.25in]{geometry}
\usepackage[hidelinks]{hyperref}

\newcommand{\calS}{\ensuremath{\mathcal{S}}}

\newcommand{\balpha}{\ensuremath{\boldsymbol \alpha}}
\newcommand{\bbeta}{\ensuremath{\boldsymbol \beta}}

\newcommand{\bmu}{\ensuremath{\boldsymbol \mu}}
\newcommand{\bnu}{\ensuremath{\boldsymbol \nu}}

\newcommand{\btheta}{\ensuremath{\boldsymbol \theta}}

\newcommand{\B}{\ensuremath{\text{Beta}}}
\newcommand{\Dir}{\ensuremath{\text{Dir}}}

\newcommand{\brho}{\ensuremath{\boldsymbol \rho}}

\newcommand{\argmax}{\mathop{\mathrm{arg\,max}{}}}

\newtheorem{theorem}{Theorem}
\newtheorem{proposition}[theorem]{Proposition}

\definecolor{DSgray}{cmyk}{0,1,0,0}

\title{Bayesian Decision Process for Cost-Efficient Dynamic Ranking via Crowdsourcing}

\begin{document}
\bibliographystyle{plainnat}

\author{
        Xi Chen\thanks{The authors are listed in alphabetical order}  \; and Kevin Jiao \\
        Stern School of Business \\
        New York University\\
        New York, New York, 10012, USA \\
        xchen3@stern.nyu.edu; jjiao@stern.nyu.edu \\
         AND \\
        Qihang Lin  \\
        Tippie College of Business\\
       University of Iowa\\
       Iowa City, Iowa, 52242, USA \\
        qihang-lin@uiowa.edu
       }

\date{}
\maketitle

\begin{abstract}
Rank aggregation based on pairwise comparisons over a set of items has a wide range of applications.
Although considerable research has been devoted to the development of rank aggregation algorithms, one basic question is how to efficiently collect a large amount of high-quality pairwise comparisons for the ranking purpose. Because of the advent of many crowdsourcing services, a crowd of workers are often hired to conduct pairwise comparisons with a small monetary reward for each pair they compare. Since different workers have different levels of reliability and different pairs have different levels of ambiguity, it is desirable to wisely allocate the limited budget for comparisons among the pairs of items and workers so that the global ranking can be accurately inferred from the comparison results. To this end, we model the active sampling problem in \emph{crowdsourced ranking} as a Bayesian Markov decision process, which dynamically selects item pairs and workers to improve the ranking accuracy under a budget constraint. We further develop a computationally efficient sampling policy based on knowledge gradient as well as a moment matching technique for posterior approximation. Experimental evaluations on both synthetic and real data show that the proposed policy achieves high ranking accuracy with a lower labeling cost.


\end{abstract}


\section{Introduction}

Inferring the ranking over a set of items, such as documents, images, movies, or URL links, is an important learning problem with many applications in areas like web search, recommendation systems, online games, etc. An interesting problem related to rank inference is estimating a score for each item based on a certain criterion that the items can be ranked, such as the score of relevance or the score of quality. Typically, both the ranking and the scores of items can be inferred from a collection of high-quality labels on the items.
There are mainly two different types of labels. The label of the first type is associated with each individual item in order to characterize the property of the item itself, for example,  a binary or an ordinal score (e.g., 5-point grade). The label of the second type is instead associated with a subset of items that reveal their relative properties, for example, a partial ranking that covers only this subset. Labels of both types can be obtained by soliciting the knowledge of human workers, depending on whether the worker is employed to evaluate a single item or to compare a subset of items according to a given criterion.
In practice, a binary score usually cannot fully distinguish all items and ordinal scores from different workers are often inconsistent due to the difference in their understandings of the grades in the ordinal scoring scheme. Therefore, the second type of labels has been more widely adopted, which can effectively reduce the impact of misunderstanding among workers and is more appropriate for ranking fine-grained items with a large number of graduations (e.g., in our real data experiment on accessing reading difficulty of an article into one of twelve American grade levels). Moreover, empirical evidences show that the ranking accuracy of a human worker typically decreases when he or she has to compare many items at a time. For this reason, in this paper, we only consider the
relative comparisons over \emph{pairs} of items and the label from a human worker indicates which item is preferred to the other. 

The traditional approach of conducting pairwise comparisons by a small group of experts is usually time consuming and expensive. It fails to meet the growing need of labeled data for ranking tasks. Because of the advent of online crowdsourcing services \citep{Howe:06} such as  Amazon Mechanical Turk, a more efficient and more economic approach has emerged: a large amount of unlabeled pairs of items are posted to a crowdsourcing platform, where a crowd of workers are hired to perform pairwise comparisons and provide labels of the assigned pairs. Given the labels from crowd workers, we can infer a global ranking over all items. We refer to the process of collecting pairwise labels and ranking items as \emph{crowdsourced ranking}.

Despite its availability and scalability, challenges remain in crowdsourced ranking. A certain amount of monetary reward is paid to a worker for each pair of items he or she compares while there is usually only a fixed amount of budget available, limiting the total number of pairwise labels we can collect. Hence, there is a need for a budget-efficient decision process for allocating the budget over item pairs and workers.
In particular, on crowdsourcing platforms, 
there are unreliable workers who submit their answers quickly but carelessly in order to obtain more monetary reward with less effort. Hence, the comparison results provided by crowd workers often contain non-negligible noise. As a remedy, multiple workers are hired to compare the same pair of items independently in the hope that the correct ranking can be recovered, and that the unreliable workers can be identified by comparing their answers with the rest of workers. However, each pairwise comparison will incur a pre-specified monetary cost. Without a careful control, such a repetitive labeling strategy often results in too many labels on the same pair by different workers, leading to a high cost. Furthermore, because of the diversity of their backgrounds and expertise, workers do not always agree with each other in the results of pairwise comparisons, especially when the two items in comparison are competitive to each other. We refer to such a competitive pair as an \emph{ambiguous pair} since the ordering of them is more difficult to be determined. Presumably, a greater budget should be spent on ambiguous pairs, but identifying ambiguous pairs under the budget constraint itself is a challenging problem, which requires some effective learning scheme. Given the trade-off between the labeling cost and the quality of ranking results, there are two fundamental challenges in crowdsourced ranking:
\begin{enumerate}
  \item Given the inconsistent pairwise labels from crowd workers with different reliability, how to aggregate these labels into a global ranking over items.
  \item With both unreliable workers and ambiguous pairs initially unidentified, how to incorporate a learning scheme with an efficient sampling procedure (over both pairs of items and workers) under the budget constraint to achieve the highest ranking accuracy.
\end{enumerate}

To address these challenges, we need to first model the reliability of workers and the ambiguity of item pairs and analyze how they influence the pairwise label. To this end, we adopt a combination of the Bradley-Terry-Luce ranking model \citep{Bradley:52,Luce:59} for modeling the comparison results and the Dawid-Skene model \citep{Dawid:79} for workers' reliability. The reason why we adopt the Bradley-Terry-Luce model
is that learning such a model will not only provide a ranking over items but also give a score to each item, which can be useful in many applications (e.g., providing player's rating in chess games).
We measure the quality of the ranking inferred from the collected labels using the \emph{Kendall's tau rank correlation coefficient} (Kendall's tau for short) with respect to the underlying true ranking.

Under such a model and a quality measure, we propose a dynamic sampling and ranking procedure which addresses the aforementioned two challenges in a unified framework. In particular, we first introduce the priors for items' latent true scores and workers' reliability and formulate the crowdsourced ranking problem into a finite-horizon Bayesian \emph{Markov decision problem} (MDP), whose state variables correspond to the posterior distributions given the observed labels. Here, the number of stages is determined by the total budget, i.e., the total number of pairs that can be requested for labeling. As the budget level increases, the size of the state space grows at an exponential rate, which makes the exact solving of such a MDP problem intractable. To address the computational difficulty, we propose an efficient sampling strategy called \emph{approximated knowledge gradient} (AKG) policy based on the  popular knowledge gradient policy \citep{Powell:11a,Frazier:thesis,Frazier:08,Ryzhov:2012}. The proposed policy dynamically chooses the next pair of items and the worker that together lead to a maximum expected improvement in Kendall's tau rank correlation coefficient. Finally, to determine the global ranking that maximizes the expected Kendall's tau, one needs to solve a maximum linear ordering problem \citep{grotschel1984cpa}, which is a NP-hard problem (and in fact, APX-hard (approximable-hard) \citep{Mishra04LOP}). To address this challenge, we propose a moment matching technique to approximate the posteriors in parametric forms so that the linear ordering problem under the approximated posterior can be easily solved by a simple sorting procedure.

The rest of the paper is organized as follows. In Section \ref{sec:related}, we review the related literature. In Section \ref{sec:homo}, we introduce the model and the proposed policy under the simplified case where all workers are homogeneous and perfectly reliable. In Section \ref{sec:hetero}, we extend our policy to the case where the crowd workers have heterogeneous reliability. In Section \ref{sec:exp}, we present numerical results on both simulated and real datasets, followed by conclusions in Section \ref{sec:conclusion}. The detailed proofs and derivations are provided in the appendix.

\section{Related Work}
\label{sec:related}

The dataset of partial rankings over items can be generated from a variety of sources including crowdsourcing services~\citep{Shah15b}, online competition games (e.g., Microsoft's TrueSkill system \citep{Ralf:07}), and online users' activities such as browsing, clicking and transactions that reveal certain preferences. Learning a global ranking of a large set of items by aggregating a collection of partial rankings/preferences has been an active research area for the past ten years~(see, e.g., \cite{Gleich:2011, nos12, Yi:13, Shah:15a,Shah15b,Rajkumar:2014,JMLR:v15:lu14a,JMLR:v15:volkovs14a}).
However, most work on rank aggregation considers a static estimation problem --- inferring a global ranking based on a pre-existing dataset. The problem we consider here is related to but significantly different from these works because we model crowdsourced ranking as a dynamic procedure where the inference of ranking and collection of data proceed concurrently and influence each other.


The crowdsourced ranking problem we considered has a close connection with the dynamic sorting problem using noisy pairwise comparisons, which has been studied by several authors~\citep{Ailon:2012,Braverman:2008,Radinsky:2011,Wauthier13,Jamieson:11}. However, these papers assume the noise of pairwise comparison results has the same distribution for all pairs, which is not reasonable in crowdsourced ranking because workers usually rank significantly different items more correctly than they do for similar items. The approaches proposed by~\cite{Pfeiffer:12} and~\cite{Qian:2015} assume that the labeling noise depends on the latent qualities or features of the items. However, their approaches do not model the reliability of workers in the decision process. In contrast, our approach allows a label's noise to depend not only on the items themselves, but also on the reliability of the worker who provides the label. The ranking model adopted in this paper, which combines the Bradley-Terry-Luce model and the Dawid-Skene model, was originally proposed in \citep{Chen:13}, which also considers a similar problem of Bayesian statistical decision-making for crowdsourced ranking. However, the sampling strategy developed in \cite{Chen:13}, which prioritizes the pair of items and the worker with the highest information gain, is a simple heuristic without a well-defined objective function to be optimized. In contrast, our work chooses the expected Kendall's tau as the objective function to maximize, which guides the development of the knowledge gradient policy.

In addition to crowdsourced ranking, the problem of crowdsourced categorical labeling/classification has been extensively studied in the past five years. Most work aims at solving a static problem, which infers the categorical labels and workers' reliability based on a static problem (see, e.g., \cite{Dawid:79, Vikas:10, Peter:10, Whitehill:09, QiangLiu:12,Gao:13,Zhang:14a}). Recently, some research has been devoted to dynamic sampling in crowdsourced classification \citep{Karger:13, Oh:12,  Thore:12, Rudin:12, Ece:12, Ho:13, Chen:15}. In particular, both \cite{Ece:12} and \cite{Chen:15} utilized the Markov decision process to model the budget allocation (i.e., sampling over items and workers) process. Since we also adopt a Bayesian Markov decision process with a variant of knowledge gradient policy, the spirit of our method is similar to that in \cite{Chen:15}. However, since the statistical model for a ranking problem is fundamentally different from that of a classification problem, the Markov decision process in this paper is significantly different from the one introduced by \cite{Chen:15} in many aspects such as the objective function, stage-wise rewards, transition probabilities, optimal policy, etc. 
For example, the policy by \cite{Chen:15} is designed to maximize the expected classification accuracy while our policy aims at maximizing the expected Kendall's tau with respective to the true ranking. In fact, even for a static problem with a given set of collected data, inferring the ranking with the maximum expected Kendall's tau is equivalent to a NP-hard maximum linear ordering problem while classifying items with a maximum expected accuracy can be done in closed-form by Bayesian decision rule. In this paper, we avoid this computational challenge by exploiting the structure of the expected Kendall's tau and approximating the posteriors using moment matching. We also note that, although one can view the problem of ranking $K$ items as a problem of classifying $K(K-1)/2$ pairs (each pair is treated as an item in \cite{Chen:15}), 
such an approach increases the size of the problem and ignores the dependency between pairwise labels.

In addition, it is worth to note that the problem we consider here is different from the typical tasks in machine-learned ranking or learning to rank~\citep{Liu:09,Acharyya:13} where some feature information is available for each item and training data is used to calibrate some statistical models for ranking new items. In contrast to these problems, the feature information is not necessary in our crowdsourced ranking problem. Moreover, besides being applied to ranking items directly, our methods can be utilized to collect training labels for learning to rank problems. According to the type of training data utilized, statistical ranking methods can be classified into three categories~\citep{Liu:09,Acharyya:13}: pointwise method, pairwise method and listwise method. The pointwise methods~\citep{export:68128,Cooper:1992,Crammer01prankingwith} learn a ranking model based on the data of scores or ratings of items. The pairwise methods~\citep{Freund:2003,Burges:2005,Zheng:07,Cao:2006} and the listwise methods~\citep{Xu:2007,export:70428,export:63585,Kuo:2009} learn a ranking model using pairwise comparison results or partial rankings over a subset of items. For the pairwise or listwise methods, the crowdsourced ranking technique we proposed can be used as an upstream procedure that provides high-quality pairwise/listwise comparison data which helps increase the accuracy of the models in the aforementioned papers.


\section{Crowdsourced Ranking by Homogeneous  Workers}
\label{sec:homo}

In this section, we first consider a simplified setting where workers are homogeneous (we will clarify the meaning of ``homogeneous workers'' shortly). In Section \ref{sec:hetero}, we further extend the developed method for homogeneous workers to heterogeneous workers with different levels of reliability.

\subsection{Model Setup}


We assume that there are $K$ items (denoted by $\{1, \ldots, K\}$) to be ranked and each item $i$ has an \emph{unknown} latent score $\theta_i> 0$ for $i=1,2,\dots,K$.  Let $\btheta=(\theta_1,\theta_2,\dots,\theta_K)^T$, where each latent score $\theta_i$ models the intensity of preference to item $i$ under some criterion.  A \emph{ranking} over $K$ items $\{1,2,\dots,K\}$ is a permutation/one-to-one mapping $\pi:\{1,2,\dots,K\} \to \{1,2,\dots,K\}$ and  $\pi(i)$ is the \emph{rank} of item $i$ under $\pi$. We follow the convention that  $\theta_i>\theta_j$ means item $i$ is preferred to item $j$ and thus item $i$ should have a higher rank than item $j$. Therefore, the underlying \emph{true ranking} $\pi^*$ over $K$ items is determined by the ranking of their latent scores, i.e.,
\begin{equation}\label{eq:true_rank}
 \pi^*(i) > \pi^*(j)\;\;\;\text{ if and only if }\;\;\;\theta_i > \theta_j.
\end{equation}
We note that the latent scores naturally provide a characterization of \emph{ambiguity for a pair of items}:  when the values of $\theta_i$ and $\theta_j$ are closer, the pair of item $i$ and $j$ is more ambiguous in the sense that the true ordering of them is less obvious.

The way we explore the ranking of $\theta_i$'s is through the collection of workers' preferences on different pairs of items. Specifically, we will present only two items at a time to a worker, who will be asked to compare these two items according to the given ranking criterion. Each worker will not be asked to compare the same pair more than once. The results of comparisons will be collected over time and become our historical data, based on which, our task is to infer the true ranking $\pi^*$.

In this section, we consider a basic setup where the crowd workers are assumed to be \emph{homogeneous}, meaning that the probabilistic outcomes of their comparisons are only affected by the ambiguities of pairs.
More specifically, suppose a worker is randomly selected from the crowd to compare a pair of items $i$ and $j$ with $i<j$ and the comparison result is denoted by a random variable $Y_{ij}$:
\begin{eqnarray}
\label{Yij}
Y_{ij}=\left\{
\begin{array}{rl}
1&\text{ if item }i\text{ is preferred to item }j\text{ by the randomly selected worker}\\
-1&\text{ if item }j\text{ is preferred to item }i\text{ by the randomly selected worker}.
\end{array}
\right.
\end{eqnarray}
The setting of homogeneous workers means the probability distribution of $Y_{ij}$ takes the following form
\begin{equation}
\label{probij}
\text{Pr}(Y_{ij}=1) = \frac{\theta_i}{\theta_i+\theta_j}\quad \text{ and }\quad  \text{Pr}(Y_{ij}=-1) = \frac{\theta_j}{\theta_i+\theta_j}\quad \text{ for }i,j=1,2,\dots,K.
\end{equation}

The probabilistic model we used in \eqref{probij} is the well-known Bradley-Terry-Luce (BTL) model~\citep{Bradley:52,Luce:59}. We choose this model for the distribution of $Y_{ij}$ because it admits a simple structure and well fits our framework of dynamic sampling. Furthermore, our method developed for the  BTL model can be easily extended to the case of heterogeneous workers which will be studied in Section \ref{sec:hetero}.

It is worthwhile to mention that other comparison models can potentially be implemented here. Considering a simplified version of the Thurstone model \citep{Thurstone:27} in which each object $i$ has a score following $N(\theta_i,1)$, then we have
$$
\text{Pr}(Y_{ij}=1) = \Phi\left( \frac{\theta_i - \theta_j}{\sqrt{2}} \right) \quad \text{ and }\quad  \text{Pr}(Y_{ij}=-1) = \Phi\left( \frac{\theta_j - \theta_i}{\sqrt{2}} \right).
$$
The problem can still be formulated using a Bayesian decision process framework. However, there are several reasons why the BTL model is favored in this paper. First of all, moment matching under the Thurstone model does not have closed-form solutions and hence we must rely on numerical scheme to compute the first and second moments of the posterior. Second, using moment matching approach, because the posterior is an $n$-dimensional multivariate Gaussian distribution, we need to update $n(n+1)/2$ parameters (the number of mean parameters plus the number of off-diagonal elements of the covariance matrix) during each iteration of the algorithm whereas with Dirichlet posterior there are only $n$ parameters. Last but not least, with Thurstone model the ranking is no longer a simple sorting of parameters, which is a feature of the BTL model as shown in Theorem \ref{thmLOP}.

Since each worker can compare the same pair at most once, we assume the size of the crowd workers is large enough so that the distribution of $Y_{ij}$ stays the same after sampling workers without replacement. Note that we can assume $\sum_{i=1}^{K} \theta_i = 1$ without loss of generality since the distribution of $Y_{ij}$ in \eqref{probij} remains unchanged if we multiply each $\theta_i$ by the same positive constant.
The probability $\frac{\theta_i}{\theta_i+\theta_j}$ in \eqref{probij} can also be interpreted as the percentage of workers in the crowd who prefer item $i$ to item $j$.

Since the probabilistic model \eqref{probij} does not incorporate or reveal the quality of each worker in the comparison result, in the subsequent study of this section, we only need to focus on how to dynamically select pairs of items to compare. The worker will be selected randomly from the crowd. A dynamic choice over workers will be incorporated into our method in Section \ref{sec:hetero} where the performance of workers is modeled heterogeneously.


\subsection{Bayesian Decision Process}
\label{sec:MDP}
In a typical crowdsourcing marketplace, a monetary cost must be paid to a worker every time this worker completes a task such as comparing a pair of items. We assume the cost for each comparison is one unit and the total budget available is $T$ units so that at most $T$ pairs (repetition allowed) can be compared in total. Since comparing different pairs will generate different historical data and reveal different information about the true ranking,  it is critical to dynamically determine the right sequence of pairs to compare in order to maximize the final ranking accuracy, especially when the budget $T$ is small.

In the traditional offline setting, one needs to determine $T$ pairs at a time beforehand and request the comparisons on those pairs in a batch. The potential problem of such a static approach is that the budget $T$ is not spent in an efficient way to discover the true ranking. In fact, the distribution in \eqref{probij} implies that, when two items have similar latent scores, workers will provide highly inconsistent preferences and it is hard to reach an agreement on such a pair. In this case, the comparison results will be very noisy and one needs to spend more budget on this pair in order to rank them correctly. In contrast, when two items have significantly different latent scores, workers will provide consistent answers so that the additional information we can obtain is little from repeatedly comparing the same two items. In this case, one might want to reduce the budget on such a pair. Unfortunately, without any prior knowledge of the latent scores, it is impossible to decide how much budget should be spent on each pair before observing some comparison results.


In order to efficiently allocate the limited total budget over all pairs, we consider a dynamic crowdsourced ranking policy (Algorithm~\ref{alg:AKG}) where only one pair of items is selected and presented to a worker at each time based on historical comparison results.
This online method allows the budget to be adaptively shifted towards the ambiguous pairs so that the final ranking accuracy can be improved.

In particular, given the total budget $T$, the dynamic decision process consists of $T$ stages and, in stage $t=0,1,\dots,T-1$,  a pair of items $(i_t,j_t)$ with $i_t<j_t$ is presented to a randomly selected worker and we receive the comparison result $Y_{i_tj_t}$ defined in \eqref{Yij} and \eqref{probij}. The historical comparison results up to stage $t$ can be summarized by a $K\times K$ matrix $M^t$ with its entry\footnote{In this paper, the notation $A_{ij}$ represents the entry in the $i$-th row and $j$-th column of matrix $A$.} $M^t_{ij}$ equal to the number of times item $i$ is preferred to item $j$ up to stage $t$. For each stage $t$ where the pair $(i_t, j_t)$ is compared, we define $\Delta^t$ to be a sparse $K \times K$ matrix with only one non-zero element: $\Delta^t_{i_tj_t}=1$ if $Y_{i_tj_t}=1$ and $\Delta^t_{j_ti_t}=1$ if $Y_{i_tj_t}=-1$. By its definition, $M^t$ can be updated iteratively as follows
\begin{equation}
\label{Mt}
M^0=\mathbf{0},\quad
M^{t+1}=M^t+\Delta^t \quad
\text{ for }t=0,1,\dots,T-1,
\end{equation}
where $\mathbf{0}$ denotes the $K \times K$ all-zero matrix.

We denote an \emph{adaptive dynamic budget allocation/sampling policy} by $\mathcal{A}=\{(i_t,j_t)\}_{t=0,1,\dots,T-1}$ where $(i_t,j_t)=(i_t(M^t),j_t(M^t))$ depends on the previous comparison results through $M^t$. Our goal is to find the best $\mathcal{A}$ so that the inferred ranking based on all the historical comparisons (represented by $M^T$) achieves the highest accuracy.

To measure the accuracy of an inferred ranking $\pi$, we adopt the popular evaluation criterion --- normalized \emph{Kendall's tau  rank correlation coefficient}~\citep{Kendall:38} between $\pi$ and $\pi^*$ (Kendall's tau for short):
\begin{eqnarray}
\label{kendalltau}
\tau(\pi, \pi^*) &\equiv& \frac{\left|\left\{(i,j):i<j,~ \left(\pi(i)-\pi(j)\right)\left(\pi^*(i)-\pi^*(j)\right) >0 \right\}\right|}{K(K-1)/2}\\\nonumber
&=&\frac{2}{K(K-1)}\sum_{i\neq j}\mathbf{1}_{\{\pi(i)>\pi(j)\}}\mathbf{1}_{\{\theta_i>\theta_j\}},
\end{eqnarray}
where $\mathbf{1}_{\{\cdot\}}$ denotes the indicator function.
Here, the numerator counts the number of pairs that $\pi$ and $\pi^*$ agree with each other and the denominator is the total number of pairs over $K$ items. Hence, $\tau(\pi, \pi^*)\in[0,1]$ and represents the percentage of agreements between $\pi$ and $\pi^*$. The ranking accuracy of $\pi$ is higher when $\tau(\pi, \pi^*)$ is closer to one and $\pi=\pi^*$ if and only if $\tau(\pi, \pi^*)=1$.

However, we cannot infer a ranking based on the collected data by directly maximizing $\tau(\pi, \pi^*)$ because $\pi^*$ and $\btheta$ are unknown. To address this challenge, we adopt a Bayesian framework by proposing a prior distribution on $\btheta$ and infer a ranking $\pi$ that maximizes the posterior expectation of $\tau(\pi, \pi^*)$. Recall that the vector of latent scores $\btheta$ is assumed to lie in the simplex
\begin{equation}\label{eq:simplex}
\Delta\equiv\Bigl\{\btheta\in\mathbb{R}^k\bigg|\sum_{i=1}^{K} \theta_i = 1,\theta_i > 0 \Bigr \}.
\end{equation}
It is natural to assume that $\btheta$ is drawn from a \emph{Dirichlet prior distribution} parameterized by $\balpha^0 = (\alpha_1^0,\dots,\alpha_K^0)^T$ with $\alpha_i^0 > 0$ for all $i$ (note that Dirichlet distribution of order $K$ is supported on $\Delta$). Namely,
$$
\btheta\sim\Dir(\balpha^0)= {\frac {1}{\mathrm {B} (\balpha^0)}}\prod _{i=1}^{K}\theta_{i}^{\alpha _{i}-1},
$$
where
$
\mathrm {B} ({\boldsymbol {\alpha }})={\frac {\prod _{i=1}^{K}\Gamma (\alpha _{i})}{\Gamma {\bigl (}\sum _{i=1}^{K}\alpha _{i}{\bigr )}}}
$
and $\Gamma(x)\equiv \int_{0}^\infty \lambda^{x-1} e^{-\lambda} \mathrm{d} \lambda$ is the gamma function.
Given the comparison data $M^t$ up to stage $t$ and the probability distribution of each comparison result in  \eqref{probij}, the density function of the posterior distribution of $\btheta$ takes the following form,
\begin{eqnarray}
\label{posterior}
p(\btheta|M^t,\balpha^0) = \frac{1}{\mathrm{H}(M^t,\balpha^0)} \prod_{i\neq j} \left(\frac{\theta_i}{\theta_i+\theta_j}\right)^{M^t_{ij}}\prod_{i} \theta_i^{\alpha_i^0-1}= \frac{1}{\mathrm{H}(M^t,\balpha^0)} \frac{\prod_{i=1}^K \theta_i^{\beta_i^t+\alpha_i^0 - 1}}{\prod_{i<j}(\theta_i+\theta_j)^{M_{ij}^t+M_{ji}^t} },
\end{eqnarray}
where $\bbeta^t=(\beta_1^t,\beta_2^t,\dots,\beta_K^t)^T$ with $\beta_i^t\equiv \sum_{j\neq i}M_{ij}^t$, i.e., the number of times item $i$ is preferred to another item up to stage $t$, and
$$
\mathrm{H}(M^t,\balpha^0)\equiv\int_{\Delta}  \frac{\prod_{i=1}^K \theta_i^{\beta_i^t+\alpha_i^0 - 1}}{\prod_{i<j}(\theta_i+\theta_j)^{M_{ij}^t+M_{ji}^t} }\mathrm{d}\btheta,
$$
is the normalization constant.

With this posterior distribution in place and with $M^t$ at any stage $t$, we can infer a ranking $\widehat{\pi}_t$ to maximize the posterior expected ranking accuracy measured by its Kendall's tau with respect to $\pi^*$, namely, to find
\begin{eqnarray}
\label{maxtau}
\widehat{\pi}_t &\in& \argmax_{\pi} \mathbb{E}\left[\tau(\pi, \pi^*)|M^t,\balpha^0\right],
\end{eqnarray}
where the expectation is taken with respect to the posterior distribution $p(\btheta|M^t,\balpha^0)$ in \eqref{posterior}. We denote the corresponding maximum posterior expected accuracy by $h(M^t)$, i.e.,
\begin{eqnarray}
\label{maxh}
h(M^t)&\equiv& \max_{\pi} \mathbb{E}\left[\tau(\pi, \pi^*)|M^t,\balpha^0\right],
\end{eqnarray}
where the dependence of $h$ on the prior $\balpha^0$ is suppressed for notational simplicity.
We are interested in finding a dynamic budget allocation policy $\mathcal{A}=\{(i_t,j_t)\}_{t=0,1,\dots,T-1}$ that maximizes $h(M^T)$, i.e., the final expected ranking accuracy when the budget is exhausted. This problem can be stated as
\begin{eqnarray}
\label{maxA}
\max_{\mathcal{A}} \mathbb{E}^{\mathcal{A}}\left[h(M^T)|\balpha^0\right],
\end{eqnarray}
where $\mathbb{E}^{\mathcal{A}}$ represents the expectation over the sample paths (i.e., the sampled pairs and outcomes) generated by the policy $\mathcal{A}$.

The maximization problem in \eqref{maxA} can be formulated as a $T$-stage Bayesian Markov decision process (MDP), where the \emph{state variable} is the posterior distribution in \eqref{posterior} or simply the matrix $M^t$. The \emph{state space} at each stage $t$ denoted by $\mathcal{S}^t$ takes the form of
\begin{equation}\label{eq:state_space}
 \mathcal{S}^t= \Bigl \{ M^t \in \mathbb{Z}_{\geq 0}^{K \times K}: \sum_{i,j} M^{t}_{ij}=t \Bigr\},
\end{equation}
where $\mathbb{Z}_{\geq 0}$ denotes the set of non-negative integers. The state variable makes a transition according to \eqref{Mt} given the observed comparison result $Y_{i_t j_t}$, where the sampled pair $(i_t, j_t)$ is determined by the policy $\mathcal{A}$. The expected transition probabilities take the form of,
\begin{eqnarray}
\label{Eprobij1}
\mathbb{E}\left[\text{Pr}(Y_{ij}=1)|M^t,\balpha^0\right] &=&\mathbb{E}\left[ \frac{\theta_i}{\theta_i+\theta_j}\big|M^t,\balpha^0\right]\\\label{Eprobij0}
\mathbb{E}\left[\text{Pr}(Y_{ij}=-1)|M^t,\balpha^0\right] &=&\mathbb{E}\left[ \frac{\theta_j}{\theta_i+\theta_j}\big|M^t,\balpha^0\right]
\end{eqnarray}
for $1 \leq i < j \leq K$ and the expectation is taken over the posterior of $\btheta$ in \eqref{posterior}. To complete the definition of our Bayesian MDP for crowdsourced ranking, we still need to define the \emph{stage-wise reward}. To this end, we rewrite $h(M^T)$  in \eqref{maxA} as a telescopic sum,
\begin{equation}\label{R}
 h(M^T)= \sum_{t=0,1,\ldots, T-1} R(M^t, i_t, j_t, Y_{i_t j_t}); \qquad  R(M^t, i_t, j_t, Y_{i_t j_t}) \equiv h(M^{t+1})-h(M^{t}),
\end{equation}
and note that $R(M^t, i_t, j_t, Y_{i_t j_t})=h(M^{t+1})-h(M^{t})$ only depends on $M^t$, $i_t$, $j_t$, $Y_{i_tj_t}$.  Given \eqref{R}, the maximization problem \eqref{maxA} is equivalent to
\begin{eqnarray}
\label{maxsum}
&&\max_{\mathcal{A}}\mathbb{E}^{\mathcal{A}}\left[h(M^0)+\sum_{t=0}^{T-1}R(M^t,i_t,j_t,Y_{i_tj_t})\bigg|\balpha^0\right]\\
&=&h(M^0)+\max_{\mathcal{A}}\mathbb{E}^{\mathcal{A}}\left[\sum_{t=0}^{T-1}\mathbb{E}\left[R(M^t,i_t,j_t,Y_{i_tj_t})\big|M^t,\balpha^0\right]\bigg|\balpha^0\right].\nonumber
\end{eqnarray}
From \eqref{maxsum}, it is clear that $R(M^t, i_t, j_t, Y_{i_t j_t})$ is  the \emph{stage-wise reward},
which can be interpreted as the improvement of the expected ranking accuracy after receiving the comparison result $Y_{i_tj_t}$ at stage $t$ for $t=0,1,\dots,T-1$.

Given the Bayesian MDP in place, we can apply the dynamic programming (DP) algorithm  (a.k.a. backward induction) \citep{Puterman:05} to compute the optimal policy. Although DP finds the optimal policy, its computation is intractable because:
\begin{enumerate}
 \item The sophisticated form of the posterior distribution in \eqref{posterior} makes it difficult to evaluate the posterior expected ranking accuracy $\mathbb{E}\left[\tau(\pi, \pi^*)|M^t,\balpha^0\right]$ in \eqref{maxh} and the expected transition probabilities in \eqref{Eprobij1} and \eqref{Eprobij0}.
  \item The maximization problem \eqref{maxh} for solving the optimal posterior expected ranking accuracy is essentially a linear ordering problem \citep{grotschel1984cpa}, which is NP-hard in general (see Section \ref{sec:AKG} for more details).
  \item The size of the state space $\calS^t$ grows exponentially in $t$ according to \eqref{eq:state_space}, which is known as the curse of dimensionality that prevents us from solving \eqref{maxsum} exactly with the standard techniques such value iteration, policy iteration and linear programming.
\end{enumerate}
To address these challenges, we propose an approximated knowledge gradient policy (AKG) in the next Section.

\subsection{Approximated Knowledge Gradient Policy}
\label{sec:AKG}

In this section, we describe an approximated policy to solve \eqref{maxA}, which is computationally efficient and still provides an inferred ranking with high quality. The proposed approximation policy belongs to the family of \emph{knowledge gradient} (KG) policies~\citep{Gupta:96,Frazier:08,Powell:11a,Ryzhov:2012}, which is essentially a single-step look-ahead policy. In our problem, the KG policy will sample the next pair of items with the highest expected stage-wise reward in each stage, i.e.,
choosing the pair $(i_t,j_t)$ such that

\begin{eqnarray}
\label{KG}
(i_t,j_t) &\in\argmax_{i<j}  & \mathbb{E}\left[R(M^t,i_t,j_t,Y_{i_tj_t})\big|M^t,\balpha^0\right]\\\nonumber
& =\argmax_{i<j} & \mathbb{E}\left[\text{Pr}(Y_{ij}=1)|M^t,\balpha^0\right] R(M^t,i_t,j_t,1) \\
&& +
\mathbb{E}\left[\text{Pr}(Y_{ij}=-1)|M^t,\balpha^0\right]R(M^t,i_t,j_t,-1). \nonumber
\end{eqnarray}
Despite its simplicity and wide applicability, the implementation of the KG policy for our problem in \eqref{KG} is still computationally intractable since we have to evaluate the expected stage-wise reward $\mathbb{E}\left[R(M^t,i_t,j_t,Y_{i_tj_t})\big|M^t,\balpha^0\right]$,  where two main challenges will arise.

First, we have to evaluate the transition probabilities \eqref{Eprobij1} and \eqref{Eprobij0} as well as the ranking accuracy \eqref{maxh}, which can be written as
\begin{eqnarray}
\nonumber
h(M^t)&=& \max_{\pi} \mathbb{E}\left[\tau(\pi, \pi^*)|M^t,\balpha^0\right]\\\nonumber
&=&\max_{\pi}\frac{2\sum_{i\neq j}\mathbb{E}\left[\mathbf{1}_{\{\pi(i)>\pi(j)\}}\mathbf{1}_{\{\theta_i > \theta_j \}}|M^t,\balpha^0\right]}{K(K-1)}\\
&=& \max_{\pi}\frac{2\sum_{i\neq j}\mathbf{1}_{\pi(i)>\pi(j)}\text{Pr}\left(\theta_i>\theta_j|M^t,\balpha^0\right)}{K(K-1)}.
\label{eq:LOP}
\end{eqnarray}
However, due to the complicated structure of the posterior distribution $p(\btheta|M^t,\balpha^0)$ in \eqref{posterior}, the expected transition probabilities  \eqref{Eprobij1} and \eqref{Eprobij0} and the posterior probability $\text{Pr}\left(\theta_i>\theta_j|M^t,\balpha^0\right)$ in \eqref{eq:LOP} do not admit a closed form so that one needs to use multidimensional numerical integral or sampling techniques to compute their values. Note that for each stage $t$, we need to evaluate \eqref{Eprobij1}, \eqref{Eprobij0} and $\text{Pr}\left(\theta_i>\theta_j|M^t,\balpha^0\right)$ for all $K(K-1)/2$ pairs. When these quantities cannot be easily computed, the overall computational cost will be extremely expensive.

Second, even if the posterior probabilities $\text{Pr}\left(\theta_i>\theta_j|M^t,\balpha^0\right)$ for all pairs are  given, the maximization problem  \eqref{eq:LOP} with respect to a global ranking $\pi$ is still very challenging.  In fact, this problem is equivalent to the \emph{maximum linear ordering problem (MAX-LOP)} described as follows.
Let $G=(V,E, w)$ be a completed directed graph defined on a set $V$ of $K$ nodes, where the edge set $E$ contains the directed arcs between all pairs of nodes and $w(i,j)$ refers to the weight associated with the arc from node $i$ to node $j$. A tournament $D$ is a sub-graph of $G$ such that, for any pair of nodes $i$ and $j$, $D$ contains either the arc from $i$ to $j$ or the arc from $j$ to $i$ but not both. The MAX-LOP aims to find an acyclic tournament $D$ with a maximum total weight on its arcs. If we interpret the arc from node $i$ and node $j$ as the preference of node $i$ to node $j$ under a ranking criterion, each acyclic tournament in $G$ corresponds one-to-one to a global ranking of the nodes. Hence, MAX-LOP is equivalent to finding a ranking $\pi$ such that the total weight $\sum_{\pi(i) > \pi(j) } w(i, j)$ is maximized. In problem~\eqref{eq:LOP}, the nodes correspond to the $K$ items and the weight $w(i,j)=\text{Pr}\left(\theta_i>\theta_j|M^t,\balpha^0\right)$. Unfortunately, the MAX-LOP is known to be a NP-hard problem and in fact, APX (approximable)-complete and thus no PTAS  (Polynomial Time Approximation Scheme) under P $\neq$ NP \citep{Mishra04LOP}.

Given these two challenges, evaluating $\mathbb{E}\left[R(M^t,i_t,j_t,Y_{i_tj_t})\big|M^t,\balpha^0\right]$ and solving \eqref{KG} repeatedly at each stage are computationally intractable. To address this problem, we propose an \emph{approximated knowledge gradient} (AKG) policy, which first replaces the stage-wise reward \eqref{R} by an approximated but computable reward and then chooses the pair that maximizes this approximated reward. 
Our approximation scheme starts with approximating the posterior distribution $p(\btheta|M^t,\balpha^0)$ in \eqref{posterior} recursively using a sequence of Dirichlet distributions $\Dir(\balpha^t)$ for $t=1,2,\dots,T$ based on \emph{moment matching}. One key benefit of such an approximation is that, at each stage $t$, the approximated posterior distribution of $\btheta$ is still a Dirichlet distribution so that the NP-hard MAX-LOP problem in \eqref{eq:LOP} will admit a simple solution via a sorting procedure (see Theorem \ref{thmLOP}).

Although there exist other methods for posterior approximation, these methods cannot be implemented as efficiently as moment matching in our application. For example, some methods such as variational inference (e.g., \citealp{Beal:03, Paisley:12}) minimize the KL-divergence between the exact posterior and the variational posterior, which requires an iterative optimization algorithm as a subroutine. Other methods like Gibbs sampler are computationally expensive in our case because the full conditional distribution does not have a closed form to allow easy sampling. In contrast, the proposed (algorithmic) moment matching admits a closed-form solution for approximating the posterior, which is computationally very efficient, and further provides a Dirichlet distribution as the approximated posterior, which facilitates solving the MAX-LOP. We note that the close-form update is critical for online crowdsourcing applications to reduce the computation time between two stages. In practice, since the crowd workers want to maximize their return in a short period of time, they may quit the current task if we let them wait for too long before we determine the next pair. Finally, we note that, although providing the theoretical guarantee for such an iterative approximation is hard in the Bayesian setup, we empirically show that the resulting AKG policy will generate a final ranking of a high accuracy with the limited budget.


Now we formally introduce the posterior approximation and AKG policy.  Suppose $\btheta\sim\Dir(\balpha)$ for some parameters $\balpha\in\mathbb{R}^K$. We consider a basic case where only one comparison result $Y_{ij}$ for a pair $(i,j)$ with $i<j$ has been observed. In this case, we approximate the posterior $p(\btheta|Y_{ij},\balpha)$ by another Dirichlet distribution $\Dir(\balpha')$ such that
\begin{eqnarray}
\label{eq:moment_matching1}
\mathbb{E} \left[\theta_k |\btheta\sim\Dir(\balpha')\right]&=&\mathbb{E}[\theta_k|Y_{ij},\balpha]\text{ for }k=1,2,\dots,K\\\label{eq:moment_matching2}
\mathbb{E} \left[\sum_{k=1}^K\theta_k^2 |\btheta\sim\Dir(\balpha')\right] &=&\mathbb{E}\left[ \sum_{k=1}^K\theta_k^2| Y_{ij},\balpha\right].
\end{eqnarray}
This system of equations has the following explicit characterization.
\begin{proposition}
\label{propMM}
Suppose $\btheta\sim\Dir(\balpha)$ and $Y_{ij}$ is the only comparison result for $i<j$. Let $\alpha_0=\sum_{k=1}^K\alpha_k$ and $\alpha'_0=\sum_{k=1}^K\alpha'_k$. The equations \eqref{eq:moment_matching1} and \eqref{eq:moment_matching2} can be represented as
\begin{eqnarray}
\label{MM1}
\left\{
\begin{array}{rcl}
\frac{\alpha'_i}{\alpha'_0}&=&\frac{\left(\alpha_i+\frac{1+Y_{ij}}{2}\right)(\alpha_i+\alpha_j)}{\alpha_0(\alpha_i+\alpha_j+1)}\\
\frac{\alpha'_j}{\alpha'_0}&=&\frac{\left(\alpha_j+\frac{1-Y_{ij}}{2}\right)(\alpha_i+\alpha_j)}{\alpha_0(\alpha_i+\alpha_j+1)}\\
\frac{\alpha'_k}{\alpha'_0}&=& \frac{\alpha_k}{\alpha_0} \qquad \mathrm{ for }\;\;\; k\neq i,j\\
\sum_{k=1}^K\frac{\alpha'_k(\alpha'_k+1)}{\alpha'_0(\alpha'_0+1)}&=&  \frac{\left(\alpha_i+\frac{1+Y_{ij}}{2}\right)\left(\alpha_i+\frac{3+Y_{ij}}{2}\right)(\alpha_i+\alpha_j)}{\alpha_0(\alpha_0+1)(\alpha_i+\alpha_j+2)} \\ && + \frac{\left(\alpha_j+\frac{1-Y_{ij}}{2}\right)\left(\alpha_j+\frac{3-Y_{ij}}{2}\right)(\alpha_i+\alpha_j)}{\alpha_0(\alpha_0+1)(\alpha_i+\alpha_j+2)}+ \sum_{k\neq i,j} \frac{\alpha_k(\alpha_k+1)}{\alpha_0(\alpha_0+1)}.
\end{array}
\right.
\end{eqnarray}
\end{proposition}

The proof of Proposition \ref{propMM} is provided in the Appendix. We denote any $\balpha'$ that satisfies \eqref{eq:moment_matching1} and \eqref{eq:moment_matching2}, and thus \eqref{MM1}, by
\begin{eqnarray}
\label{defMM}
\balpha'=\textbf{MM}(\balpha,i,j,Y_{ij}).
\end{eqnarray}
Note that, given $\balpha$, $i$, $j$ and $Y_{ij}$, the right-hand sides of \eqref{MM1} are all constants so that we can solve $\balpha'=\textbf{MM}(\balpha,i,j,Y_{ij})$ in a closed form. In fact, we denote the constants on the right hand sides of \eqref{MM1} as $C_i$, $C_j$, $C_k$ (for $k\neq i ,j$) and $D$, respectively. It is easy to show that $\sum_{k=1}^KC_k=1$. The first three equalities in \eqref{MM1} imply that $\alpha'_k=C_k\alpha'_0$ for $k=1,2,\dots,K$ so that the fourth equality in \eqref{MM1} can be represented as
$
\sum_{k=1}^KC_k(C_k\alpha'_0+1)=D(\alpha'_0+1).
$
Solving $\alpha'_0$ from this equation leads to a closed-form for $\balpha'=\textbf{MM}(\balpha,i,j,Y_{ij})$ as follows
\begin{eqnarray}
\label{MM1form}
\alpha'_0=\frac{D-1}{\sum_{k=1}^KC_k^2-D}\quad\text{ and }\quad \alpha'_k=C_k\alpha'_0\text{ for }k=1,2,\dots,K.
\end{eqnarray}

Although the above approximation scheme is established for only one comparison result, it produces a Dirichlet distribution $\Dir(\balpha')$ which has the same type as the prior distribution $\Dir(\balpha)$. Therefore, as more comparison results are generated sequentially,  we can apply this approximation scheme iteratively after each comparison result.
In particular, given a policy $\mathcal{A}=\{(i_t,j_t)\}_{t=0,1,\dots,T-1}$ with $i_t<j_t$ and the comparison results $\{Y_{i_tj_t}\}_{t=0,1,\dots,T-1}$, we define $\balpha^t$ recursively as
\begin{eqnarray}
\label{alphat}
\balpha^{t+1}=\textbf{MM}(\balpha^t,i_t,j_t,Y_{i_tj_t})
\end{eqnarray}
for $t=1,2,\dots,T$. By doing so, we approximate the posterior distribution $p(\btheta|M^t,\balpha^0)$ by the Dirichlet distribution $\Dir(\balpha^t)$ for $t=1,2,\ldots, T$.

With $p(\btheta|M^t,\balpha^0)$ approximated by $\Dir(\balpha^t)$, we can mitigate the two challenges mentioned at the beginning of this subsection. First, we can approximate \eqref{Eprobij1} and \eqref{Eprobij0} as
\begin{eqnarray}
\label{approxprob1}
\mathbb{E}\left[\text{Pr}(Y_{ij}=1)\big|M^t,\balpha^0\right]
\approx\mathbb{E}\left[ \frac{\theta_i}{\theta_i+\theta_j}\big|\btheta\sim\Dir(\balpha^t)\right]&=&\frac{\alpha^t_i}{\alpha^t_i+\alpha^t_j}\\
\label{approxprob2}
\mathbb{E}\left[\text{Pr}(Y_{ij}=-1)\big|M^t,\balpha^0\right]
\approx\mathbb{E}\left[ \frac{\theta_i}{\theta_i+\theta_j}\big|\btheta\sim\Dir(\balpha^t)\right]&=&\frac{\alpha^t_j}{\alpha^t_i+\alpha^t_j}
\end{eqnarray}
and approximate
$\text{Pr}\left(\theta_i>\theta_j|M^t,\balpha^0\right)$ in \eqref{posterior} as
\begin{eqnarray}
\label{approxprob}
\text{Pr}\left(\theta_i>\theta_j|M^t,\balpha^0\right)
\approx\text{Pr}\left(\theta_i>\theta_j|\btheta\sim\Dir(\balpha^t)\right)
=\int_{\frac{1}{2}}^1t^{\alpha^t_i-1}(1-t)^{\alpha^t_j-1}\mathrm {d}t
=I_{\frac{1}{2}}(\alpha^t_j,\alpha^t_i),
\end{eqnarray}
where $I_{x}(a,b)={\dfrac {\mathrm {B} (x;\,a,b)}{\mathrm {B} (a,b)}}$ is known as the \emph{regularized incomplete beta function} with $\mathrm {B} (x;\,a,b)=\int_{0}^{x} \lambda^{a-1}\,(1-\lambda)^{b-1}\,\mathrm {d} \lambda$ and $\mathrm {\mathrm {B} } (a,b)=\int _{0}^{1} \lambda^{a-1}(1-\lambda)^{b-1}\,\mathrm {d} \lambda$. Note that the approximated quantities in \eqref{approxprob1}, \eqref{approxprob2} and \eqref{approxprob} are much easier to compute than the original ones.

More importantly, the approximation \eqref{approxprob} simplifies the NP-hard MAX-LOP in \eqref{eq:LOP}:
$$
\max_{\pi} \mathbb{E}\left[\tau(\pi, \pi^*)|M^t,\balpha^0\right]\approx
\max_{\pi} \mathbb{E}\left[\tau(\pi, \pi^*)|\btheta\sim\Dir(\balpha^t)\right].
$$
The right-hand side is still a MAX-LOP but has a special structure so that it can be solved easily by a simple sorting procedure. In particular, the following theorem shows that  when $\btheta \sim \mathrm{Dir}(\balpha)$, the optimal ranking in \eqref{KG} can be obtained by sorting the components of $\balpha$.
\begin{theorem}\label{thmLOP}
Suppose $\btheta \sim \mathrm{Dir}(\balpha)$. We have
\begin{eqnarray}
\nonumber
\Pi_{\balpha}&\equiv&\left\{\pi|\pi\text{ is a ranking of }\{1,2,\dots,K\}\text{ such that }\pi(i)>\pi(j) \text{ only if } \alpha_i\geq\alpha_j \text{ for all } i,j\right\}\\\label{thm:LOP}
&=&\argmax_{\pi}\mathbb{E}\left[\tau(\pi, \pi^*)|\btheta \sim \Dir(\balpha)\right]
\end{eqnarray}
\end{theorem}
\begin{proof}
We first show that $\argmax_{\pi}\mathbb{E}\left[\tau(\pi, \pi^*)|\btheta \sim \Dir(\balpha)\right]\subset\Pi_{\balpha}$.
Suppose $\hat\pi$ is the optimal solution of \eqref{thm:LOP} where $\hat\pi(j)>\hat\pi(i)$ for a pair $i$ and $j$ with $\alpha_i>\alpha_j$. We put all items in a row with their ranks given by $\hat\pi$ decreasing from the left to the right and obtain a pattern like
$$
X\cdots Xj\underbrace{X\cdots X}_{S}iX\cdots X,
$$
where $X$ represents some item different from $i$ and $j$ and $S$ represents the set of items ranked between $i$ and $j$. We will show that the objective value of \eqref{thm:LOP} can be increased by switching the ranks of $i$ and $j$.

Recall that the expected accuracy of $\hat\pi$ can be represented as
\begin{eqnarray}
\label{tauvalue}
 \mathbb{E}\left[\tau(\hat\pi, \pi^*)|\btheta \sim \Dir(\balpha)\right]&=& \frac{2}{K(K-1)}\sum_{i'\neq j'}\mathbf{1}_{\hat\pi(i')>\hat\pi(j')}\text{Pr}\left(\theta_{i'}>\theta_{j'}|\btheta \sim \Dir(\balpha)\right)\\\nonumber
 &=&\frac{2}{K(K-1)}\left[ I_{\frac{1}{2}}(\alpha_i,\alpha_j)+ \sum_{s \in S} I_{\frac{1}{2}}(\alpha_s,\alpha_j) + \sum_{s \in S} I_{\frac{1}{2}}(\alpha_i,\alpha_s)+C\right],
\end{eqnarray}
where $C$ is the summation of the remaining terms like $I_{\frac{1}{2}}(\alpha_{i'},\alpha_{j'})$ which have either at least one of $i'$ and $j'$ not in $S\cup\{i,j\}$ or both $i'$ and $j'$ in $S$.

Note that switching the ranks of $i$ and $j$ does not change the values of the terms in $C$. In fact, after such a switch, we obtain a new ranking $\hat\pi'$ whose objective value in \eqref{thm:LOP} is
\begin{eqnarray*}
 \mathbb{E}\left[\tau(\hat\pi', \pi^*)|\btheta \sim \Dir(\balpha)\right]
 &=&\frac{2}{K(K-1)}\left[ I_{\frac{1}{2}}(\alpha_j,\alpha_i)+ \sum_{s \in S} I_{\frac{1}{2}}(\alpha_j,\alpha_s) + \sum_{s \in S} I_{\frac{1}{2}}(\alpha_s,\alpha_i)+C\right].
\end{eqnarray*}
Using the fact that $I_{\frac{1}{2}}(a,b)$ is monotonically decreasing in $a$ and monotonically increasing in $b$ and noticing that $\alpha_i> \alpha_j$, we have \begin{eqnarray*}
I_{\frac{1}{2}}(\alpha_j,\alpha_i)+ \sum_{s \in S} I_{\frac{1}{2}}(\alpha_j,\alpha_s) + \sum_{s \in S} I_{\frac{1}{2}}(\alpha_s,\alpha_i)>
 I_{\frac{1}{2}}(\alpha_i,\alpha_j)+ \sum_{s \in S} I_{\frac{1}{2}}(\alpha_s,\alpha_j) + \sum_{s \in S} I_{\frac{1}{2}}(\alpha_i,\alpha_s),
\end{eqnarray*}
which implies $\mathbb{E}\left[\tau(\hat\pi', \pi^*)|\btheta \sim \Dir(\balpha)\right]>\mathbb{E}\left[\tau(\hat\pi, \pi^*)|\btheta \sim \Dir(\balpha)\right]$, contradicting with the optimality of $\hat\pi$. Hence, we can have $\hat\pi(i)>\hat\pi(j)$ only if $\alpha_i\geq\alpha_j$, meaning that $\hat\pi\in\Pi_{\balpha}$.

We then show $\argmax_{\pi}\mathbb{E}\left[\tau(\pi, \pi^*)|\btheta \sim \Dir(\balpha)\right]=\Pi_{\balpha} $ by showing that $\mathbb{E}\left[\tau(\pi, \pi^*)|\btheta \sim \Dir(\balpha)\right]$ has the same value for any $\pi\in\Pi_{\balpha}$. Suppose $\hat\pi$ and $\hat\pi'$ both belong to $\Pi_{\balpha}$ and there exists a pair $i$ and $j$ with $i\neq j$ such that $\hat\pi(i)>\hat\pi(j)$ and $\hat\pi'(j)>\hat\pi'(i)$. By the definition of $\Pi_{\balpha}$, we have $\alpha_i=\alpha_j$ so that
$$
\text{Pr}\left(\theta_{i}>\theta_{j}|\btheta \sim \Dir(\balpha)\right)
=I_{\frac{1}{2}}(\alpha_j,\alpha_i)=\frac{1}{2}=I_{\frac{1}{2}}(\alpha_i,\alpha_j)=
\text{Pr}\left(\theta_{j}>\theta_{i}|\btheta \sim \Dir(\balpha)\right).
$$
This means
\begin{eqnarray*}
&&\mathbf{1}_{\hat\pi(i)>\hat\pi(j)}\text{Pr}\left(\theta_{i}>\theta_{j}|\btheta \sim \Dir(\balpha)\right)
+
\mathbf{1}_{\hat\pi(j)>\hat\pi(i)}\text{Pr}\left(\theta_{j}>\theta_{i}|\btheta \sim \Dir(\balpha)\right)\\
&=&\mathbf{1}_{\hat\pi'(i)>\hat\pi'(j)}\text{Pr}\left(\theta_{i}>\theta_{j}|\btheta \sim \Dir(\balpha)\right)
+
\mathbf{1}_{\hat\pi'(j)>\hat\pi'(i)}\text{Pr}\left(\theta_{j}>\theta_{i}|\btheta \sim \Dir(\balpha)\right)
\end{eqnarray*}
for any pair $i$ and $j$ so that $\mathbb{E}\left[\tau(\hat\pi, \pi^*)|\btheta \sim \Dir(\balpha)\right]=\mathbb{E}\left[\tau(\hat\pi', \pi^*)|\btheta \sim \Dir(\balpha)\right]$ by the formulation \eqref{tauvalue}, which completes the proof.
\end{proof}

Given a parameter vector $\balpha$, we denote any ranking in $\Pi_{\balpha}$ by $\pi_{\balpha}$.
Using moment matching and Theorem \ref{thmLOP}, we can approximate the stage-wise reward  $R(M^t,i,j,Y_{ij})$ by
\begin{eqnarray}
\nonumber
R(M^t,i,j,Y_{ij})&=&h(M^{t+1})-h(M^t)\\\nonumber
&=&\max_{\pi} \mathbb{E}\left[\tau(\pi, \pi^*)|M^{t+1},\balpha^0\right]-\max_{\pi} \mathbb{E}\left[\tau(\pi, \pi^*)|M^t,\balpha^0\right]\\\nonumber
&\approx&\max_{\pi} \mathbb{E}\left[\tau(\pi, \pi^*)|\btheta\sim\Dir(\hat\balpha)\right]-\max_{\pi} \mathbb{E}\left[\tau(\pi, \pi^*)|\btheta\sim\Dir(\balpha^t)\right]\\\nonumber
&=& \mathbb{E}\left[\tau(\pi_{\hat\balpha}, \pi^*)|\btheta\sim\Dir(\hat\balpha)\right]-\mathbb{E}\left[\tau(\pi_{\hat\balpha}, \pi^*)|\btheta\sim\Dir(\balpha^t)\right]\\\nonumber
&=& \frac{2}{K(K-1)}\Biggl(\sum_{i',j' \; : \; \pi_{\hat\balpha}(i')>\pi_{\hat\balpha}(j')}I_{\frac{1}{2}}(\hat\alpha_{j'},\hat\alpha_{i'})
-\sum_{i',j' \; :  \; \pi_{\balpha^t}(i')>\pi_{\balpha^t}(j')}I_{\frac{1}{2}}(\alpha^t_{j'},\alpha^t_{i'})\Biggr)\\\label{approxreward}
&\equiv&\tilde R(\balpha^t,i,j,Y_{ij})
\end{eqnarray}
where
$\hat\balpha=\textbf{MM}(\balpha^t,i,j,Y_{ij})$, the third equality is from Theorem~\ref{thmLOP} and the fourth equality is due to \eqref{approxprob}. Putting \eqref{KG}, \eqref{approxprob1}, \eqref{approxprob2}, and \eqref{approxreward} together, we can approximate the expected stage-wise reward $\mathbb{E}\left[R(M^t,i,j,Y_{ij})\big|M^t,\balpha^0\right]$ as
\begin{eqnarray}
\nonumber
&&\mathbb{E}\left[R(M^t,i,j,Y_{ij})\big|M^t,\balpha^0\right]\\\nonumber
&=&\mathbb{E}\left[\text{Pr}(Y_{ij}=1)|M^t,\balpha^0\right]R(\balpha^t,i,j,1)
+
\mathbb{E}\left[\text{Pr}(Y_{ij}=-1)|M^t,\balpha^0\right]R(M^t,i,j,-1)\\\label{AKG}
&\approx&\frac{\alpha^t_i}{\alpha^t_i+\alpha^t_j} \tilde R(\balpha^t,i,j,1)
+
\frac{\alpha^t_i}{\alpha^t_i+\alpha^t_j} \tilde R(\balpha^t,i,j,-1).
\end{eqnarray}
The proposed AKG policy  will choose the pair $(i_t,j_t)$ that maximizes the approximated expected stage-wise reward in \eqref{AKG}.  As a summary, we describe the AKG policy as Algorithm~\ref{alg:AKG}.


\begin{algorithm}[!t]
\caption{Approximated Knowledge Gradient Policy with Homogeneous Workers}
\begin{algorithmic}[1]\label{alg:AKG}
  \item[\textbf{Initialization:}] Choose $\balpha^0$ for the prior distribution. Let $M^0$ be a $K\times K$ all-zero matrix.

  \item[\textbf{For}] $t = 0,\dots,T-1$ \textbf{do}

  \STATE For each pair $(i,j)$ with $i<j$, compute $\tilde R(\balpha^t,i,j,1)$ and $\tilde R(\balpha^t,i,j,-1)$ according to \eqref{approxreward}.
  \STATE Select $(i_t,j_t)$ such that
    \begin{eqnarray}
    \label{eq:AKGpair}
  (i_t,j_t)\in\argmax_{i<j}\left[\frac{\alpha^t_i}{\alpha^t_i+\alpha^t_j} \tilde R(\balpha^t,i,j,1)
+
\frac{\alpha^t_i}{\alpha^t_i+\alpha^t_j} \tilde R(\balpha^t,i,j,-1)\right]
\end{eqnarray}
and present item $i_t$ and item $j_t$ to a randomly selected worker and receive the comparison result $Y_{i_tj_t}$.
  \STATE According to \eqref{defMM} and \eqref{MM1form}, compute
  \begin{eqnarray}
  \label{eq:AKGMM}
\balpha^{t+1}=\textbf{MM}(\balpha^t,i_t,j_t,Y_{i_tj_t})
\end{eqnarray}

  \item[\textbf{End For}]
  \item[\textbf{Return:}] The aggregated ranking $\pi_{\balpha^T}$ obtained by sorting the components of $\balpha^T$.
\end{algorithmic}
\end{algorithm}

%

It is noteworthy that it is  easy to implement a \emph{batch version} of Algorithm \ref{alg:AKG}. In fact, the AKG policy in  Algorithm \ref{alg:AKG} is known as an \emph{index policy} where the right-hand side of \eqref{eq:AKGpair}, which calculates the marginal improvement on the ranking accuracy, can be treated as the index for each pair of items. The AKG policy selects the pair with the highest index at each stage. In the batch version, instead of selecting only one pair, one heuristics is to select the top $B$ pairs and distribute to workers simultaneously, where $B$ is a pre-defined batch size. Such a batch implementation can reduce the waiting time of crowd workers and thus accelerate the ranking procedure. Moreover, the AKG policy can be combined with some other batch optimization techniques \citep{wu2016parallel} to determine the optimal set of pairs to evaluate next.

\section{Crowdsourced Ranking by Heterogeneous Workers}
\label{sec:hetero}

In the previous section, we considered the setting of homogeneous workers, where the comparison results are determined only by the intrinsic latent scores of items but not by the characteristics of workers. However, on crowdsourcing platforms, the quality of the workers varies a lot. Some workers are less reliable or lack of the domain knowledge;  some workers are spammers, who either do not actually take a look at the assigned pairs or are robots pretending to be human workers, and thus provide random comparison results in order to quickly receive payment; some workers may be poorly informed (or even malicious), misunderstand the ranking criteria and thus always flip the comparison results. To identify the reliability of a worker, one can assign the same pair of items to multiple workers and hope to identify the unreliable ones whose labels are often different from the majority. However, the abuse of this strategy will result in hiring too many workers and lead to a quick growth of the monetary cost. In order to maximize the accuracy of the final ranking under the limited amount of budget, it is critical to balance the budget spent on estimating the reliability of the workers and learning the true ranking of the items. To formalize such trade-off, we incorporate the reliability of each worker to our previous Bayesian MDP and generalize the AKG policy to the heterogeneity of workers.

\subsection{Model Setup}
Similar to the previous setting, we assume that each item $i$ has an unknown latent score $\theta_i> 0$ for $i=1,2,\dots,K$ which determines its true ranking $\pi^*$ (see \eqref{eq:true_rank}) and $\btheta\sim\Dir(\balpha^0)$. In the setting of heterogeneous workers, we assume that there are $M$ crowd workers in total, denoted by $w=1,2,\dots,M$. If a pair of items $i$ and $j$ with $i<j$ is presented to the worker $w$, we denote the returned comparison result by a random variable $Y_{ij}^w$ such that
\begin{eqnarray}
\label{Yijw}
Y_{ij}^w=\left\{
\begin{array}{rl}
1&\text{ if item }i\text{ is preferred to item }j\text{ by worker }w\\
-1&\text{ if item }j\text{ is preferred to item }i\text{ by worker }w.
\end{array}
\right.
\end{eqnarray}


To model the reliability for workers, we introduce $M$ latent parameters $\brho=(\rho_1,\rho_2,\dots,\rho_M)^T$ of reliability with $\rho_w \in [0,1]$ for worker $w$ and assume $Y_{ij}^w$ has the following distribution
\begin{eqnarray}
\label{probijw1}
\text{Pr}(Y_{ij}^w=1)
&=&\rho_w\frac{\theta_i}{\theta_i+\theta_j}+(1-\rho_w)\frac{\theta_j}{\theta_i+\theta_j}\\\label{probijw2}
\text{Pr}(Y_{ij}^w=-1)
&=&\rho_w\frac{\theta_j}{\theta_i+\theta_j}+(1-\rho_w)\frac{\theta_i}{\theta_i+\theta_j}
\end{eqnarray}
for $1\leq i<j\leq K$ and $w=1,2,\dots,M$. This model can be viewed as a combination of Dawid-Skene model for categorical labeling tasks \citep{Dawid:79, Vikas:10, Oh:12} and Bradley-Terry-Luce (BTL) model, which was first introduced in \cite{Chen:13}. Such a mixture of BTL model is flexible and capable of modeling various types of workers. When $\rho_w=1$, the distribution in \eqref{probijw1} and \eqref{probijw2} reduces to \eqref{probij}, and we refer to worker $w$ with $\rho_w=1$ as a ``fully reliable" worker\footnote{We note that the full reliability does not imply that the worker is capable of identifying the latent scores of items and always give the correct comparison result, i.e., preferring the item with a higher latent score.  Instead, being fully reliable only means the worker tries her best to provide the preference after a careful consideration, and the inconsistency of comparisons among workers is mainly because the intrinsic ambiguity of the pair of items.}. Therefore,  the reliability parameter $\rho_w$ can be interpreted as the probability that worker $w$ behaves as a random fully reliable workers in the previous section, namely, the one whose preference over a pair $i$ and $j$ follows a distribution in accordance with the BTL model~\eqref{probij}. The worker with $\rho_w$ closer to 1 is considered to be more reliable while a worker with $\rho_w$ closer to 0 tends to be a poorly informed (or malicious) one who intentionally gives answers oppositive to the majority (truth).  Also, a worker is known as a spammer if the associated $\rho_w$ is near $0.5$ since this worker prefers $i$ or $j$ in any pair $i$ and $j$ with an equal probability regardless of their latent scores.

The reliability of each worker is unknown for the ranking task, which needs to be gradually identified during the comparison process. In the Bayesian framework, since the reliability parameter $\rho_w$ is supported on $[0,1]$, it can be naturally modeled to follow a Beta prior distribution, i.e., $\rho_w\sim\B(\mu_w^0,\nu_w^0)$, for $w=1,2,\dots,M$, where $\bmu^0=(\mu_1^0,\mu_2^0,\dots,\mu_M^0)$ and $\bnu^0=(\nu_1^0,\nu_2^0,\dots,\nu_M^0)$ are positive parameters.

\subsection{Bayesian Decision Process}

In this section, we model the sequential decision problem with a finite budget of $T$ in the setting of heterogeneous workers. Since the workers now have different levels of reliability, we can no longer randomly select a worker from the crowd in each stage. Instead, we need to adaptively determine not only which pair of items to be compared but also who should perform this comparison task according to the historical results so that the budget can be gradually shifted towards more reliable workers.

Suppose a pair of items $(i_t,j_t)$ with $i_t<j_t$ is compared by a worker $w_t$ in stage $t$ and the comparison result is  $Y_{i_tj_t}^{w_t}$ defined in \eqref{Yijw}. The historical comparison results up to stage $t$ can be summarized by a $K\times K\times M$ tensor $\mathbf{M}^t$, which is updated iteratively as follows. In particular, at each stage $t$, we define $\boldsymbol{\Delta}^t$ to be a sparse $K \times K \times M$ tensor with only non-zero element: if $Y_{i_tj_t}^{w_t}=1$, $\boldsymbol{\Delta}^t_{i_tj_tw_t}=1$ and if $Y_{i_tj_t}^{w_t}=-1$, $\boldsymbol{\Delta}^t_{j_ti_tw_t}=1$. Let
\begin{equation}
\label{Mtw}
\mathbf{M}^0=\mathbf{0},\quad
\mathbf{M}^{t+1}=\mathbf{M}^t+ \boldsymbol{\Delta}^t \quad \text{ for }t=0,1,\dots,T-1,
\end{equation}
where $\mathbf{0}$ is a $K \times K \times M$ all-zero tensor. In contrast to the matrix $M^t$ in \eqref{Mt}, each element in the tensor $\mathbf{M}^t$ takes the value either zero or one because each worker is not allowed to compare the same pair more than once. The dynamic budget allocation policy is denoted by $\mathcal{A}=\{(i_t,j_t,w_t)\}_{t=0,1,\dots,T-1}$ where $(i_t,j_t,w_t)=(i_t(\mathbf{M}^t),j_t(\mathbf{M}^t),w_t(\mathbf{M}^t))$ depends on the previous comparison results through $\mathbf{M}^t$. The posterior distributions of $\btheta$ and $\brho$ in stage $t$ are denoted by $p(\btheta|\mathbf{M}^t,\balpha^0,\bmu^0,\bnu^0)$ and $p(\brho|\mathbf{M}^t,\balpha^0,\bmu^0,\bnu^0)$, respectively.

Similar to the homogeneous worker setup, we adopt the Kendall's tau (\ref{kendalltau}) to measure the ranking accuracy. At each stage $t$, we denote the maximum posterior expected ranking accuracy by (with a slight abuse of notation)
\begin{eqnarray}
\label{maxhw}
h(\mathbf{M}^t) & \equiv & \max_{\pi} \mathbb{E}\left[\tau(\pi, \pi^*)|\mathbf{M}^t,\balpha^0,\bmu^0,\bnu^0\right] \\
& = &
\max_{\pi}\frac{2\sum_{i\neq j}\mathbf{1}_{\pi(i)>\pi(j)}\text{Pr}\left(\theta_i>\theta_j|\mathbf{M}^t,\balpha^0,\bmu^0,\bnu^0\right)}{K(K-1)}.
\nonumber
\end{eqnarray}
The maximizer in \eqref{maxhw} is the optimal ranking inferred from the historical comparison results up to the stage $t$. Our goal is to search for the optimal policy $\mathcal{A}$ that maximizes the final expected ranking accuracy $h(\mathbf{M}^T)$, i.e.,
\begin{eqnarray}
\label{maxAw}
\max_{\mathcal{A}} \mathbb{E}^{\mathcal{A}}\left[h(\mathbf{M}^T)|\balpha^0,\bmu^0,\bnu^0\right].
\end{eqnarray}
This maximization problem can be further reformulated in a telescopic sum
\begin{eqnarray}
\label{maxAwsum}
h(\mathbf{M}^0)+\max_{\mathcal{A}}\mathbb{E}\left[\sum_{t=0}^{T-1}\mathbb{E}\left[R(\mathbf{M}^t,i_t,j_t,w_t,Y_{i_tj_t}^{w_t})\big|\mathbf{M}^t,\balpha^0,\bmu^0,\bnu^0\right]\bigg|\balpha^0,\bmu^0,\bnu^0\right],
\end{eqnarray}
where
\begin{eqnarray}
\label{Rw}
R(\mathbf{M}^t,i_t,j_t,w_t,Y_{i_tj_t}^{w_t})\equiv h(\mathbf{M}^{t+1})-h(\mathbf{M}^t),
\end{eqnarray}
is the \emph{stage-wise reward} depending on $\mathbf{M}^t$,$i_t$,$j_t$,$w_t$ and $Y_{i_tj_t}^{w_t}$. It can be interpreted as the improvement of the expected ranking accuracy after receiving the comparison result at stage $t$.  The \emph{state variable} of the MDP \eqref{maxAw} or \eqref{maxAwsum} is the tensor $\mathbf{M}^t$ which evolves according to \eqref{Mtw} and the state space at each $t$ is
\[
\mathcal{S}^t=\Bigl\{\mathbf{M} \in \{0,1\}^{K \times K \times M}: \sum_{i,j,w} \mathbf{M}_{ijw}=t\Bigr\}.
\]
The \emph{expected transition probabilities} of MDP \eqref{maxAw} are
\begin{align}
\label{Eprobij1w}
\mathbb{E}\left[\text{Pr}(Y_{ij}^w=1)|\mathbf{M}^t,\balpha^0,\bmu^0,\bnu^0\right] &=\mathbb{E}\left[ \rho_w\frac{\theta_i}{\theta_i+\theta_j}+(1-\rho_w)\frac{\theta_j}{\theta_i+\theta_j}\big|\mathbf{M}^t,\balpha^0,\bmu^0,\bnu^0\right]\\\label{Eprobij0w}
\mathbb{E}\left[\text{Pr}(Y_{ij}^w=-1)|\mathbf{M}^t,\balpha^0,\bmu^0,\bnu^0\right] &=\mathbb{E}\left[ \rho_w\frac{\theta_j}{\theta_i+\theta_j}+(1-\rho_w)\frac{\theta_i}{\theta_i+\theta_j}\big|\mathbf{M}^t,\balpha^0,\bmu^0,\bnu^0\right]
\end{align}
for $i,j=1,2,\dots,K$ and $w=1,2,\dots,M$. So far, we have modeled the sequential budget allocation in the heterogeneous worker setting as a Bayesian MDP. Due to the similar reasons that have been explained in Section \ref{sec:MDP}, although the dynamic programming can be directly applied to solve the Bayesian MDP and obtain the optimal policy, it is computationally intractable.   In fact, the Bayesian MDP \eqref{maxAwsum} is even more challenging to solve than that for the homogeneous worker setting due to a much larger state space after introducing the reliability of workers. In the next subsection, we will propose a computationally efficient approximated knowledge gradient policy for \eqref{maxAwsum}.


\subsection{Approximated Knowledge Gradient Policy}

To solve the Bayesian MDP \eqref{maxAwsum},  we still consider the family of knowledge gradient (KG) policies. In our problem, the KG policy will select the pair of items and the worker that together give the highest expected stage-wise reward. In particular, at the $t$-stage, the KG policy for \eqref{maxAwsum} will choose the pair $(i_t,j_t)$ and the worker $w_t$ such that
\begin{eqnarray}
\label{KGw}
(i_t,j_t,w_t)&\in&\argmax_{i<j,w}\mathbb{E}\left[R(\mathbf{M}^t,i,j,w,Y_{ij}^{w})\big|\mathbf{M}^t,\balpha^0,\bmu^0,\bnu^0\right]\\\nonumber
&=&\argmax_{i<j,w}\bigg\{
\mathbb{E}\left[\text{Pr}(Y_{ij}^w=1)|\mathbf{M}^t,\balpha^0,\bmu^0,\bnu^0\right] R(\mathbf{M}^t,i,j,w,1)\\\nonumber
&&\quad\quad\quad\quad+
\mathbb{E}\left[\text{Pr}(Y_{ij}^w=-1)|\mathbf{M}^t,\balpha^0,\bmu^0,\bnu^0\right]R(\mathbf{M}^t,i,j,w,-1)\bigg\}.
\end{eqnarray}

To implement the KG policy~\eqref{KGw}, we encounter the same difficulties as when we implemented~\eqref{KG}. Specifically, since the posterior distributions $p(\btheta|\mathbf{M}^t,\balpha^0,\bmu^0,\bnu^0)$ and $p(\brho|\mathbf{M}^t,\balpha^0,\bmu^0,\bnu^0)$ are sophisticated and the MAX-LOP problem \eqref{maxhw} is NP-hard, we cannot efficiently evaluate the stage-wise reward~\eqref{Rw} and the transition probabilities \eqref{Eprobij1w} and \eqref{Eprobij0w}. To obtain a computationally efficient policy, we follow the techniques in Section \ref{sec:AKG} 
to approximate the posterior distributions $p(\btheta|\mathbf{M}^t,\balpha^0,\bmu^0,\bnu^0)$ and $p(\rho_w|\mathbf{M}^t,\balpha^0,\bmu^0,\bnu^0)$ recursively using a sequence of Dirichlet distributions $\Dir(\balpha^t)$ and a sequence of beta distributions $\B(\mu^t_w,\nu^t_w)$, respectively, for $w=1,2,\dots,M$ and $t=1,2,\dots,T$. The parameters $\balpha^t(\alpha_1^t,\alpha_2^t,\dots,\alpha_K^t)$, $\bmu^t=(\mu_1^t,\mu_2^t,\dots,\mu_M^t)$ and $\bnu^t=(\nu_1^t,\nu_2^t,\dots,\nu_M^t)$ will be chosen recursively based on moment matching.

Suppose $\btheta\sim\Dir(\balpha)$ for some parameter vector $\balpha\in\mathbb{R}^K$ and $\rho_w\sim\B(\mu_w,\nu_w)$ for each $w$ with $\bmu=(\mu_1,\mu_2,\dots,\mu_M)$ and $\bnu=(\nu_1,\nu_2,\dots,\nu_M)$. We consider a basic scenario where only one comparison result $Y_{ij}^w$ from worker $w$ for a pair $(i,j)$ has been observed. We can approximate $p(\btheta|Y_{ij}^w,\balpha,\bmu,\bnu)$ by a Dirichlet distribution $\Dir(\balpha')$ and $p(\rho_w|Y_{ij}^w,\balpha,\bmu,\bnu)$ by a Beta distribution $\B(\mu'_w,\nu'_w)$ such that
\begin{eqnarray}
\label{eq:moment_matching1w}
\mathbb{E} \left[\theta_k |\btheta\sim\Dir(\balpha')\right]&=&\mathbb{E}[\theta_k|Y_{ij}^w,\balpha,\bmu,\bnu]\text{ for }k=1,2,\dots,K\\\label{eq:moment_matching2w}
\mathbb{E} \left[\sum_{k=1}^K\theta_k^2 |\btheta\sim\Dir(\balpha')\right] &=&\mathbb{E}\left[ \sum_{k=1}^K\theta_k^2| Y_{ij}^w,\balpha,\bmu,\bnu\right]\\\label{eq:moment_matching3w}
\mathbb{E} \left[\rho_w|\rho_w\sim\B(\mu'_w,\nu'_w)\right] &=&\mathbb{E}\left[ \rho_w| Y_{ij}^w,\balpha,\bmu,\bnu\right]\\\label{eq:moment_matching4w}
\mathbb{E} \left[\rho_w^2+(1-\rho_w)^2 |\rho_w\sim\B(\mu'_w,\nu'_w)\right] &=&\mathbb{E}\left[ \rho_w^2+(1-\rho_w)^2| Y_{ij}^w,\balpha,\bmu,\bnu\right].
\end{eqnarray}
Note that we do not need to approximate $p(\rho_{w'}|Y_{ij}^w,\balpha,\bmu,\bnu)$  for $w'\neq w$ since the worker $w'$ has not performed any comparison so that $p(\rho_{w'}|Y_{ij}^w,\balpha,\bmu,\bnu)$ is still the prior distribution $\B(\mu_{w'},\nu_{w'})$. This system of equations has the following explicit characterization.
\begin{proposition}
\label{propMMw}
Suppose $\btheta\sim\Dir(\balpha)$ and $\rho_w\sim\B(\mu_w,\nu_w)$ for worker $w$ and $Y_{ij}^w$ is the only comparison result. Let $\alpha_0=\sum_{k=1}^K\alpha_k$ and $\alpha'_0=\sum_{k=1}^K\alpha'_k$. The equations \eqref{eq:moment_matching1w}, \eqref{eq:moment_matching2w}, \eqref{eq:moment_matching3w} and \eqref{eq:moment_matching4w} can be represented as
\begin{eqnarray}
\label{MM1w}
\left\{
\begin{array}{rcl}
\frac{\alpha'_i}{\alpha'_0}&=&\eta_{ijw}\frac{(\alpha_i+1)(\alpha_i+\alpha_j)}{\alpha_0(\alpha_i+\alpha_j+1)}
+ (1-\eta_{ijw})\frac{\alpha_i(\alpha_i+\alpha_j)}{\alpha_0(\alpha_i+\alpha_j+1)}\\
\frac{\alpha'_j}{\alpha'_0}&=&\eta_{ijw}\frac{\alpha_j(\alpha_i+\alpha_j)}{\alpha_0(\alpha_i+\alpha_j+1)}
+ (1-\eta_{ijw})\frac{(\alpha_j+1)(\alpha_i+\alpha_j)}{\alpha_0(\alpha_i+\alpha_j+1)}\\
\frac{\alpha'_k}{\alpha'_0}&=& \frac{\alpha_k}{\alpha_0}\quad\text{ for }k\neq i,j\\
\sum_{k=1}^K\frac{\alpha'_k(\alpha'_k+1)}{\alpha'_0(\alpha'_0+1)}&=&  \eta_{ijw}\frac{(\alpha_i+1)(\alpha_i+2)(\alpha_i+\alpha_j)}{\alpha_0(\alpha_0+1)(\alpha_i+\alpha_j+2)}
+(1-\eta_{ijw})\frac{\alpha_i(\alpha_i+1)(\alpha_i+\alpha_j)}{\alpha_0(\alpha_0+1)(\alpha_i+\alpha_j+2)}\\
&&+\eta_{ijw}\frac{\alpha_j(\alpha_j+1)(\alpha_i+\alpha_j)}{\alpha_0(\alpha_0+1)(\alpha_i+\alpha_j+2)}
+(1-\eta_{ijw})\frac{(\alpha_j+1)(\alpha_j+2)(\alpha_i+\alpha_j)}{\alpha_0(\alpha_0+1)(\alpha_i+\alpha_j+2)}\\
&&+\sum_{k\neq i,j}\frac{\alpha_k(\alpha_k+1)}{\alpha_0(\alpha_0+1)}\\
\frac{\mu'_w}{\mu'_w+\nu'_w}&=&\eta_{ijw}\frac{\mu_w+(1+Y_{ij}^w)/2}{\mu_w+\nu_w+1}
+(1-\eta_{ijw})\frac{\mu_w+(1-Y_{ij}^w)/2}{\mu_w+\nu_w+1}\\
\frac{\mu'_w(\mu'_w+1)+\nu'_w(\nu'_w+1)}{(\mu'_w+\nu'_w)(\mu'_w+\nu'_w+1)}
&=&\eta_{ijw}\frac{(\mu_w+(1+Y_{ij}^w)/2)(\mu_w+(3+Y_{ij}^w)/2)}{(\mu_w+\nu_w+1)(\mu_w+\nu_w+2)}
\\ && +(1-\eta_{ijw})\frac{(\mu_w+(1-Y_{ij}^w)/2)(\mu_w+(3-Y_{ij}^w)/2)}{(\mu_w+\nu_w+1)(\mu_w+\nu_w+2)}
\\
&&+\eta_{ijw}\frac{(\nu_w+(1-Y_{ij}^w)/2)(\nu_w+(3-Y_{ij}^w)/2)}{(\mu_w+\nu_w+1)(\mu_w+\nu_w+2)}\\
&&
+(1-\eta_{ijw})\frac{(\nu_w+(1+Y_{ij}^w)/2)(\nu_w+(3+Y_{ij}^w)/2)}{(\mu_w+\nu_w+1)(\mu_w+\nu_w+2)}.
\end{array}
\right.
\end{eqnarray}
where $\eta_{ijw} = \frac{[(1+Y_{ij}^w)\mu_w+(1-Y_{ij}^w)\nu_w]\alpha_i}
{[(1+Y_{ij}^w)\mu_w+(1-Y_{ij}^w)\nu_w]\alpha_i+[(1+Y_{ij}^w)\nu_w+(1-Y_{ij}^w)\mu_w]\alpha_j}$.
\end{proposition}

The proof of Proposition \ref{propMMw} is given in Appendix. We denote any $\balpha'$, $\mu'_w$ and $\nu'_w$ that satisfy \eqref{eq:moment_matching1w}, \eqref{eq:moment_matching2w}, \eqref{eq:moment_matching3w} and \eqref{eq:moment_matching4w}, and thus \eqref{MM1w}, by
\begin{eqnarray}
\label{defMM1w}
\balpha'=\textbf{MM}_{\alpha}(\balpha,i,j,w,Y_{ij}^w)\quad\text{ and } \quad (\mu'_w,\nu'_w)=\textbf{MM}_{\mu\nu}(\balpha,i,j,w,Y_{ij}^w).
\end{eqnarray}
Although the equations in Proposition \ref{propMMw} are more complicated than those in Proposition \ref{propMM}, the right-hand sides of \eqref{MM1w} are still constants for any given $i$, $j$, $w$, $Y_{ij}^w$, $\balpha$, $\mu_w$ and $\nu_w$ so that both $\balpha'=\textbf{MM}_{\alpha}(\balpha,i,j,w,Y_{ij}^w)$ and $(\mu'_w,\nu'_w)=\textbf{MM}_{\mu\nu}(\balpha,i,j,w,Y_{ij}^w)$ can be solved in a closed form.  In fact, we denote the constants on the right-hand sides of \eqref{MM1} as $C_i$, $C_j$, $C_k$ (for $k\neq i ,j$), $D$, $E$ and $F$, respectively. It is easy to see that $\sum_{k=1}^KC_k=1$.
By the same derivation for \eqref{MM1form}, we obtain the following closed form for $\balpha'=\textbf{MM}_{\alpha}(\balpha,i,j,w,Y_{ij}^w)$
\begin{eqnarray}
\label{MM1wform1}
\alpha'_0=\frac{D-1}{\sum_{k=1}^KC_k^2-D}\quad\text{ and }\quad \alpha'_k=C_k\alpha'_0\text{ for }k=1,2,\dots,K,
\end{eqnarray}
which takes the same form as \eqref{MM1form} but with the constants $C_k$ for $k=1,2,\dots,K$ defined differently (which involve the information of worker $w$, i.e., $\mu_w$ and $\nu_w$). Similarly, solving $\mu'_w$ and $\nu'_w$ from the last two equations in \eqref{MM1w}, we obtain the following closed form for $(\mu'_w,\nu'_w)=\textbf{MM}_{\mu\nu}(\balpha,i,j,w,Y_{ij}^w)$
\begin{eqnarray}
\label{MM1wform2}
\mu'_w=\frac{(F-1)E}{E^2+(1-E)^2-F}\quad\text{ and }\quad \nu'_w=\frac{(F-1)(1-E)}{E^2+(1-E)^2-F}.
\end{eqnarray}

Although the approximate scheme above is derived when there is only one comparison result, it generates a Dirichlet distribution $\Dir(\balpha')$ for $\btheta$ and a Beta distribution $\B(\mu'_w,\nu'_w)$ for $\rho_w$ and does not change the Beta distribution $\B(\mu_{w'},\nu_{w'})$ for $w'\neq w$. The fact that the approximated posteriors take the same form as the priors suggests that we can apply this approximation scheme iteratively to approximate $p(\btheta|\mathbf{M}^t,\balpha^0,\bmu^0,\bnu^0)$ and $p(\brho|\mathbf{M}^t,\balpha^0,\bmu^0,\bnu^0)$ for any given policy $\mathcal{A}=\{(i_t,j_t,w_t)\}_{t=0,1,\dots,T-1}$. In particular, let $\balpha^t$, $\bmu^t$ and $\bnu^t$ be the sequences of parameters generated recursively as follows
\begin{eqnarray}
\label{alphatw}
\balpha^{t+1}&=&\textbf{MM}_{\alpha}(\balpha^t,i_t,j_t,w_t,Y_{i_tj_t}^{w_t})\\\label{munutw}
(\mu^{t+1}_w,\nu^{t+1}_w)&=&\left\{
\begin{array}{ll}
\textbf{MM}_{\mu\nu}(\balpha^t,i_t,j_t,w_t,Y_{i_tj_t}^{w_t})&\text{ if }w=w_t\\
(\mu^{t}_w,\nu^{t}_w)&\text{ if }w\neq w_t
\end{array}
\right.
\end{eqnarray}
for $t=1,2,\dots,T$.
The posterior distributions $p(\btheta|\mathbf{M}^t,\balpha^0,\bmu^0,\bnu^0)$ and $p(\brho|\mathbf{M}^t,\balpha^0,\bmu^0,\bnu^0)$ can be approximated by $\Dir(\balpha^t)$ and $\Pi_{w=1,\ldots, M}\B(\mu^t_w,\nu^t_w)$, respectively.

Following the same strategy as in \eqref{approxprob1} and \eqref{approxprob2}, we can approximate \eqref{Eprobij1w} and \eqref{Eprobij0w} as
\begin{eqnarray}
\label{approxprob1w}
&&\mathbb{E}\left[\text{Pr}(Y_{ij}^w=1)\big|\mathbf{M}^t,\balpha^0,\bmu^0,\bnu^0\right]\\\nonumber
&\approx&\mathbb{E}\left[ \rho_w\frac{\theta_i}{\theta_i+\theta_j}+(1-\rho_w)\frac{\theta_j}{\theta_i+\theta_j}\big|\btheta\sim\Dir(\balpha^t),\rho_w\sim\B(\mu^t_w,\nu^t_w)\right]\\\nonumber
&=&\frac{\mu^t_w}{\mu^t_w+\nu^t_w}\frac{\alpha^t_i}{\alpha^t_i+\alpha^t_j}+\frac{\nu^t_w}{\mu^t_w+\nu^t_w}\frac{\alpha^t_j}{\alpha^t_i+\alpha^t_j}
\end{eqnarray}
and
\begin{eqnarray}
\label{approxprob2w}
&&\mathbb{E}\left[\text{Pr}(Y_{ij}^w=-1)\big|\mathbf{M}^t,\balpha^0,\bmu^0,\bnu^0\right]\\\nonumber
&\approx&\mathbb{E}\left[ \rho_w\frac{\theta_j}{\theta_i+\theta_j}+(1-\rho_w)\frac{\theta_i}{\theta_i+\theta_j}\big|\btheta\sim\Dir(\balpha^t),\rho_w\sim\B(\mu^t_w,\nu^t_w)\right]\\\nonumber
&=&\frac{\mu^t_w}{\mu^t_w+\nu^t_w}\frac{\alpha^t_j}{\alpha^t_i+\alpha^t_j}+\frac{\nu^t_w}{\mu^t_w+\nu^t_w}\frac{\alpha^t_i}{\alpha^t_i+\alpha^t_j}
\end{eqnarray}
and approximate
$\text{Pr}\left(\theta_i>\theta_j|\mathbf{M}^t,\balpha^0,\bmu^0,\bnu^0\right)$ in \eqref{maxhw} as
\begin{eqnarray}
\label{approxprobw}
\text{Pr}\left(\theta_i>\theta_j|\mathbf{M}^t,\balpha^0,\bmu^0,\bnu^0\right)
\approx\text{Pr}\left(\theta_i>\theta_j|\btheta\sim\Dir(\balpha^t)\right)
=I_{\frac{1}{2}}(\alpha^t_j,\alpha^t_i).
\end{eqnarray}
The approximation \eqref{approxprobw} helps to simplify the NP-hard MAX-LOP in \eqref{maxhw} as
$$
\max_{\pi} \mathbb{E}\left[\tau(\pi, \pi^*)|\mathbf{M}^t,\balpha^0,\bmu^0,\bnu^0\right]\approx
\max_{\pi} \mathbb{E}\left[\tau(\pi, \pi^*)|\btheta\sim\Dir(\balpha^t)\right],
$$
where the right-hand side can be solved easily by sorting of the components of $\balpha^t$ according to
Theorem \ref{thmLOP}.

Similar to \eqref{approxreward}, the stage-wise reward is approximated as
\begin{eqnarray}
\nonumber
R(\mathbf{M}^t,i,j,w,Y_{ij}^w)&=&\max_{\pi} \mathbb{E}\left[\tau(\pi, \pi^*)|\mathbf{M}^{t+1},\balpha^0,\bmu^0,\bnu^0\right]-\max_{\pi} \mathbb{E}\left[\tau(\pi, \pi^*)|\mathbf{M}^t,\balpha^0,\bmu^0,\bnu^0\right]\\\nonumber
&\approx&\max_{\pi} \mathbb{E}\left[\tau(\pi, \pi^*)|\btheta\sim\Dir(\hat\balpha)\right]-\max_{\pi} \mathbb{E}\left[\tau(\pi, \pi^*)|\btheta\sim\Dir(\balpha^t)\right]\\\nonumber
&=& \frac{2}{K(K-1)}\left(\sum_{\pi_{\hat\balpha}(i')>\pi_{\hat\balpha}(j')}I_{\frac{1}{2}}(\hat\alpha_{j'},\hat\alpha_{i'})
-\sum_{\pi_{\balpha^t}(i')>\pi_{\balpha^t}(j')}I_{\frac{1}{2}}(\alpha^t_{j'},\alpha^t_{i'})\right)\\\label{approxrewardw}
&\equiv&\tilde R(\balpha^t,i,j,w,Y_{ij}^w)
\end{eqnarray}
where
$\hat\balpha=\textbf{MM}_{\alpha}(\balpha^t,i,j,w,Y_{ij}^w)$. Putting \eqref{approxprob1w}, \eqref{approxprob2w}, \eqref{approxprobw} and \eqref{approxrewardw} together, we can approximate the expected stage-wise reward $\mathbb{E}\left[R(\mathbf{M}^t,i,j,w,Y_{ij}^w)\big|M^t,\balpha^0,\bmu^0,\bnu^0\right]$ as
\begin{eqnarray}
\nonumber
&&\mathbb{E}\left[R(\mathbf{M}^t,i,j,w,Y_{ij}^w)\big|M^t,\balpha^0,\bmu^0,\bnu^0\right]\\\nonumber
&=&\mathbb{E}\left[\text{Pr}(Y_{ij}^w=1)\big|\mathbf{M}^t,\balpha^0,\bmu^0,\bnu^0\right] R(\balpha^t,i,j,w,1) \nonumber \\
&& +
\mathbb{E}\left[\text{Pr}(Y_{ij}^w=-1)\big|\mathbf{M}^t,\balpha^0,\bmu^0,\bnu^0\right]R(\balpha^t,i,j,w,-1)\nonumber \\
&\approx&\left(\frac{\mu^t_w}{\mu^t_w+\nu^t_w}\frac{\alpha^t_i}{\alpha^t_i+\alpha^t_j}+\frac{\nu^t_w}{\mu^t_w+\nu^t_w}\frac{\alpha^t_j}{\alpha^t_i+\alpha^t_j}\right) \tilde R(\balpha^t,i,j,w,1) \nonumber \\
&&+
\left(\frac{\mu^t_w}{\mu^t_w+\nu^t_w}\frac{\alpha^t_j}{\alpha^t_i+\alpha^t_j}+\frac{\nu^t_w}{\mu^t_w+\nu^t_w}\frac{\alpha^t_i}{\alpha^t_i+\alpha^t_j}\right) \tilde R(\balpha^t,i,j,w,-1). \label{AKGw}
\end{eqnarray}
When the workers have various levels of reliability, our AKG policy will choose the pair $(i_t,j_t)$ and present it to worker $w_t$ so that \eqref{AKGw} is maximized. The AKG policy for the setting of heterogeneous workers is formally presented as Algorithm~\ref{alg:AKGw}. Note that when $\rho_w = 1$ for all $w$, we do not need to solve \eqref{eq:moment_matching3w} and \eqref{eq:moment_matching4w} anymore and thus the rest of the problem reduces to the homogeneous setting.

\begin{algorithm}[!t]
\caption{Approximated Knowledge Gradient Policy with Heterogeneous Workers}
\begin{algorithmic}[1]\label{alg:AKGw}
  \item[\textbf{Initialization:}] Choose $\balpha^0$, $\bmu^0$ and $\bnu^0$ for the prior distributions.

  \item[\textbf{For}] $t = 0,\dots,T-1$ \textbf{do}

  \STATE For each pair $(i,j)$ with $i<j$, compute $\tilde R(\balpha^t,i,j,w,1)$ and $\tilde R(\balpha^t,i,j,w,-1)$ according to \eqref{approxrewardw}.
  \STATE Select $(i_t,j_t,w_t)$ such that
  \small
   \begin{eqnarray}\label{eq:single_stage_AGKw}
  (i_t,j_t)&\in&\argmax_{i<j, \;w}\Bigg[\left(\frac{\mu^t_w}{\mu^t_w+\nu^t_w}\frac{\alpha^t_i}{\alpha^t_i+\alpha^t_j}+\frac{\nu^t_w}{\mu^t_w+\nu^t_w}\frac{\alpha^t_j}{\alpha^t_i+\alpha^t_j}\right) \tilde R(\balpha^t,i,j,w,1)\\
  &&\quad\quad\quad\quad\quad+\left(\frac{\mu^t_w}{\mu^t_w+\nu^t_w}\frac{\alpha^t_j}{\alpha^t_i+\alpha^t_j}+\frac{\nu^t_w}{\mu^t_w+\nu^t_w}\frac{\alpha^t_i}{\alpha^t_i+\alpha^t_j}\right) \tilde R(\balpha^t,i,j,w,-1) \nonumber
  \Bigg]
 \end{eqnarray}
\normalsize
and present item $i_t$ and item $j_t$ to worker $w_t$ and receive the comparison result $Y_{i_tj_t}^{w_t}$.
  \STATE According to \eqref{defMM1w}, \eqref{MM1wform1} and \eqref{MM1wform2}, compute
  \small
  \begin{eqnarray}
  \label{eq:AKGworkeralpha}
\balpha^{t+1}&=&\textbf{MM}_{\alpha}(\balpha^t,i_t,j_t,w_t,Y_{i_tj_t}^{w_t})\\
\label{eq:AKGworkeruv}
(\mu^{t+1}_w,\nu^{t+1}_w)&=&\left\{
\begin{array}{ll}
\textbf{MM}_{\mu\nu}(\balpha^t,i_t,j_t,w_t,Y_{i_tj_t}^{w_t})&\text{ if }w=w_t\\
(\mu^{t}_w,\nu^{t}_w)&\text{ if }w\neq w_t
\end{array}
\right.
\end{eqnarray}
\normalsize
  \item[\textbf{End For}]
  \item[\textbf{Return:}] The aggregated ranking $\pi_{\balpha^T}$ obtained by sorting the components of $\balpha^T$.
\end{algorithmic}
\end{algorithm}



\section{Experiment}
\label{sec:exp}

In this section, we conduct empirical studies using both simulated and real data. We compare the proposed AKG algorithms to some existing methods in terms of ranking accuracy versus different levels of budget as well as computation time. We also show some interesting properties of the proposed AKG policies, e.g., how budget will be allocated over pairs of items with different levels of ambiguity and workers with different levels of reliability. The ranking accuracy is evaluated using the Kendall's tau as defined in \eqref{kendalltau}.

\subsection{Simulated Study under the Homogeneous Workers Setting}

In this section, we assume that all workers are fully reliable and investigate the performance of the AKG policy (Algorithm \ref{alg:AKG}). Two scenarios are designed: 10 items with a total budget of 100, and 100 items with a total budget of 1000. Each scenario consists of 100 independent trials and the average ranking accuracy is reported. 
For each trial, the latent item scores $\btheta$ is sampled uniformly from the simplex in \eqref{eq:simplex}, which determines the true ranking $\pi^*$. Given $\btheta$, the comparison results are generated according to the Bradley-Terry-Luce model~\eqref{probij}. We compare several different methods, including the proposed AKG, random sampling (uniformly random sampling), distance-based sampling, adaptive polling~\citep{Pfeiffer:12} and rank centrality with uniform sampling or knowledge gradient sampling~\citep{nos12}. The details of the methods are provided as follows. 

\begin{enumerate}
\item \textbf{AKG} (see Algorithm \ref{alg:AKG}): We set the prior of $\btheta$ to be the  uniform distribution on the simplex (i.e., $\balpha^0$ is set to be an all-one vector).

\item \textbf{Random Sampling}: The random sampling algorithm is similar to Algorithm \ref{alg:AKG} in terms of the posterior approximation (by moment matching) and rank inference (by sorting the approximated posterior parameters $\balpha^t$) after receiving each label. The only difference is that this algorithm replaces Step 2 of Algorithm \ref{alg:AKG} by a random sampling policy, which selects $(i_t,j_t)$ randomly at each stage. We also choose the uniform distribution on the simplex as the prior.

\item \textbf{Distance-Based Sampling}: This algorithm is also the same as Algorithm \ref{alg:AKG} in terms of the posterior approximation. However, in the sampling phase, this algorithm simply selects the pair of items $(i_t,j_t)$ with the closest posterior parameters $\alpha^t_i$ and $\alpha^t_j$. We choose the uniform distribution on the simplex as the prior.

\item \textbf{Adaptive Polling}: This is a greedy policy proposed by \cite{Pfeiffer:12}, which chooses the pair of items to maximize the KL-divergence between the posterior and prior. The initial $K  \times K$ matrix $M$ used in adaptive polling is set to 0 on the diagonal and 0.15 everywhere else.


\item \textbf{Rank Centrality}: This is a static rank aggregation algorithm recently proposed by \cite{nos12}. We combine it with both the random sampling policy and the knowledge gradient policy. Specifically, for \textbf{Centrality + RS}, we randomly select a pair of items at each stage and infer the true ranking using rank centrality. For \textbf{Centrality + KG}, we select the next pair of items using AKG policy, but estimate the ranking using rank centrality.

\end{enumerate}

It is worthwhile to point out that we are able to compute the optimal policy exactly only up to the 4-item case, which is not interesting from the ranking perspective and thus is left out from the experiment.

\begin{figure*}[!t]
    \centering
    \begin{subfigure}[h]{0.5\textwidth}
        \centering
        \includegraphics[width=\textwidth]{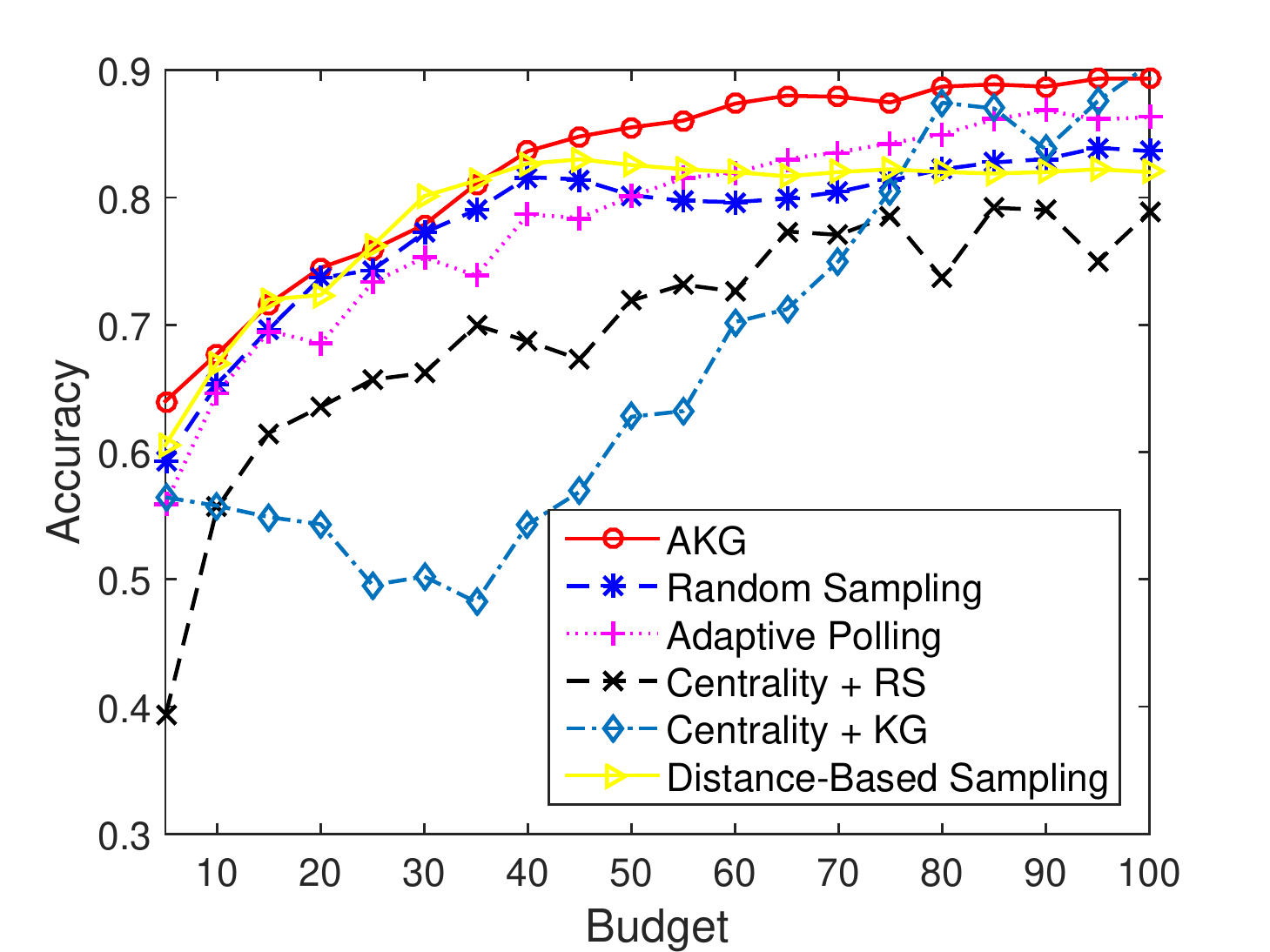}
        \caption{10 Items}
    \end{subfigure}%
	~
    \begin{subfigure}[h]{0.5\textwidth}
        \centering
        \includegraphics[width=\textwidth]{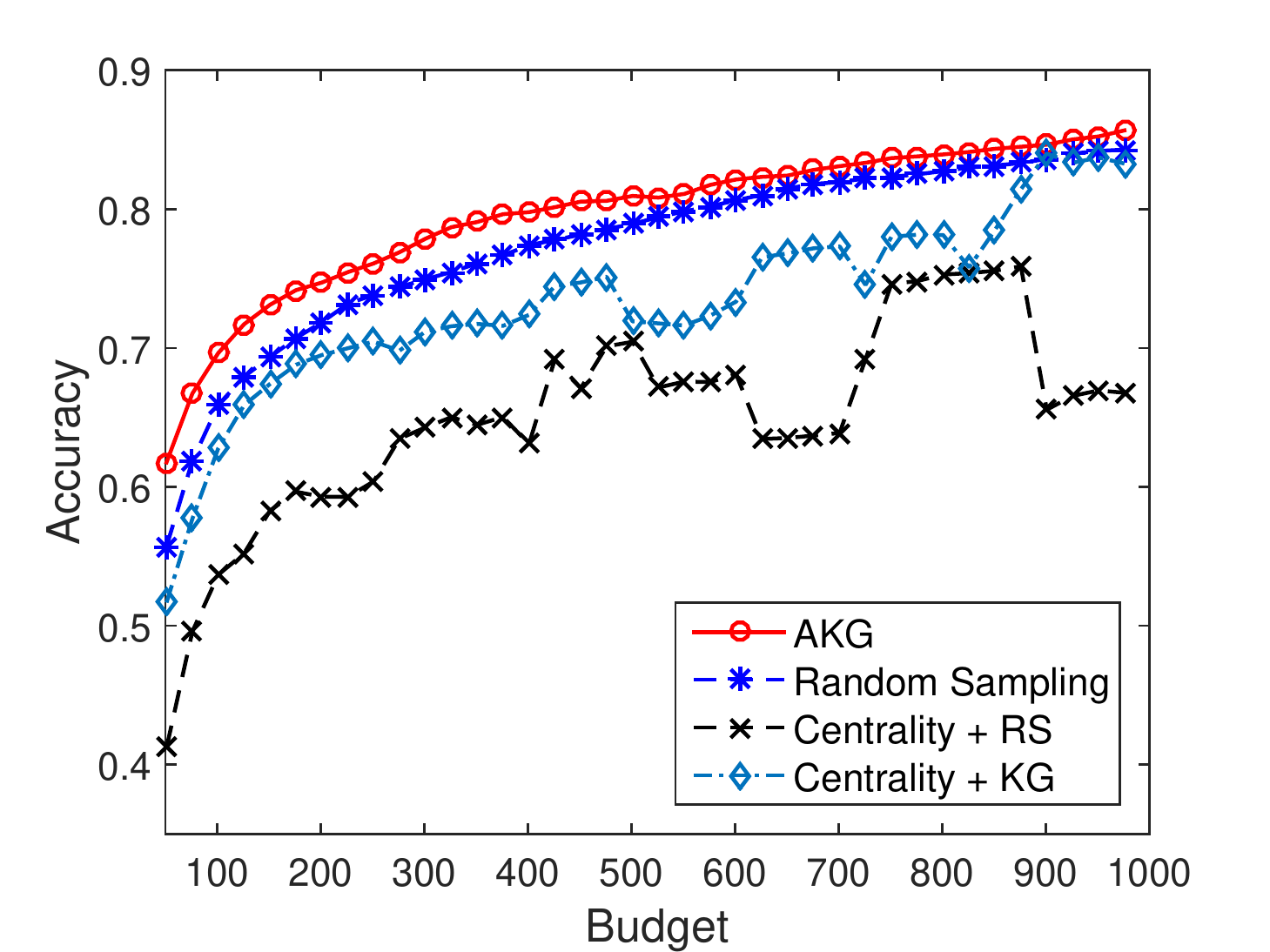}
        \caption{100 Items}
        \label{fig:homo_100}
    \end{subfigure}
    \caption{Performance comparison under the homogeneous workers setting. The $x$-axis is the budget level and $y$-axis is the averaged ranking accuracy.}
    \label{fig:homogeneous}
\end{figure*}

As we can see from Figure \ref{fig:homogeneous}, the AKG policy has higher accuracy than other methods at all budget levels. Note that the average accuracy of AKG surpasses the level of 70\% with only 20 pairs in the case of 10 items. In general, random sampling has similar performance as AKG at the beginning, but eventually AKG will outperform random sampling as it will spend more budget on the ambiguous pairs. This will be verified in the next experiment. Meanwhile, if we combine rank centrality with knowledge gradient sampling, the performance of the algorithm can be boosted significantly. Furthermore, the curves of ranking accuracy of AKG are in general monotonically increasing and have fewer ``bumps'' than other algorithms. This implies that the sequence of posterior parameters $\balpha^t$ is quite stable when the budget level becomes larger. We also note that due to the high computational cost of adaptive polling, it takes extremely long time when the number of items is 100 and thus we omit its performance in Figure \ref{fig:homo_100}.

It is worthwhile to note that AKG runs significantly faster than the adaptive polling method. It enjoys the advantage of closed-form updating rule during each iteration/stage without using a numerical algorithm as a subroutine, which is a good feature for online applications.  In contrast, adaptive polling is much slower because it requires inverting a $K \times K$ matrix for all $O(K^2)$ possible pairs and all possible comparison results in each iteration. Table \ref{tab:comp} gives the computation time of a \emph{single iteration} for both AKG and adaptive polling. Note that in the 25-item case, the computation time for adaptive polling of a single iteration has already exceeded 40 minutes. Therefore, we omit to present the computation time of adaptive polling when the number of items is 100 in Table \ref{tab:comp} since each iteration/stage would take hours to run.

\begin{table}[!t]
\centering
    \caption{Comparison in computation time under the homogeneous workers setting.}
    \begin{tabular}{ | c | c | c |}
    \hline
    No. of Items & AKG & Adaptive Polling \\ \hline
    10 & 0.023 sec & 20 sec \\ \hline
    25 & 0.75 sec &  42 min \\ \hline
    100 & 22 sec & - \\ \hline
    \end{tabular}
    \label{tab:comp}
\end{table}

\begin{figure*}[!t]
    \centering
    \begin{subfigure}[h]{0.5\textwidth}
        \centering
        \includegraphics[width=\textwidth]{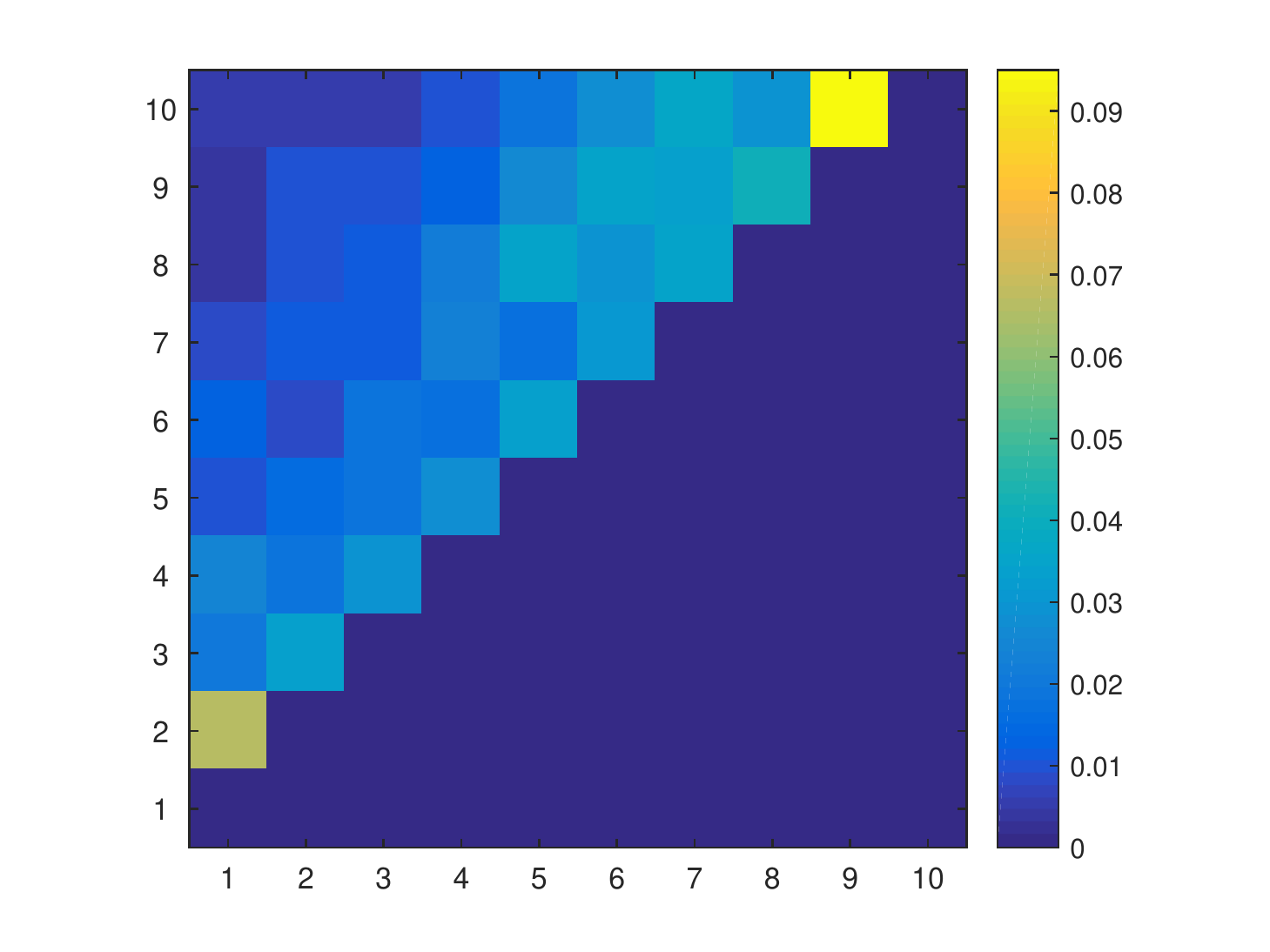}
        \caption{10 Items}
    \end{subfigure}%
	~
    \begin{subfigure}[h]{0.5\textwidth}
        \centering
        \includegraphics[width=\textwidth]{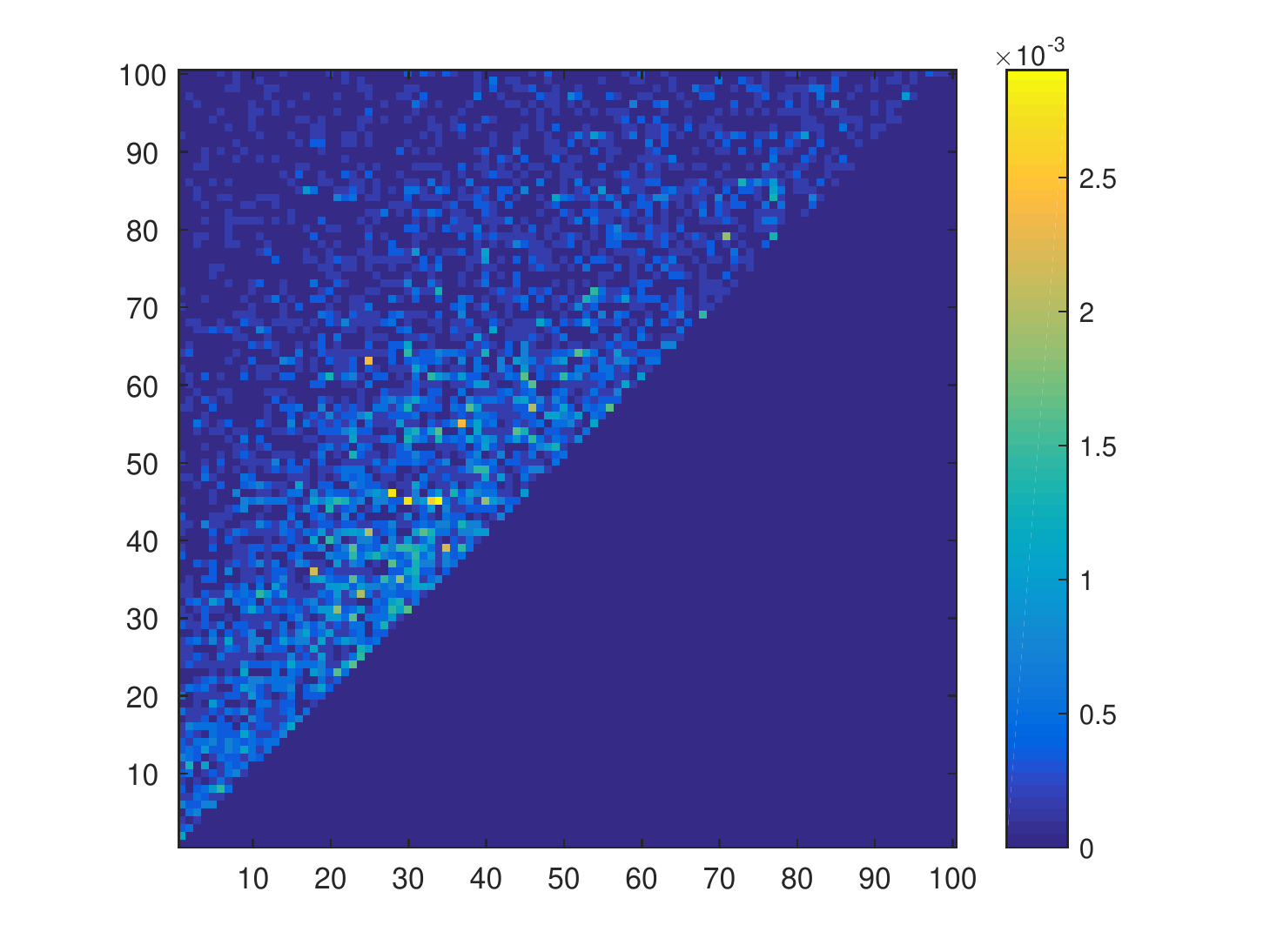}
        \caption{100 Items}
    \end{subfigure}
    \caption{Heat map of labeling frequency for item pairs with different levels of ambiguity}
    \label{fig: item_frequency}
\end{figure*}

\begin{figure*}[!t]
    \centering
    \includegraphics[scale = 0.6]{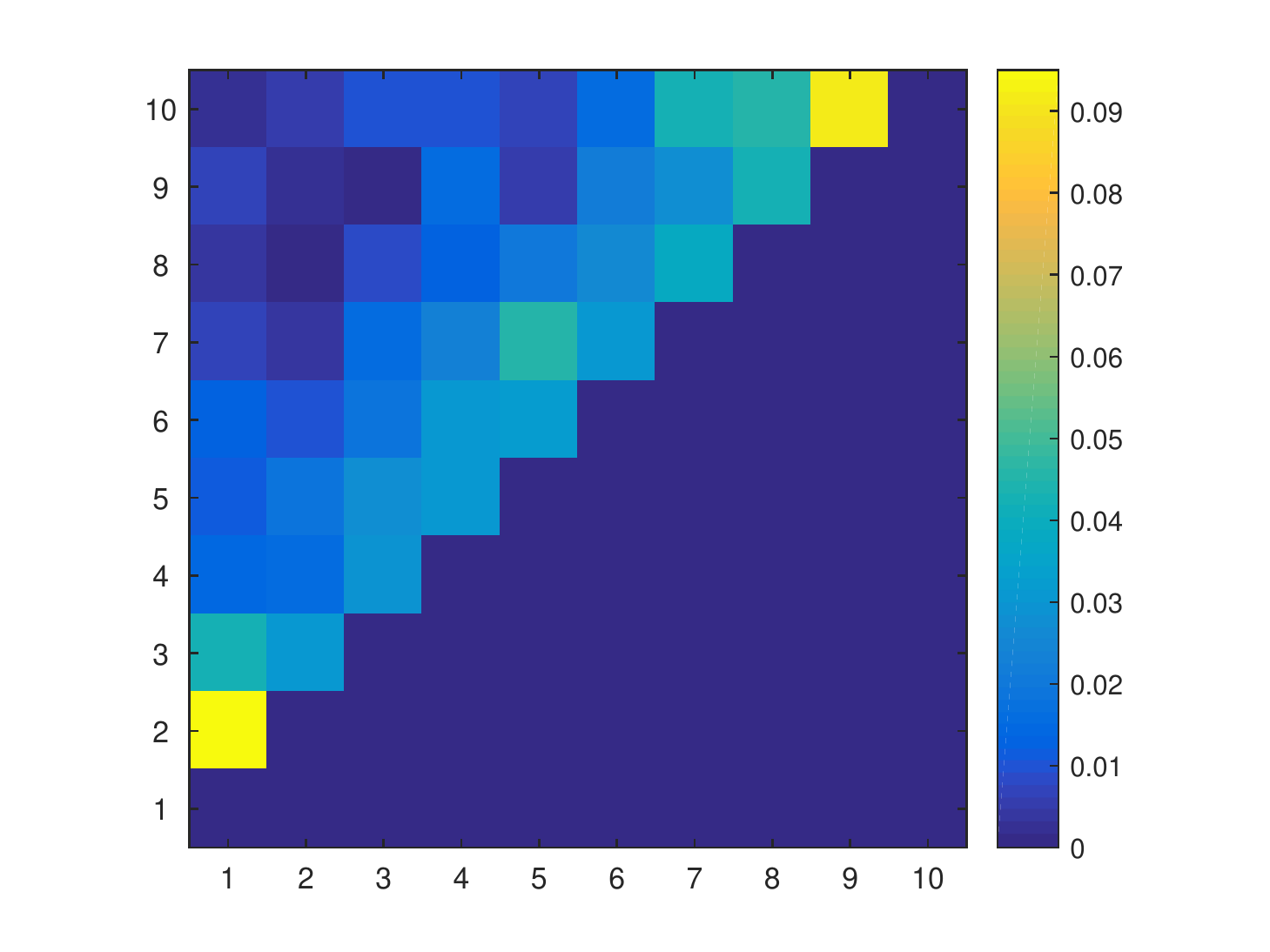}
    \caption{Heat map of labeling frequency for pairs with very close scores}
    \label{fig: item_frequency_close}
\end{figure*}

Next, we study the allocation of labeling budget over pairs of items with different levels of ambiguity when using the AKG policy. Again, we consider two scenarios: $K = 10, T = 100$ and $K = 100, T = 1000$, each with 100 independent trials.  We report the averaged labeling frequency of each pair. The results are presented in Figure \ref{fig: item_frequency} in the form of heat maps. In Figure \ref{fig: item_frequency}, each small block represents a pair of items. Items are sorted based on their true latent scores, from lowest to highest along both $y$-axis and $x$-axis, so that the item pairs along the back-diagonal are more ambiguous than those around the corner. Figure \ref{fig: item_frequency} presents the  normalized number of comparisons over different pairs in total $T$ stages. It can be seen from Figure~\ref{fig: item_frequency} that the back-diagonal pairs in general have higher labeling frequency than other pairs. Some adjacent pairs are labeled 10 times more frequently than the distant pairs. To further demonstrate this property, we design a scenario in which out of 10 items, the two worst items and the two best items have very close true scores respectively. Although the main goal of the algorithm is to achieve higher ranking accuracy, we are still curious to see whether our policy can spend the budget on these two pairs. As we can see from Figure \ref{fig: item_frequency_close}, it is clear that the algorithm concentrates on the 1-2 pair and the 9-10 pair. This implies that our policy can identify and explore more ambiguous pairs to improve the learning of the true ranks.

\subsection{Simulated Study under the Heterogeneous Workers Setting}

In this section we bring worker quality $\rho_w$ into consideration, which is assumed to be drawn from the Beta(4,1) distribution. We choose the Beta(4,1) to generate $\rho_w$ since the average reliability measure of workers in this case is $4/5 = 80\%$. This assumption is in line with the practice in that there are usually more reliable workers than unreliable ones.
Similar to the homogeneous worker setting, we consider two scenarios: 10 items with 10 heterogeneous workers ($K$ = 10, $M$ = 10); 100 items with 50 heterogeneous workers ($K$ = 100, $M$ = 50) and we note that each worker is allowed to label any pair at most once. We compare the following three methods.
\begin{enumerate}
\item \textbf{AKG} (see Algorithm \ref{alg:AKGw}): We set the prior of $\btheta$ to be the uniform distribution on the simplex (i.e., $\balpha^0$ is set to be an all-one vector) and choose $\mu^0_w = 4, \nu^0_w = 1$ for each worker $w = 1,2,\dots,M$.

\item \textbf{Random Sampling}:  It is implemented simply by replacing Step 2 of Algorithm \ref{alg:AKGw} by a random sampling policy, which selects a triplet \{item $i$, item $j$, worker $w$\} uniformly randomly at each stage. The choices of priors are the same as in AKG. Like the AKG method, the random sampling algorithm also maintains a Dirichlet distribution for the scores of items and a beta distribution for the reliability parameter of each worker using moment matching.

\item \textbf{Crowd-BT}: This is an  adaptive algorithm recently proposed by \cite{Chen:13}, which chooses the triplet \{item $i$, item $j$, worker $w$\} at each iteration to maximize the information gain. This can be viewed as an extension of the adaptive polling \citep{Pfeiffer:12} by incorporating the workers' reliability. Unlike adaptive polling which computes the relative entropy for each pair exactly, Crowd-BT uses moment matching to approximate the posterior and hence runs significantly faster than adaptive polling. The parameter $\gamma$, which balances the exploitation-exploration trade-off in \cite{Chen:13}, is set to 1 in this experiment.

\end{enumerate}

\begin{figure*}[!t]
   \centering
    \begin{subfigure}[h]{0.5\textwidth}
       \centering
        \includegraphics[width=\textwidth]{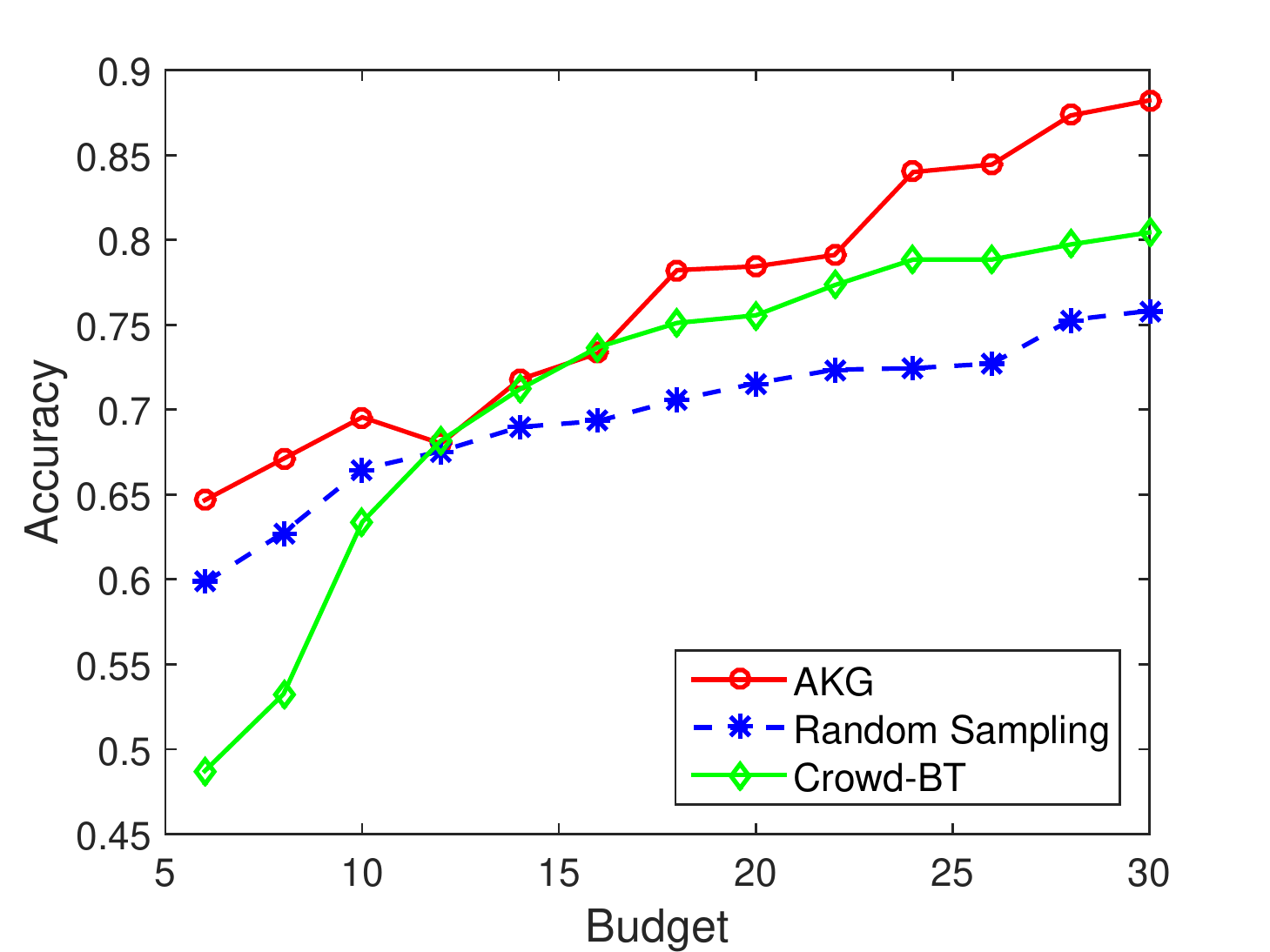}
        \caption{10 Items}
    \end{subfigure}%
	~
    \begin{subfigure}[h]{0.5\textwidth}
        \centering
        \includegraphics[width=\textwidth]{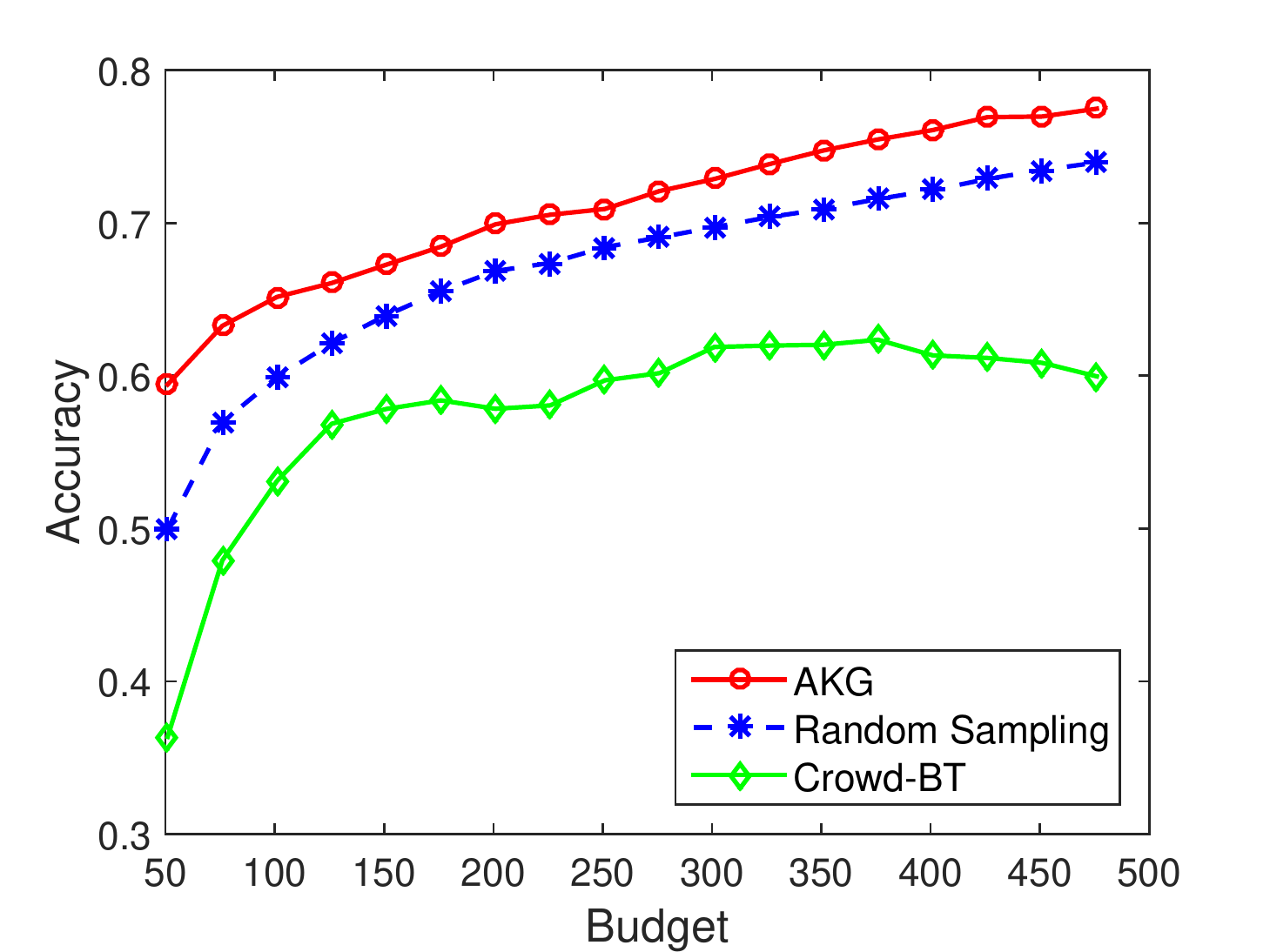}
        \caption{100 Items}
    \end{subfigure}
   \caption{Performance comparison under the heterogeneous workers setting. The $x$-axis is the budget level and $y$-axis is the averaged ranking accuracy.}
   \label{fig: heterogeneous}
\end{figure*}

The comparison results are presented in Figure \ref{fig: heterogeneous}, where AKG outperforms the other two methods, especially when the budget level is low. The performance of random sampling is comparable to AKG at the beginning. As we gather more information, AKG can learn the reliability of workers so that the budget will be gradually shifted towards those reliable workers (as shown later in Figure~\ref{fig:frequency_w}). In fact, it can be seen from Figure \ref{fig: heterogeneous} that the ranking accuracy of AKG increases more quickly than that of other methods. In this experiment, even if there is a small amount of budget (e.g. $T=K$), the AKG policy is still able to achieve reasonably good performance. We notice that in the 100-item case Crowd-BT is beaten by random sampling. The main reason is that when the reliability of workers varies and the pool is large, it is difficult to balance exploration and exploitation for Crowd-BT, which has already been acknowledged in \cite{Chen:13}. Similar to the previous setting, we also give the table of the computation time of a \emph{single iteration} for AKG in Table \ref{tab:comp_2}. As we can see from the table, even with another dimension of uncertainty --- the reliability of workers, AKG is still quite fast, and thus is suitable for online implementation.

\begin{table}[!t]
\centering
    \caption{Computation time under the heterogeneous workers setting.}
    \begin{tabular}{ | c | c | c |}
    \hline
    No. of Items & No. of Workers & AKG \\ \hline
    10 & 10 & 0.038 sec \\ \hline
    25 & 20 &  0.82 sec \\ \hline
    100 & 50 & 41 sec \\ \hline
    \end{tabular}
    \label{tab:comp_2}
\end{table}

\begin{figure*}[!t]
   \centering
    \begin{subfigure}[h]{0.31\textwidth}
       \centering
        \includegraphics[width=\textwidth]{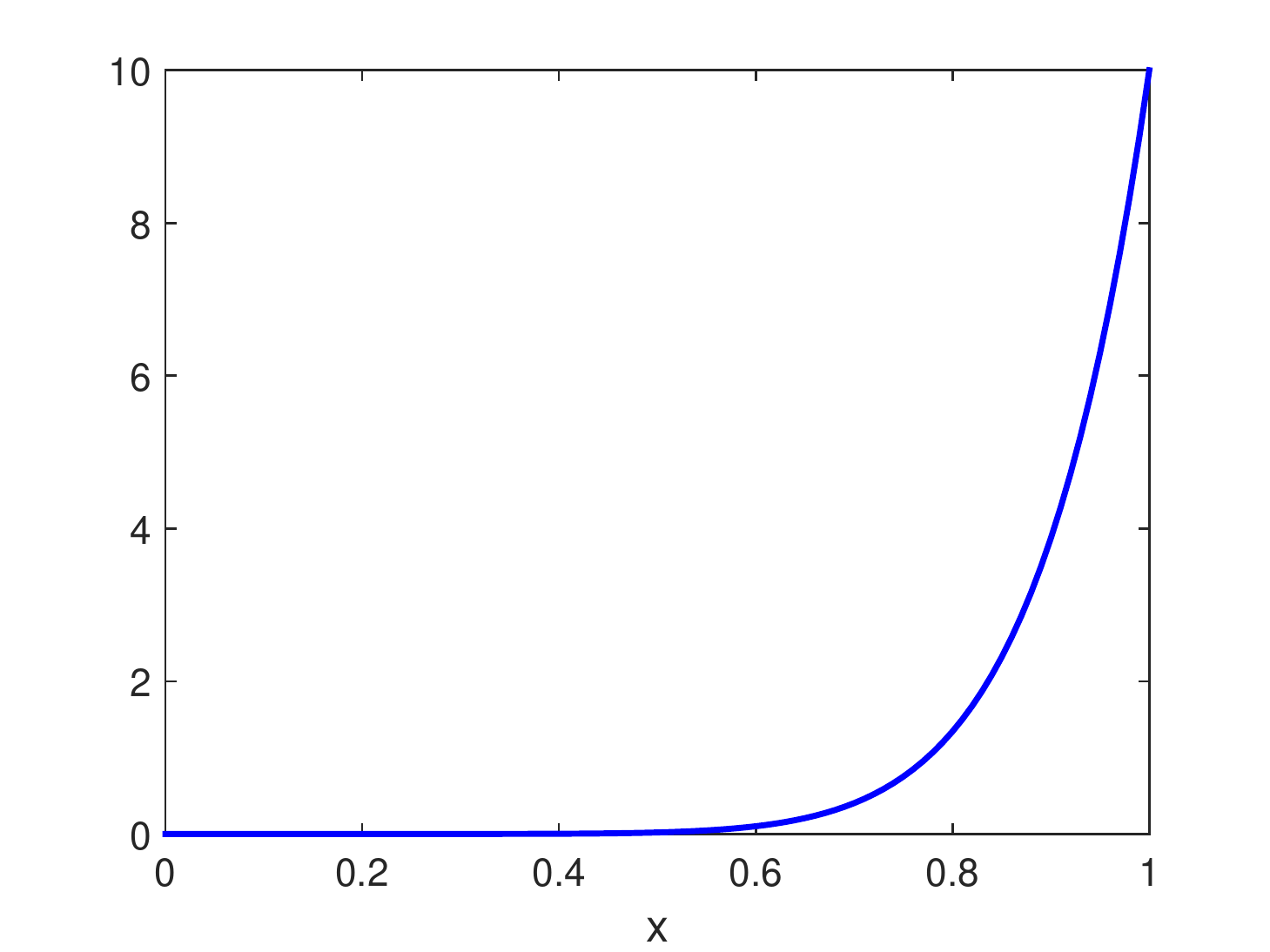}
        \caption{Beta$(10,1)$}
    \end{subfigure}%
    ~
    \begin{subfigure}[h]{0.31\textwidth}
        \centering
        \includegraphics[width=\textwidth]{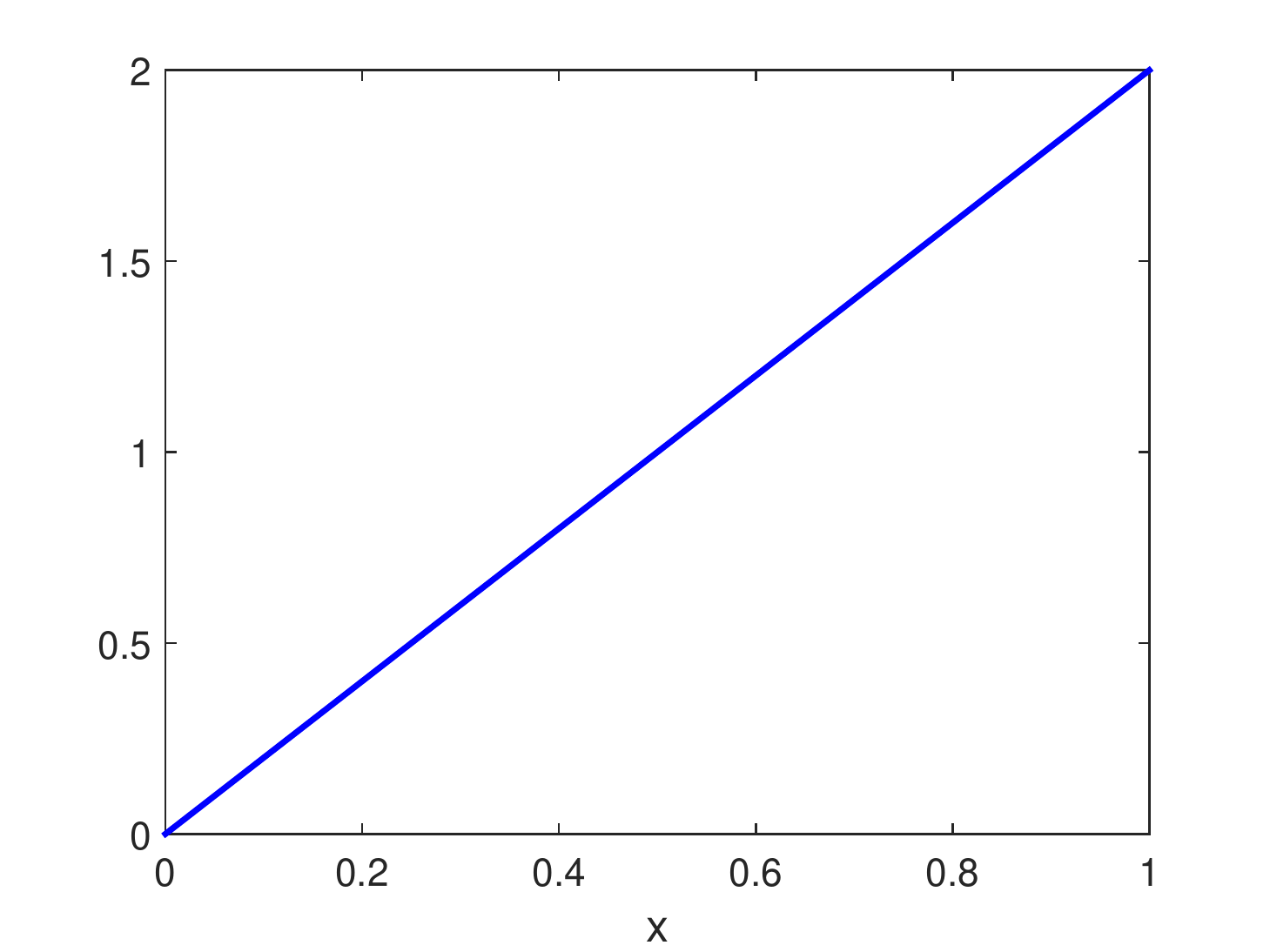}
        \caption{Beta$(2,1)$}
    \end{subfigure}
    ~
    \begin{subfigure}[h]{0.31\textwidth}
        \centering
        \includegraphics[width=\textwidth]{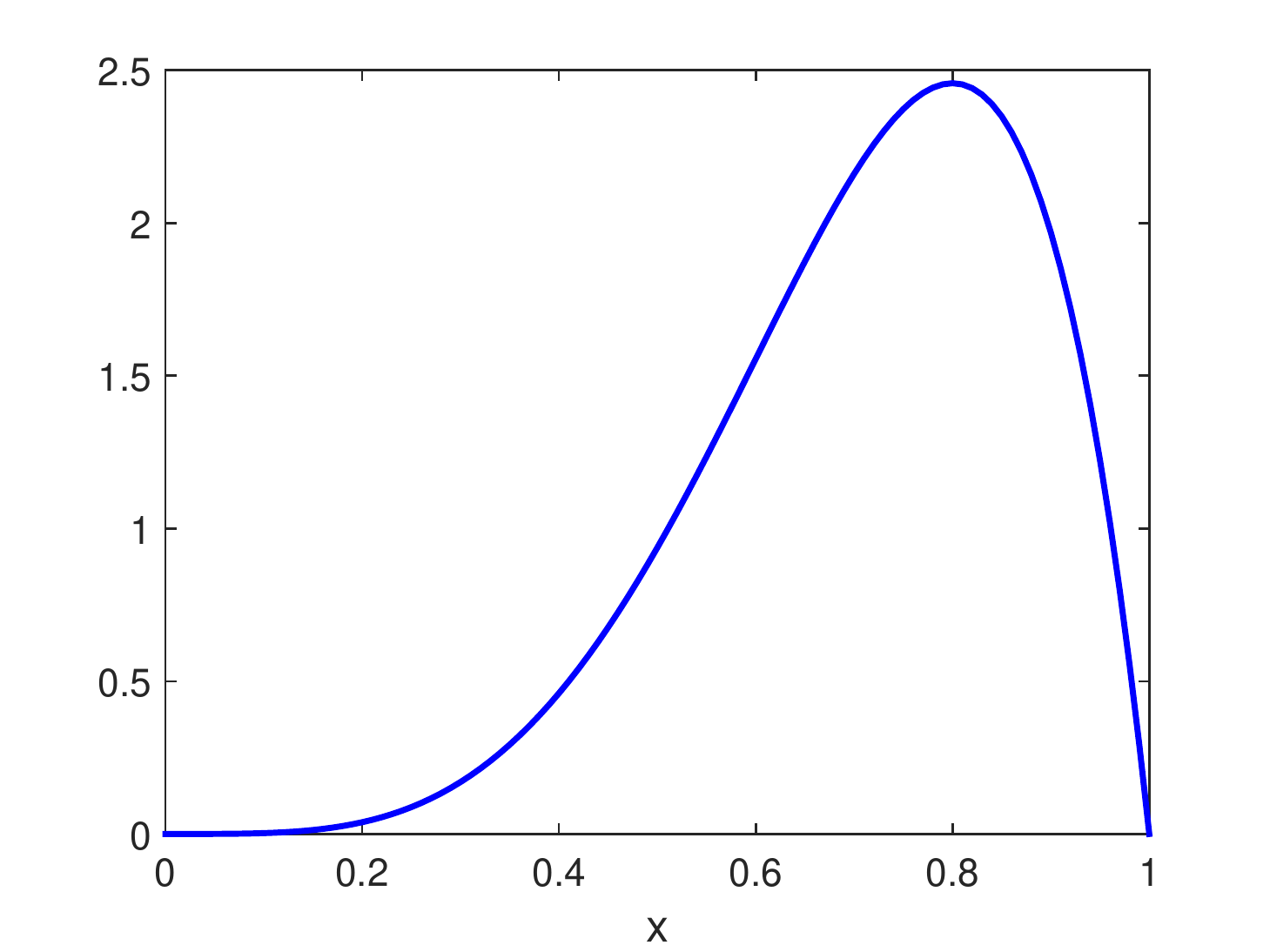}
        \caption{Beta$(5,2)$}
    \end{subfigure}
   \caption{Density plots of different Beta distributions for generating $\rho_w$}
\end{figure*}
\begin{figure*}[!h]
   \centering
    \begin{subfigure}[h]{0.31\textwidth}
       \centering
        \includegraphics[width=\textwidth]{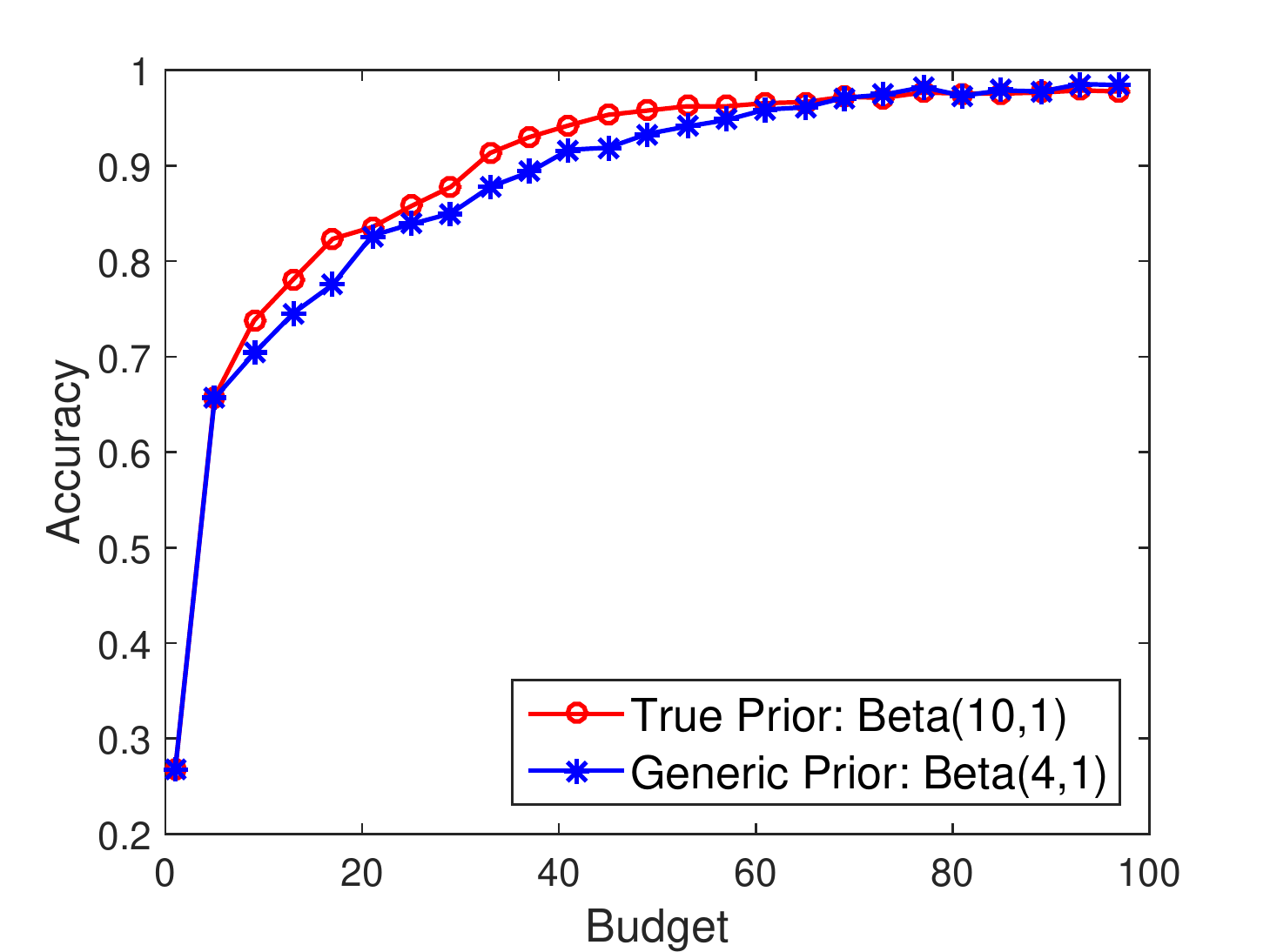}
        \caption{Beta$(10,1)$ Prior}
    \end{subfigure}%
    ~
    \begin{subfigure}[h]{0.31\textwidth}
        \centering
        \includegraphics[width=\textwidth]{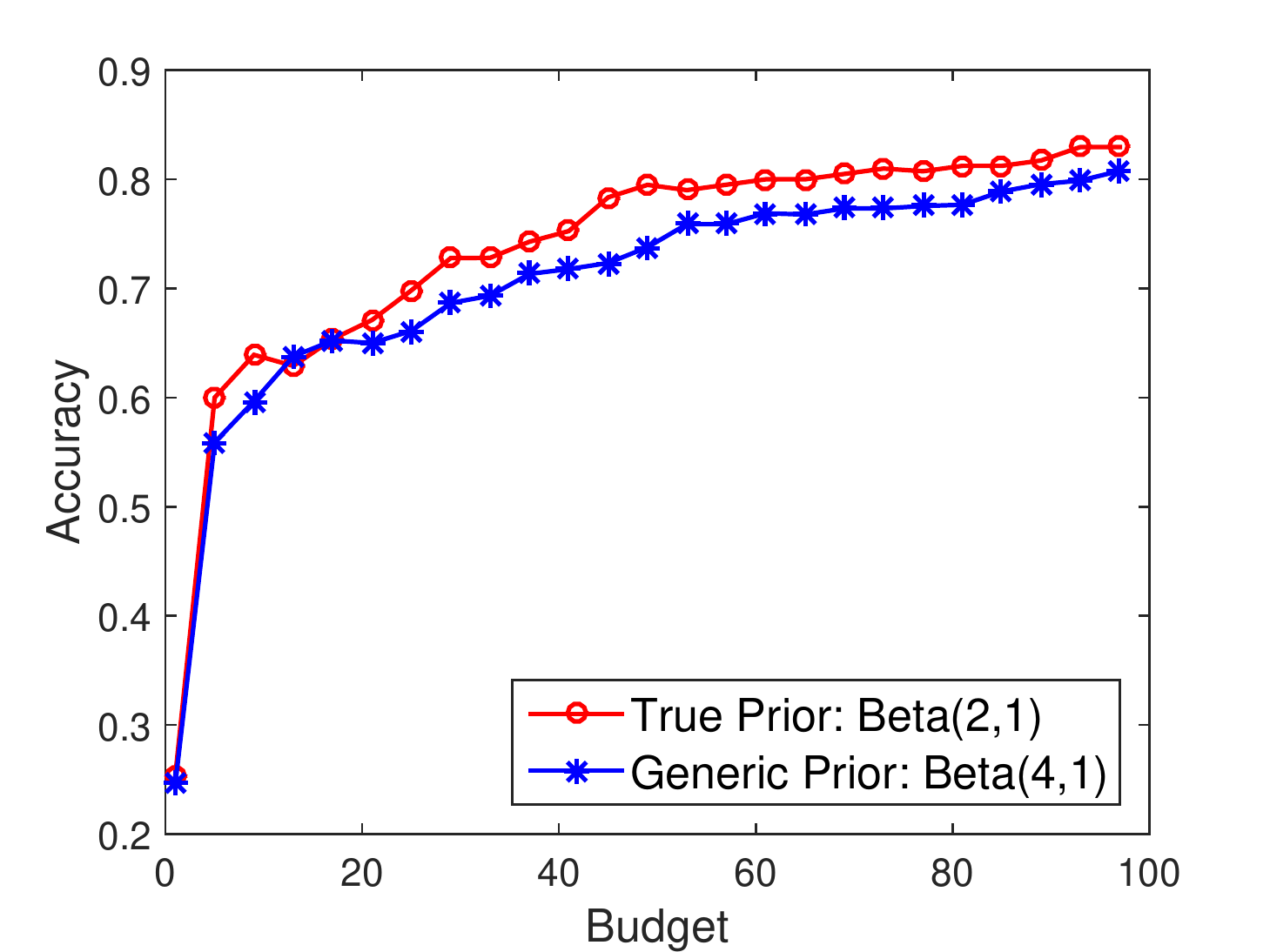}
        \caption{Beta$(2,1)$ Prior}
    \end{subfigure}
    ~
    \begin{subfigure}[h]{0.31\textwidth}
        \centering
        \includegraphics[width=\textwidth]{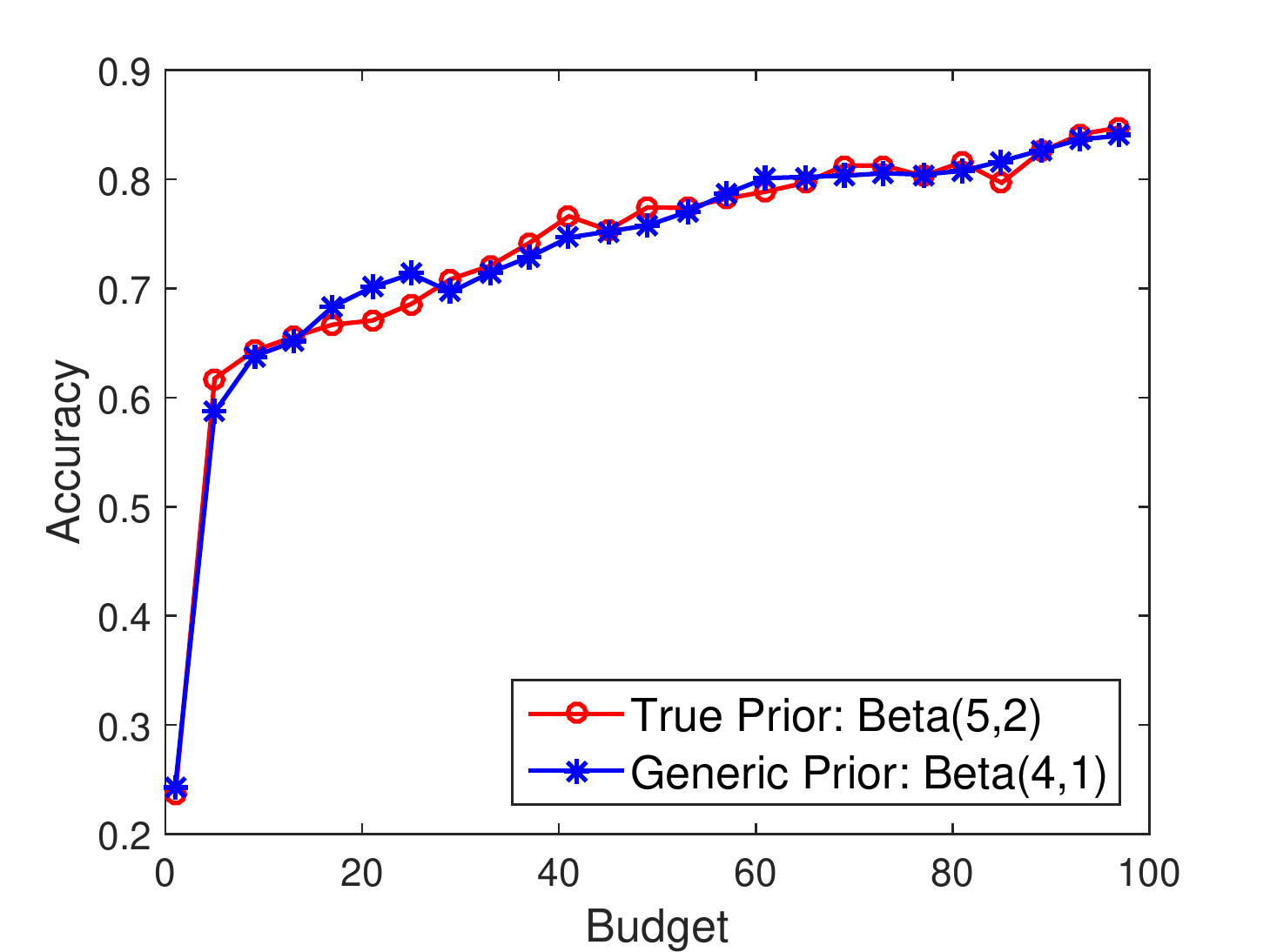}
        \caption{Beta$(5,2)$ Prior}
    \end{subfigure}
   \caption{Comparisons between AKG using Beta(4,1) prior and AKG using the true generating distribution as prior.}
   \label{fig: prior_comparison}
\end{figure*}

In order to investigate how sensitive the prior for workers' reliability $\rho_w$ is, we generate the workers' true reliability parameters from three different distributions, $\text{Beta}(10,1)$, $\text{Beta}(2,1)$, and $\text{Beta}(5,2)$, and compare the performances of AKG between using the true generating distribution as the prior and using the generic $\text{Beta}(4,1)$ as the prior. The results are plotted in Figure \ref{fig: prior_comparison}. As one can see from Figure \ref{fig: prior_comparison}, using the true generating distribution and generic $\text{Beta}(4,1)$ prior lead to very similar performance in all three cases. Although there are some small differences between the two groups of curves, they are not significant as to the overall performance of the algorithm.
This result shows that when there is no exact information on the quality of all workers, $\text{Beta}(4,1)$ is a reasonable prior for workers' reliability and the proposed AKG policy is quite robust to the prior distribution in use.

Finally, we investigate whether good workers are indeed assigned more comparison tasks by our AKG policy in the setting of heterogeneous workers. In particular, we consider $K = 10$ items and $M = 15$ workers with the workers' true reliability parameters $\rho_w, w = 1,2,\dots,M$ ranging from 0.4 to 1 with an equal space in between. This crowd of workers is fixed and the total budget in each trial $T = 250$. We report the averaged number of pairs assigned to workers with different levels of reliability in Figure \ref{fig:frequency_w}. As one can see from Figure \ref{fig:frequency_w}, there is a clear trend that more reliable workers receive more pairs on average.


\begin{figure}[!t]
\centering
\includegraphics[width=0.5\textwidth]{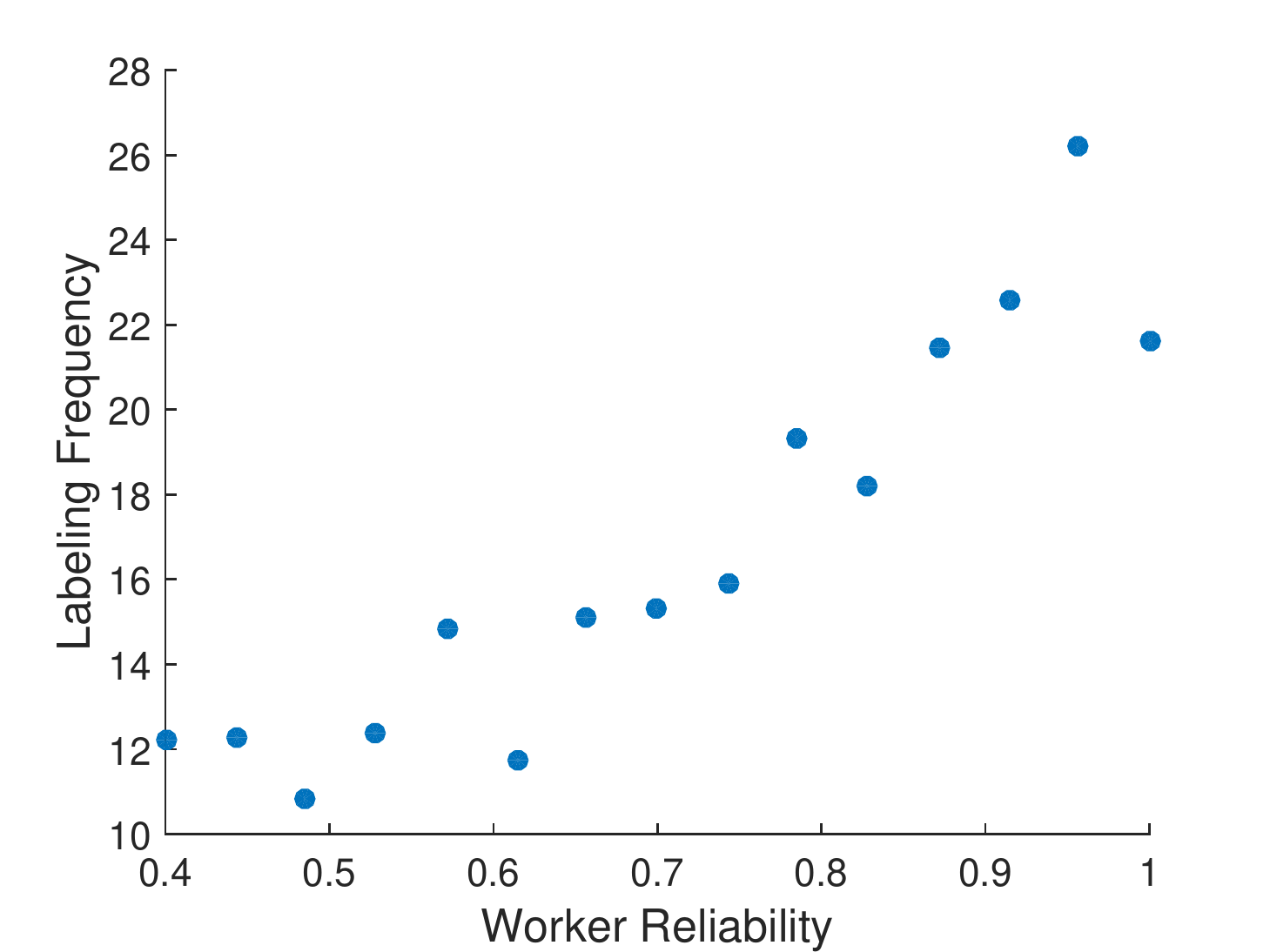}
 \caption{Averaged number of comparisons (a.k.a., labeling frequency) made by workers with different levels of reliability $\rho_w$.}
 \label{fig:frequency_w}
\end{figure}

\subsection{Real Data Study}

\begin{figure*}[!t]
   \centering
    \begin{subfigure}[h]{0.5\textwidth}
       \centering
        \includegraphics[width=\textwidth]{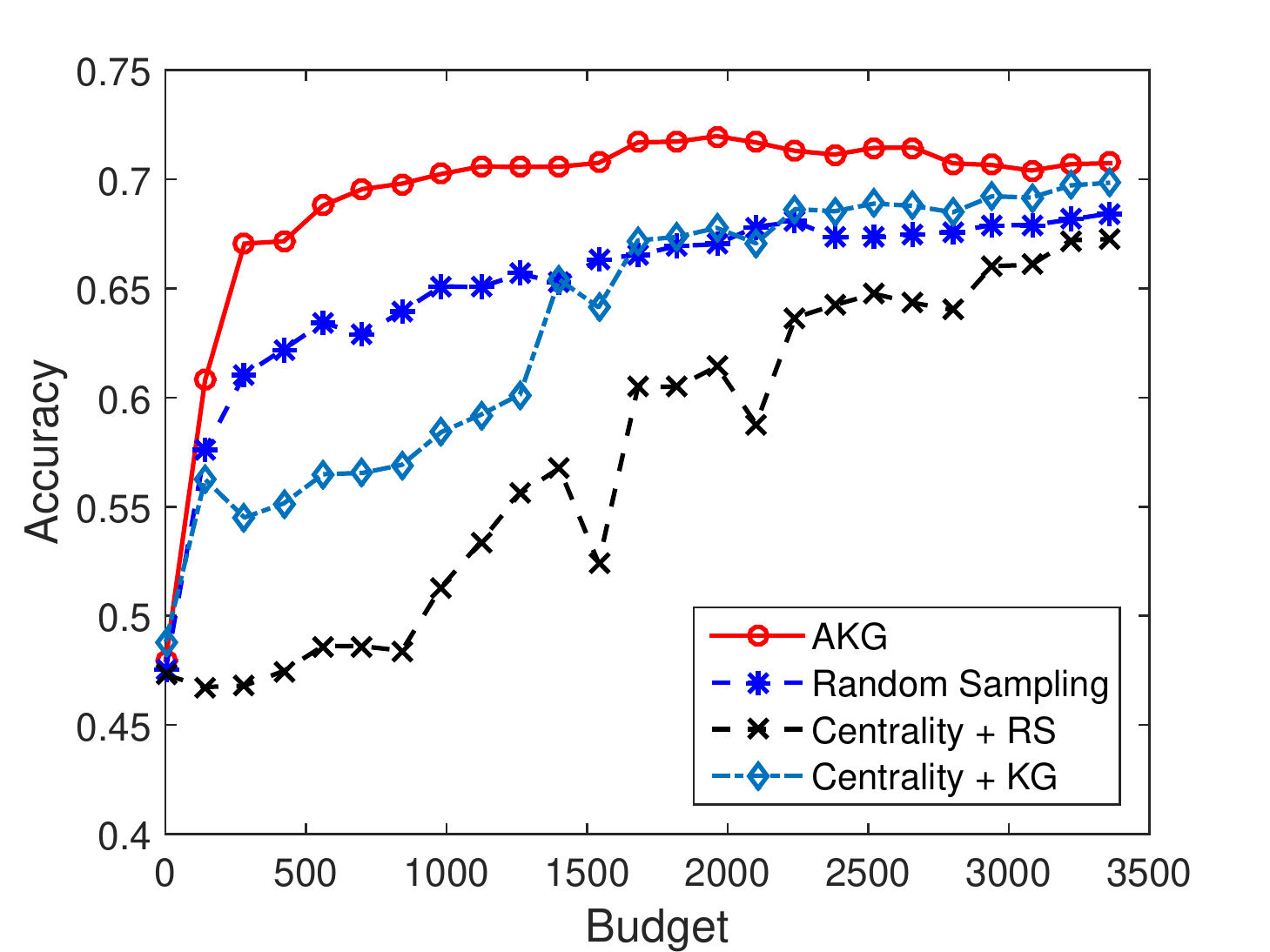}
        \caption{Homogeneous workers (fully reliable)}
        \label{fig:real_homo}
    \end{subfigure}%
	~
    \begin{subfigure}[h]{0.5\textwidth}
        \centering
        \includegraphics[width=\textwidth]{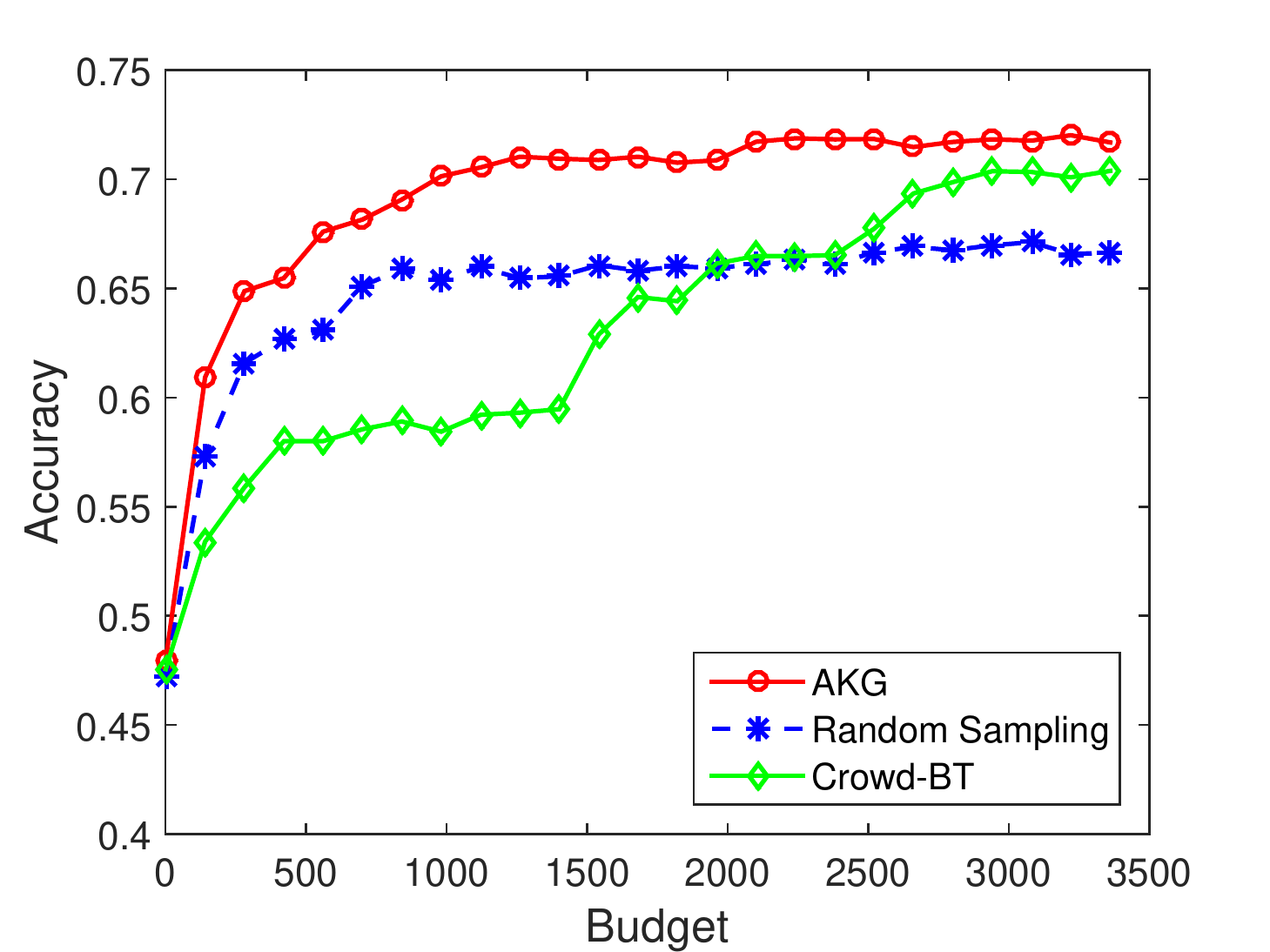}
        \caption{Heterogeneous workers}
        \label{fig: real_data_hetero}
    \end{subfigure}
   \caption{Performance comparison on the real dataset}
   \label{fig: real_data}
\end{figure*}

We now apply the proposed AKG policy (Algorithm \ref{alg:AKGw}) to a real dataset on reading difficulty levels \citep{Kevyn:04}. The dataset comprises $K=491$ different paragraphs, each assigned an integer-valued true reading difficulty score ranging from $1,2,\dots,12$. Here, a higher score means the paragraph is more difficult to read. A total number of $M=217$ different workers from Canada and the United States performed the comparison tasks on an online crowdsourcing platform called CrowdFlower\footnote{\url{http://www.crowdflower.com/}}. Each worker was presented a pair of paragraphs every time and the worker identified which paragraph is more difficult to read. To overcome the issue of an imbalanced judgemental pool, each worker was allowed to compare at most 40 different pairs.  There are 7,898 pairwise comparison results available in this dataset. Using these pairwise labels, we apply the AKG policy to recover the ranking by difficulty of these 491 paragraphs. We note that since the underlying truth is given as a difficulty level (1--12) for each paragraph (denoted by $s_i$ for $i=1,\ldots, K$) instead of a global ranking, we measure the accuracy of a ranking $\pi$ as
\[
\frac{2}{K(K-1)}\sum_{i\neq j}\mathbf{1}_{\{\pi(i)>\pi(j)\}}\mathbf{1}_{\{s_i\geq s_j\}}.
\]
In the above definition of ranking accuracy, when two paragraphs have the same reading difficulty level, any ranking between this pair will be treated as correct. It is also worth noting that, in the knowledge gradient step in \eqref{eq:single_stage_AGKw}, it is possible that the selected triplet $(i_t, j_t, w_t)$ does not exist in the dataset (i.e., the worker $w_t$ did not compare $i_t$ and $j_t$ in this data). Hence, in our implementation of AKG, we select the triplet in the dataset that maximizes the right-hand side of \eqref{eq:single_stage_AGKw}. We set the prior of $\btheta$ to be the uniform distribution on the simplex. This dataset also comes with a rating for each worker which measures the long-run performance of this worker on CrowdFlower. A higher rating implies a higher reliability of the worker. This dataset shows the averaged workers' rating is above 0.75. Thus, we still use Beta(4,1) as the prior on workers' reliability.

We run experiments in two different settings. The first one assumes that all workers are homogeneous and fully reliable. In this setting, we only need to select the next pair of paragraphs to compare but can randomly choose a worker to perform the comparison task. In this case, four algorithms are implemented (AKG policy (Algorithm \ref{alg:AKG}), random sampling, rank centrality with the random sampling policy, and rank centrality with the knowledge gradient policy) and we report the averaged accuracy over 100 independent trials in Figure \ref{fig:real_homo} to minimize the sampling effect of randomly selecting the next worker. The second experiment incorporates the heterogeneous reliability of workers so that the algorithms have to select both the pair to compare and the worker to perform the comparison task. In this case, three algorithms, AKG policy (Algorithm \ref{alg:AKGw}), random sampling and Crowd-BT, are implemented and the result is shown in Figure \ref{fig: real_data_hetero}. As one can see from these two plots, AKG outperforms the other methods in both settings, especially when the amount of budget is relatively low. As the budget level increases, the performance of Crowd-BT and rank centrality will eventually improve and achieve a similar accuracy as AKG.


\section{Conclusion}
\label{sec:conclusion}

In this paper, we address the dynamic budget allocation problem in crowdsourced ranking. Using the Kendall's tau with respect to the true ranking as the measure of ranking accuracy, we formulate the problem of maximizing expected Kendall's tau by sequential comparisons into a Bayesian Markov decision process. To further address the computational challenges (especially, solving the NP-hard MAX-LOP) involved in the decision process, we propose an approximated knowledge gradient policy, which is not only computationally efficient but also achieves good performance as shown in the experimental sections. 

We note that although this paper focuses on the Bradley-Terry-Luce model \citep{Bradley:52, Luce:59}, 
it will be interesting to study the dynamic sampling in crowdsourced ranking for other ranking models such as permutation-based models (e.g., Mallows \citep{Mallows:57} and CPS \citep{Qin:10} models) or stochastically transitive models \citep{Fis:73, Shah15b}). Meanwhile, theoretical bounds on posterior approximation errors are difficult to obtain and error propagation does exist during each iteration of the algorithm. In our future analysis we would like to quantify this error. Another interesting future direction is to incorporate the feature information of each item into the probabilistic model of the pairwise comparison results and develop a dynamic sampling policy that can further improve the ranking accuracy via modeling the feature information.



\section*{Acknowledgement} 
Xi Chen would like to acknowledge support for this project
from the Google Faculty Research Award.

\newpage

\section*{Appendix}

In this section, we provide detailed proofs of some propositions in the paper.

\bigskip

\begin{proof}(\textbf{of Proposition~\ref{propMM}})

We will only show that \eqref{eq:moment_matching1} and \eqref{eq:moment_matching2} can be represented as \eqref{MM1} when $Y_{ij}=1$. The proof for $Y_{ij}=-1$ is similar.

It is known that $\mathbb{E} \left[\theta_k |\btheta\sim\Dir(\balpha')\right]=\frac{\alpha'_k}{\alpha'_0}$ and $\mathbb{E} \left[\theta_k^2 |\btheta\sim\Dir(\balpha')\right]=\frac{\alpha'_k(\alpha'_k+1)}{\alpha'_0(\alpha'_0+1)}$ for $k=1,2,\dots,K$, which characterize the left-hand sides of \eqref{eq:moment_matching1} and \eqref{eq:moment_matching2}.

With elementary calculus, we can show
\begin{eqnarray}
\label{prop1eq1}
\text{Pr}(Y_{ij}=1|\balpha)=\int_{\Delta}\frac{\theta_i}{\theta_i+\theta_j}\frac{1}{\mathrm{B}(\balpha)} \prod_{k=1}^K \theta_k^{\alpha_k - 1}\mathrm{d}\btheta=\frac{\alpha_i}{\alpha_i+\alpha_j}
\end{eqnarray}
so that
\begin{eqnarray}
\label{prop1eq0}
p(\btheta|Y_{ij}=1,\balpha)=\frac{p(\btheta, Y_{ij}=1 | \balpha)}{\Pr(Y_{ij}=1|  \balpha)}=\frac{\alpha_i+\alpha_j}{\alpha_i}\frac{\theta_i}{\theta_i+\theta_j}\frac{1}{\mathrm{B}(\alpha)}\prod_{k=1}^K \theta_k^{\alpha_k - 1}.
\end{eqnarray}
Let $\bbeta=(\beta_1,\dots,\beta_K)$ with $\beta_i=\alpha_i+1$ and $\beta_k=\alpha_k$ for $k\neq i$. Then, we can show that
\begin{eqnarray}
\nonumber
\mathbb{E}[\theta_i|Y_{ij}=1,\balpha]&=&\frac{\alpha_i+\alpha_j}{\alpha_i}\left[\int_{\Delta}\frac{\theta_i^2}{\theta_i+\theta_j}\frac{1}{\mathrm{B}(\balpha)} \prod_{k=1}^K \theta_k^{\alpha_k - 1}\mathrm{d}\btheta\right]\\\nonumber
&=&\frac{\alpha_i+\alpha_j}{\alpha_0}\left[\int_{\Delta}\frac{\theta_i}{\theta_i+\theta_j}\frac{1}{\mathrm{B}(\bbeta)} \prod_{k=1}^K \theta_k^{\beta_k - 1}\mathrm{d}\btheta\right] \\ \label{prop1eq2}
&=& \frac{\alpha_i+1}{\alpha_0}\frac{\alpha_i+\alpha_j}{\alpha_i+\alpha_j+1},
\end{eqnarray}
where the first and the third equalities are due to \eqref{prop1eq1} and \eqref{prop1eq0} and the second equality is by the definition of $\bbeta$ and the property $\Gamma(x+1)=x\Gamma(x)$ of Gamma function. Using a similar argument, we can show that
\begin{eqnarray}
\label{prop1eq3}
\mathbb{E}[\theta_j|Y_{ij}=1,\balpha]&=&\frac{\alpha_j}{\alpha_0}\frac{\alpha_i+\alpha_j}{\alpha_i+\alpha_j+1}\\\label{prop1eq4}
\mathbb{E}[\theta_k|Y_{ij}=1,\balpha]&=&\frac{\alpha_k}{\alpha_0}\quad\text{ for }k\neq i, j\\\label{prop1eq5}
\mathbb{E}\left[\theta_i^2| Y_{ij}=1,\balpha\right]&=& \frac{\alpha_i+1}{\alpha_0}\frac{\alpha_i+2}{\alpha_0+1}\frac{\alpha_i+\alpha_j}{\alpha_i+\alpha_j+2} \\\label{prop1eq6}
\mathbb{E}\left[\theta_j^2| Y_{ij}=1,\balpha\right]&=& \frac{\alpha_j}{\alpha_0}\frac{\alpha_j+1}{\alpha_0+1}\frac{\alpha_i+\alpha_j}{\alpha_i+\alpha_j+2}\\\label{prop1eq7}
\mathbb{E}\left[\theta_k^2| Y_{ij}=1,\balpha\right]&=& \frac{\alpha_k}{\alpha_0}\frac{\alpha_k+1}{\alpha_0+1}\quad\text{ for }k\neq i, j.
\end{eqnarray}
Note that, when $Y_{ij}=1$, the right-hand sides of \eqref{eq:moment_matching1} and \eqref{eq:moment_matching2} can be represented as the right-hand side of \eqref{MM1} using \eqref{prop1eq2}$\sim$\eqref{prop1eq7}
\end{proof}

\begin{proof}(\textbf{of Proposition~\ref{propMMw}})

We will only show  the conclusion when $Y_{ij}^w=1$. The proof for $Y_{ij}^w=-1$ is similar.

When $Y_{ij}^w=1$, we have $\eta_{ijw} = \frac{\mu_w\alpha_i}{\mu_w\alpha_i+\nu_w\alpha_j}$. We will first show \eqref{eq:moment_matching1w} and \eqref{eq:moment_matching2w} can be represented as the first four equations in \eqref{MM1w}. Since $\mathbb{E} \left[\theta_k |\btheta\sim\Dir(\balpha')\right]=\frac{\alpha'_k}{\alpha'_0}$ and $\mathbb{E} \left[\theta_k^2 |\btheta\sim\Dir(\balpha')\right]=\frac{\alpha'_k(\alpha'_k+1)}{\alpha'_0(\alpha'_0+1)}$ for $k=1,2,\dots,K$, the left-hand sides of the first four equations in \eqref{MM1w} and those of \eqref{eq:moment_matching1w} and \eqref{eq:moment_matching2w} are identical.

With \eqref{prop1eq1} and some basic properties of the Beta distribution, we can show
\begin{eqnarray}
\nonumber
\text{Pr}(Y_{ij}^w=1|\balpha,\mu_w,\nu_w)&=&\mathbb{E}\left[ \rho_w\frac{\theta_i}{\theta_i+\theta_j}+(1-\rho_w)\frac{\theta_j}{\theta_i+\theta_j}\big|\btheta\sim\Dir(\balpha),\rho_w\sim\B(\mu_w,\nu_w)\right]\\\label{prop4eq1}
&=&\frac{\mu_w}{\mu_w+\nu_w}\frac{\alpha_i}{\alpha_i+\alpha_j}+\frac{\nu_w}{\mu_w+\nu_w}\frac{\alpha_j}{\alpha_i+\alpha_j}
\end{eqnarray}
so that
\begin{eqnarray}
\label{prop4eq0}
&&p(\btheta,\rho_w|Y_{ij}^w=1,\balpha,\mu_w,\nu_w)
\\
&=&\frac{\left(\rho_w\frac{\theta_i}{\theta_i+\theta_j}+(1-\rho_w)\frac{\theta_j}{\theta_i+\theta_j}\right)\frac{1}{\mathrm{B}(\balpha)\mathrm {B} (\mu_w,\nu_w )}\prod_{k=1}^K \theta_k^{\alpha_k - 1}\rho_w^{\mu_w -1}(1-\rho_w)^{\nu_w  -1}}
{\frac{\mu_w}{\mu_w+\nu_w}\frac{\alpha_i}{\alpha_i+\alpha_j}+\frac{\nu_w}{\mu_w+\nu_w}\frac{\alpha_j}{\alpha_i+\alpha_j}}. \nonumber
\end{eqnarray}
The equations \eqref{prop4eq1} and \eqref{prop4eq0}, together with \eqref{prop1eq2}, imply
\begin{eqnarray}
\nonumber
&&\mathbb{E}[\theta_i|Y_{ij}^w=1,\balpha,\mu_w,\nu_w]\\\nonumber
&=&\frac{\int_0^1\int_{\Delta}\left(\rho_w\frac{\theta_i^2}{\theta_i+\theta_j}+(1-\rho_w)\frac{\theta_j\theta_i}{\theta_i+\theta_j}\right)\frac{1}{\mathrm{B}(\balpha)\mathrm {B} (\mu_w,\nu_w )}\prod_{k=1}^K \theta_k^{\alpha_k - 1}\rho_w^{\mu_w -1}(1-\rho_w)^{\nu_w  -1}\mathrm{d}\btheta\mathrm{d}\rho_w}
{\frac{\mu_w}{\mu_w+\nu_w}\frac{\alpha_i}{\alpha_i+\alpha_j}+\frac{\nu_w}{\mu_w+\nu_w}\frac{\alpha_j}{\alpha_i+\alpha_j}}\\\nonumber
&=&\frac{\int_{\Delta}\left(\frac{\mu_w}{\mu_w+\nu_w}\frac{\theta_i^2}{\theta_i+\theta_j}+\frac{\nu_w}{\mu_w+\nu_w}\frac{\theta_j\theta_i}{\theta_i+\theta_j}\right)\frac{1}{\mathrm{B}(\balpha)}\prod_{k=1}^K \theta_k^{\alpha_k - 1}\mathrm{d}\btheta}
{\frac{\mu_w}{\mu_w+\nu_w}\frac{\alpha_i}{\alpha_i+\alpha_j}+\frac{\nu_w}{\mu_w+\nu_w}\frac{\alpha_j}{\alpha_i+\alpha_j}}\\\nonumber
&=&\frac{\frac{\mu_w}{\mu_w+\nu_w}\frac{\alpha_i}{\alpha_0}\frac{\alpha_i+1}{\alpha_i+\alpha_j+1}}
{\frac{\mu_w}{\mu_w+\nu_w}\frac{\alpha_i}{\alpha_i+\alpha_j}+\frac{\nu_w}{\mu_w+\nu_w}\frac{\alpha_j}{\alpha_i+\alpha_j}}
+\frac{\frac{\nu_w}{\mu_w+\nu_w}\frac{\alpha_i}{\alpha_0}\frac{\alpha_j}{\alpha_i+\alpha_j+1}}
{\frac{\mu_w}{\mu_w+\nu_w}\frac{\alpha_i}{\alpha_i+\alpha_j}+\frac{\nu_w}{\mu_w+\nu_w}\frac{\alpha_j}{\alpha_i+\alpha_j}}\\\label{prop4eq2}
&=&\eta_{ijw}\frac{(\alpha_i+1)(\alpha_i+\alpha_j)}{\alpha_0(\alpha_i+\alpha_j+1)}
+ (1-\eta_{ijw})\frac{\alpha_i(\alpha_i+\alpha_j)}{\alpha_0(\alpha_i+\alpha_j+1)}.
\end{eqnarray}
Using a similar argument, we can show that
\begin{align}
\label{prop4eq3}
\mathbb{E}[\theta_j|Y_{ij}^w=1,\balpha,\mu_w,\nu_w]&=\eta_{ijw}\frac{\alpha_j(\alpha_i+\alpha_j)}{\alpha_0(\alpha_i+\alpha_j+1)}
+ (1-\eta_{ijw})\frac{(\alpha_j+1)(\alpha_i+\alpha_j)}{\alpha_0(\alpha_i+\alpha_j+1)}\\\label{prop4eq4}
\mathbb{E}[\theta_k|Y_{ij}^w=1,\balpha,\mu_w,\nu_w]&=\frac{\alpha_k}{\alpha_0}\quad\text{ for }k\neq i, j\\\label{prop4eq5}
\mathbb{E}\left[\theta_i^2| Y_{ij}^w=1,\balpha,\mu_w,\nu_w\right]&= \frac{\eta_{ijw}(\alpha_i+1)(\alpha_i+2)(\alpha_i+\alpha_j)}{\alpha_0(\alpha_0+1)(\alpha_i+\alpha_j+2)}
+\frac{(1-\eta_{ijw})\alpha_i(\alpha_i+1)(\alpha_i+\alpha_j)}{\alpha_0(\alpha_0+1)(\alpha_i+\alpha_j+2)}\quad\\\label{prop4eq6}
\mathbb{E}\left[\theta_j^2| Y_{ij}^w=1,\balpha,\mu_w,\nu_w\right]&= \frac{\eta_{ijw}\alpha_j(\alpha_j+1)(\alpha_i+\alpha_j)}{\alpha_0(\alpha_0+1)(\alpha_i+\alpha_j+2)}
+\frac{(1-\eta_{ijw})(\alpha_j+1)(\alpha_j+2)(\alpha_i+\alpha_j)}{\alpha_0(\alpha_0+1)(\alpha_i+\alpha_j+2)}\quad\\\label{prop4eq7}
\mathbb{E}\left[\theta_k^2| Y_{ij}^w=1,\balpha,\mu_w,\nu_w\right]&= \frac{\alpha_k}{\alpha_0}\frac{\alpha_k+1}{\alpha_0+1}\quad\text{ for }k\neq i, j.
\end{align}

In the next,  we will show \eqref{eq:moment_matching3w} and \eqref{eq:moment_matching4w} can be represented as the last two equations in \eqref{MM1w}. When $Y_{ij}^w=1$, the last two equations in \eqref{MM1w} become
\begin{eqnarray}
\label{MM1wtemp}
\left\{
\begin{array}{rcl}
\frac{\mu'_w}{\mu'_w+\nu'_w}&=&\eta_{ijw}\frac{\mu_w+1}{\mu_w+\nu_w+1}
+(1-\eta_{ijw})\frac{\mu_w }{\mu_w+\nu_w+1}\\
\frac{\mu'_w(\mu'_w+1)+\nu'_w(\nu'_w+1)}{(\mu'_w+\nu'_w)(\mu'_w+\nu'_w+1)}
&=&\eta_{ijk}\frac{(\mu_w+1)(\mu_w+2)}{(\mu_w+\nu_w+1)(\mu_w+\nu_w+2)}
+(1-\eta_{ijk})\frac{(\mu_w )(\mu_w+1)}{(\mu_w+\nu_w+1)(\mu_w+\nu_w+2)}\\
&&+\eta_{ijk}\frac{(\nu_w )(\nu_w+1)}{(\mu_w+\nu_w+1)(\mu_w+\nu_w+2)}
+(1-\eta_{ijk})\frac{(\nu_w+1)(\nu_w+2)}{(\mu_w+\nu_w+1)(\mu_w+\nu_w+2)}.
\end{array}
\right.
\end{eqnarray}
It is known that $\mathbb{E} \left[\rho_w|\rho_w\sim\B(\mu'_w,\nu'_w)\right]=\frac{\mu'_w}{\mu'_w+\nu'_w}$, $\mathbb{E} \left[\rho_w^2 |\rho_w\sim\B(\mu'_w,\nu'_w)\right]=\frac{\mu'_w(\mu'_w+1)}{(\mu'_w+\nu'_w)(\mu'_w+\nu'_w+1)}$ and $\mathbb{E} \left[(1-\rho_w)^2 |\rho_w\sim\B(\mu'_w,\nu'_w)\right]=\frac{\nu'_w(\nu'_w+1)}{(\mu'_w+\nu'_w)(\mu'_w+\nu'_w+1)}$, indicating that the left-hand sides of \eqref{eq:moment_matching3w} and \eqref{eq:moment_matching4w} match those of \eqref{MM1wtemp}.

To characterize the right-hand sides of \eqref{eq:moment_matching3w} and \eqref{eq:moment_matching4w}, we first derive from \eqref{prop4eq0} that
\begin{eqnarray}
\nonumber
&&\mathbb{E}\left[ \rho_w| Y_{ij}^w,\balpha,\bmu,\bnu\right]\\\nonumber
&=&\frac{\int_0^1\int_{\Delta}\left(\rho_w^2\frac{\theta_i}{\theta_i+\theta_j}+\rho_w(1-\rho_w)\frac{\theta_j}{\theta_i+\theta_j}\right)\frac{1}{\mathrm{B}(\balpha)\mathrm {B} (\mu_w,\nu_w )}\prod_{k=1}^K \theta_k^{\alpha_k - 1}\rho_w^{\mu_w -1}(1-\rho_w)^{\nu_w  -1}\mathrm{d}\btheta\mathrm{d}\rho_w}
{\frac{\mu_w}{\mu_w+\nu_w}\frac{\alpha_i}{\alpha_i+\alpha_j}+\frac{\nu_w}{\mu_w+\nu_w}\frac{\alpha_j}{\alpha_i+\alpha_j}}\\\nonumber
&=&\frac{\int_0^1\left(\frac{\alpha_i}{\alpha_i+\alpha_j}\rho_w^2+\frac{\alpha_j}{\alpha_i+\alpha_j}\rho_w(1-\rho_w)\right)\frac{1}{\mathrm{B} (\mu_w,\nu_w )}\rho_w^{\mu_w -1}(1-\rho_w)^{\nu_w  -1}\mathrm{d}\rho_w}
{\frac{\mu_w}{\mu_w+\nu_w}\frac{\alpha_i}{\alpha_i+\alpha_j}+\frac{\nu_w}{\mu_w+\nu_w}\frac{\alpha_j}{\alpha_i+\alpha_j}}\\\nonumber
&=&\frac{\frac{\mu_w}{\mu_w+\nu_w}\frac{\mu_w+1}{\mu_w+\nu_w+1}\frac{\alpha_i}{\alpha_i+\alpha_j}}
{\frac{\mu_w}{\mu_w+\nu_w}\frac{\alpha_i}{\alpha_i+\alpha_j}+\frac{\nu_w}{\mu_w+\nu_w}\frac{\alpha_j}{\alpha_i+\alpha_j}}
+\frac{\frac{\mu_w}{\mu_w+\nu_w}\frac{\nu_w}{\mu_w+\nu_w+1}\frac{\alpha_j}{\alpha_i+\alpha_j}}
{\frac{\mu_w}{\mu_w+\nu_w}\frac{\alpha_i}{\alpha_i+\alpha_j}+\frac{\nu_w}{\mu_w+\nu_w}\frac{\alpha_j}{\alpha_i+\alpha_j}}\\\label{prop4eq8}
&=&\eta_{ijw}\frac{\mu_w+1}{\mu_w+\nu_w+1}
+(1-\eta_{ijw})\frac{\mu_w }{\mu_w+\nu_w+1}.
\end{eqnarray}
Following a similar procedure, we can show
\begin{eqnarray}
\nonumber
&&\mathbb{E}[\rho_w^2+(1-\rho_w)^2| o_i \succ_w o_j,\theta \sim Dir(\alpha),\rho_w\sim Beta(\mu_w,\nu_w)]\\\nonumber
&=&\frac{\frac{\mu_w}{\mu_w+\nu_w}\frac{\mu_w+1}{\mu_w+\nu_w+1}\frac{\mu_w+2}{\mu_w+\nu_w+2}\frac{\alpha_i}{\alpha_i+\alpha_j}}
{\frac{\mu_w}{\mu_w+\nu_w}\frac{\alpha_i}{\alpha_i+\alpha_j}+\frac{\nu_w}{\mu_w+\nu_w}\frac{\alpha_j}{\alpha_i+\alpha_j}}
+\frac{\frac{\mu_w}{\mu_w+\nu_w}\frac{\mu_w+1}{\mu_w+\nu_w+1}\frac{\nu_w}{\mu_w+\nu_w+2}\frac{\alpha_j}{\alpha_i+\alpha_j}}
{\frac{\mu_w}{\mu_w+\nu_w}\frac{\alpha_i}{\alpha_i+\alpha_j}+\frac{\nu_w}{\mu_w+\nu_w}\frac{\alpha_j}{\alpha_i+\alpha_j}}\\\nonumber
&&+\frac{\frac{\nu_w}{\mu_w+\nu_w}\frac{\nu_w+1}{\mu_w+\nu_w+1}\frac{\mu_w}{\mu_w+\nu_w+2}\frac{\alpha_i}{\alpha_i+\alpha_j}}
{\frac{\mu_w}{\mu_w+\nu_w}\frac{\alpha_i}{\alpha_i+\alpha_j}+\frac{\nu_w}{\mu_w+\nu_w}\frac{\alpha_j}{\alpha_i+\alpha_j}}
+\frac{\frac{\nu_w}{\mu_w+\nu_w}\frac{\nu_w+1}{\mu_w+\nu_w+1}\frac{\nu_w+2}{\mu_w+\nu_w+2}\frac{\alpha_j}{\alpha_i+\alpha_j}}
{\frac{\mu_w}{\mu_w+\nu_w}\frac{\alpha_i}{\alpha_i+\alpha_j}+\frac{\nu_w}{\mu_w+\nu_w}\frac{\alpha_j}{\alpha_i+\alpha_j}}\\\label{prop4eq9}
&=&\eta_{ijw}\frac{(\mu_w+1)(\mu_w+2)}{(\mu_w+\nu_w+1)(\mu_w+\nu_w+2)}
+(1-\eta_{ijw})\frac{(\mu_w )(\mu_w+1)}{(\mu_w+\nu_w+1)(\mu_w+\nu_w+2)}\\\nonumber
&&+\eta_{ijw}\frac{(\nu_w )(\nu_w+1)}{(\mu_w+\nu_w+1)(\mu_w+\nu_w+2)}
+(1-\eta_{ijw})\frac{(\nu_w+1)(\nu_w+2)}{(\mu_w+\nu_w+1)(\mu_w+\nu_w+2)}.
\end{eqnarray}
Putting \eqref{prop4eq8} and \eqref{prop4eq9} together, we have shown that the right-hand sides of \eqref{MM1wtemp} are exactly the right-hand sides of \eqref{prop4eq1} and \eqref{prop4eq0}, which completes the proof.
\end{proof}

\vskip 0.2in
\bibliography{arref}

\begin{thebibliography}{64}
\providecommand{\natexlab}[1]{#1}
\providecommand{\url}[1]{\texttt{#1}}
\expandafter\ifx\csname urlstyle\endcsname\relax
  \providecommand{\doi}[1]{doi: #1}\else
  \providecommand{\doi}{doi: \begingroup \urlstyle{rm}\Url}\fi

\bibitem[Acharyya(2013)]{Acharyya:13}
Sreangsu Acharyya.
\newblock \emph{Learning to rank in supervised and unsupervised settings using
  convexity and monotonicity}.
\newblock PhD thesis, Electrical and Computer Engineering, The University of
  Texas at Austin, 2013.

\bibitem[Ailon(2012)]{Ailon:2012}
Nir Ailon.
\newblock An active learning algorithm for ranking from pairwise preferences
  with an almost optimal query complexity.
\newblock \emph{Journal of Machine Learning Research}, 13\penalty0
  (1):\penalty0 137--164, 2012.

\bibitem[Bachrach et~al.(2012)Bachrach, Minka, Guiver, and Graepel]{Thore:12}
Yoram Bachrach, Tom Minka, John Guiver, and Thore Graepel.
\newblock How to grade a test without knowing the answers - a {B}ayesian
  graphical model for adaptive crowdsourcing and aptitude testing.
\newblock In \emph{International Conference on Machine Learning (ICML)}, 2012.

\bibitem[Beal(2003)]{Beal:03}
M.~J. Beal.
\newblock \emph{Variational Algorithms for Approximate Bayesian Inference}.
\newblock PhD thesis, Gatsby Computational Neuroscience Unit, University
  College London, 2003.

\bibitem[Bradley and Terry(1952)]{Bradley:52}
R.~A. Bradley and M.~Terry.
\newblock Rank analysis of incomplete block designs: I. the method of paired
  comparisons.
\newblock \emph{Biometrika}, 39:\penalty0 324�345, 1952.

\bibitem[Braverman and Mossel(2008)]{Braverman:2008}
Mark Braverman and Elchanan Mossel.
\newblock Noisy sorting without resampling.
\newblock In \emph{ACM-SIAM Symposium on Discrete Algorithms (SODA)}, 2008.

\bibitem[Burges et~al.(2005)Burges, Shaked, Renshaw, Lazier, Deeds, Hamilton,
  and Hullender]{Burges:2005}
Chris Burges, Tal Shaked, Erin Renshaw, Ari Lazier, Matt Deeds, Nicole
  Hamilton, and Greg Hullender.
\newblock Learning to rank using gradient descent.
\newblock In \emph{International Conference on Machine Learning (ICML)}, 2005.

\bibitem[Cao et~al.(2006)Cao, Xu, Liu, Li, Huang, and Hon]{Cao:2006}
Yunbo Cao, Jun Xu, Tie-Yan Liu, Hang Li, Yalou Huang, and Hsiao-Wuen Hon.
\newblock Adapting ranking svm to document retrieval.
\newblock In \emph{Annual International ACM SIGIR Conference on Research and
  Development in Information Retrieval}, 2006.

\bibitem[Cao et~al.(2007)Cao, Qin, Liu, Tsai, and Li]{export:70428}
Zhe Cao, Tao Qin, Tie-Yan Liu, Ming-Feng Tsai, and Hang Li.
\newblock Learning to rank: from pairwise approach to listwise approach.
\newblock In \emph{International Conference on Machine Learning (ICML)}, 2007.

\bibitem[Chen et~al.(2013)Chen, Bennett, Collins-Thompson, and
  Horvitz.]{Chen:13}
Xi~Chen, Paul~N. Bennett, Kevyn Collins-Thompson, and Eric Horvitz.
\newblock Pairwise ranking aggregation in a crowdsourced setting.
\newblock In \emph{ACM International Conference on Web Search and Data Mining
  (WSDM)}, 2013.

\bibitem[Chen et~al.(2015)Chen, Lin, and Zhou]{Chen:15}
Xi~Chen, Qihang Lin, and Dengyong Zhou.
\newblock Statistical decision making for optimal budget allocation in crowd
  labelling.
\newblock \emph{Journal of Machine Learning Research}, 16:\penalty0 1--46,
  2015.

\bibitem[Collins-Thompson and Callan(2004)]{Kevyn:04}
K.~Collins-Thompson and J.~Callan.
\newblock A language modeling approach to predicting reading difficulty.
\newblock In \emph{HLT}, 2004.

\bibitem[Cooper et~al.(1992)Cooper, Gey, and Dabney]{Cooper:1992}
William~S. Cooper, Fredric~C. Gey, and Daniel~P. Dabney.
\newblock Probabilistic retrieval based on staged logistic regression.
\newblock In \emph{Annual International ACM SIGIR Conference on Research and
  Development in Information Retrieval}, 1992.

\bibitem[Crammer and Singer(2001)]{Crammer01prankingwith}
Koby Crammer and Yoram Singer.
\newblock Pranking with ranking.
\newblock In \emph{Advances in Neural Information Processing Systems (NIPS)},
  2001.

\bibitem[Dawid and Skene(1979)]{Dawid:79}
A.~P. Dawid and A.~M. Skene.
\newblock Maximum likelihood estimation of observer error-rates using the {EM}
  algorithm.
\newblock \emph{Journal of the Royal Statistical Society Series C},
  28:\penalty0 20--28, 1979.

\bibitem[Ertekin et~al.(2012)Ertekin, Hirsh, and Rudin]{Rudin:12}
Seyda Ertekin, Haym Hirsh, and Cynthia Rudin.
\newblock Wisely using a budget for crowdsourcing.
\newblock Technical report, MIT, 2012.

\bibitem[Fishburn(1973)]{Fis:73}
Peter~C. Fishburn.
\newblock Binary choice probabilities: on the varieties of stochastic
  transitivity.
\newblock \emph{Journal of Mathematical Psychology}, 10\penalty0 (4):\penalty0
  327 -- 352, 1973.

\bibitem[Frazier et~al.(2008)Frazier, Powell, and Dayanik]{Frazier:08}
P.~Frazier, W.~B. Powell, and S.~Dayanik.
\newblock A knowledge-gradient policy for sequential information collection.
\newblock \emph{SIAM Journal on Control and Optimization}, 47(5):\penalty0
  2410--2439, 2008.

\bibitem[Frazier(2009)]{Frazier:thesis}
Peter Frazier.
\newblock \emph{Knowledge-Gradient Methods for Statistical Learning}.
\newblock PhD thesis, Princeton University, 2009.

\bibitem[Freund et~al.(2003)Freund, Iyer, Schapire, and Singer]{Freund:2003}
Yoav Freund, Raj Iyer, Robert~E. Schapire, and Yoram Singer.
\newblock An efficient boosting algorithm for combining preferences.
\newblock \emph{Journal of Machine Learning Research}, 4\penalty0
  (11):\penalty0 933--969, 2003.

\bibitem[Gao and Zhou(2013)]{Gao:13}
C.~Gao and D.~Zhou.
\newblock Minimax optimal convergence rates for estimating ground truth from
  crowdsourced labels.
\newblock arXiv:1310.5764, 2013.

\bibitem[Gleich and Lim(2011)]{Gleich:2011}
David Gleich and Lek~heng Lim.
\newblock Rank aggregation via nuclear norm minimization.
\newblock In \emph{ACM SIGKDD International Conference on Knowledge Discovery
  and Data Mining}, 2011.

\bibitem[Gr{\"{o}}tschel et~al.(1984)Gr{\"{o}}tschel, J{\"{u}}nger, and
  Reinelt]{grotschel1984cpa}
M.~Gr{\"{o}}tschel, M.~J{\"{u}}nger, and G.~Reinelt.
\newblock {A cutting plane algorithm for the linear ordering problem}.
\newblock \emph{Operations Research}, 32\penalty0 (6):\penalty0 1195--1220,
  1984.

\bibitem[Gupta and Miescke(1996)]{Gupta:96}
S.~S. Gupta and K.~J. Miescke.
\newblock Bayesian look ahead one-stage sampling allocations for selection of
  the best population.
\newblock \emph{Journal of Statistical Planning and Inference}, 54\penalty0
  (2):\penalty0 229--244, 1996.

\bibitem[Herbrich et~al.(2007)Herbrich, Minka, and Graepel]{Ralf:07}
R.~Herbrich, T.~Minka, and T.~Graepel.
\newblock Trueskill ({TM}): {a} bayesian skill rating system.
\newblock In \emph{Advances in Neural Information Processing Systems (NIPS)},
  2007.

\bibitem[Ho et~al.(2013)Ho, Jabbari, and Vaughan]{Ho:13}
C.~Ho, S.~Jabbari, and J.~W. Vaughan.
\newblock Adaptive task assignment for crowdsourced classification.
\newblock In \emph{International Conference on Machine Learning (ICML)}, 2013.

\bibitem[Howe(2006)]{Howe:06}
Jeff Howe.
\newblock The rise of crowdsourcing.
\newblock \emph{Wired}, 2006.

\bibitem[Jamieson and Nowak(2011)]{Jamieson:11}
Kevin~G. Jamieson and Robert Nowak.
\newblock Active ranking using pairwise comparisons.
\newblock In \emph{Advances in Neural Information Processing Systems (NIPS)},
  2011.

\bibitem[Kamar et~al.(2012)Kamar, Hacker, and Horvitz]{Ece:12}
E.~Kamar, S.~Hacker, and E.~Horvitz.
\newblock Combing human and machine intelligence in large-scale crowdsourcing.
\newblock In \emph{International Conference on Autonomous Agents and Multiagent
  System}, 2012.

\bibitem[Karger et~al.(2013{\natexlab{a}})Karger, Oh, and Shah]{Oh:12}
D.~Karger, S.~Oh, and D.~Shah.
\newblock Budget-optimal task allocation for reliable crowdsourcing systems.
\newblock \emph{Operations Research}, 62(1):\penalty0 1--24,
  2013{\natexlab{a}}.

\bibitem[Karger et~al.(2013{\natexlab{b}})Karger, Oh, and Shah]{Karger:13}
David~R Karger, Sewoong Oh, and Devavrat Shah.
\newblock Efficient crowdsourcing for multi-class labeling.
\newblock \emph{ACM SIGMETRICS Performance Evaluation Review}, 41\penalty0
  (1):\penalty0 81--92, 2013{\natexlab{b}}.

\bibitem[Kendall(1938)]{Kendall:38}
M.~Kendall.
\newblock A new measure of rank correlation.
\newblock \emph{Biometrika}, 30:\penalty0 81--89, 1938.

\bibitem[Kuo et~al.(2009)Kuo, Cheng, and Wang]{Kuo:2009}
Jen-Wei Kuo, Pu-Jen Cheng, and Hsin-Min Wang.
\newblock Learning to rank from {B}ayesian decision inference.
\newblock In \emph{ACM Conference on Information and Knowledge Management},
  2009.

\bibitem[Li et~al.(2008)Li, Burges, and Wu]{export:68128}
P.~Li, C.J.C. Burges, and Q.~Wu.
\newblock Learning to rank using classification and gradient boosting.
\newblock In \emph{Advances in Neural Information Processing Systems (NIPS)},
  2008.

\bibitem[Liu et~al.(2012)Liu, Peng, and Ihler]{QiangLiu:12}
Q.~Liu, J.~Peng, and A.~Ihler.
\newblock Variational inference for crowdsourcing.
\newblock In \emph{Advances in Neural Information Processing Systems (NIPS)},
  2012.

\bibitem[Liu(2009)]{Liu:09}
T.Y. Liu.
\newblock Learning to rank for information retrieval.
\newblock \emph{Foundations and Trends in Information Retrieval}, 3:\penalty0
  225--331, 2009.

\bibitem[Lu and Boutilier(2014)]{JMLR:v15:lu14a}
Tyler Lu and Craig Boutilier.
\newblock Effective sampling and learning for mallows models with
  pairwise-preference data.
\newblock \emph{Journal of Machine Learning Research}, 15\penalty0
  (1):\penalty0 3783--3829, 2014.

\bibitem[Luce(1959)]{Luce:59}
R.D. Luce.
\newblock \emph{Individual choice behavior: a theoretical analysis}.
\newblock Wiley, 1959.

\bibitem[Mallows(1957)]{Mallows:57}
C.~L. Mallows.
\newblock Non-null ranking models.
\newblock \emph{Biometrika}, 44:\penalty0 114--130, 1957.

\bibitem[Mishra and Sikdar(2004)]{Mishra04LOP}
Sounaka Mishra and Kripasindhu Sikdar.
\newblock On approximability of linear ordering and related
  np-optimizationproblems ongraphs.
\newblock \emph{Discrete Applied Mathematics}, 136\penalty0 (2--3):\penalty0
  249--269, 2004.

\bibitem[Negahban et~al.(2012)Negahban, Oh, and Shah]{nos12}
Sahand Negahban, Sewoong Oh, and Devavrat Shah.
\newblock Rank centrality: ranking from pair-wise comparisons.
\newblock arXiv:1209.1688, 2012.

\bibitem[Paisley et~al.(2012)Paisley, Blei, and Jordan]{Paisley:12}
J.~Paisley, D.~Blei, and M.~Jordan.
\newblock Variational bayesian inference with stochastic search.
\newblock In \emph{International Conference on Machine Learning (ICML)}, 2012.

\bibitem[Pfeiffer et~al.(2012)Pfeiffer, Gao, Chen, Mao, and Rand]{Pfeiffer:12}
Thomas Pfeiffer, Xi~Alice Gao, Yiling Chen, Andrew Mao, and David~G. Rand.
\newblock Adaptive polling for information aggregation.
\newblock In \emph{AAAI Conference on Artificial Intelligence}, 2012.

\bibitem[Powell(2010)]{Powell:11a}
Warron~B. Powell.
\newblock \emph{The Knowledge Gradient for Optimal Learning}.
\newblock Wiley Encyclopedia for Operations Research and Management Science,
  2010.

\bibitem[Puterman(2005)]{Puterman:05}
M.~L. Puterman.
\newblock \emph{Markov Decision Processes: Discrete Stochastic Dynamic
  Programming}.
\newblock Wiley, 2005.

\bibitem[Qian et~al.(2015)Qian, Gao, and Jagadish]{Qian:2015}
Li~Qian, Jinyang Gao, and H.~V. Jagadish.
\newblock Learning user preferences by adaptive pairwise comparison.
\newblock \emph{Proceedings of the VLDB Endowment}, 8\penalty0 (11):\penalty0
  1322--1333, 2015.

\bibitem[Qin et~al.(2010)Qin, Geng, and Liu]{Qin:10}
T.~Qin, X.~Geng, and T.~Y. Liu.
\newblock A new probabilistic model for rank aggregation.
\newblock In \emph{Advances in Neural Information Processing Systems (NIPS)},
  2010.

\bibitem[Radinsky and Ailon(2011)]{Radinsky:2011}
Kira Radinsky and Nir Ailon.
\newblock Ranking from pairs and triplets: Information quality, evaluation
  methods and query complexity.
\newblock In \emph{ACM International Conference on Web Search and Data Mining},
  2011.

\bibitem[Rajkumar and Agarwal(2014)]{Rajkumar:2014}
Arun Rajkumar and Shivani Agarwal.
\newblock A statistical convergence perspective of algorithms for rank
  aggregation from pairwise data.
\newblock In \emph{International Conference on Machine Learning (ICML)}, 2014.

\bibitem[Raykar et~al.(2010)Raykar, Yu, Zhao, Valadez, Florin, Bogoni, and
  Moy]{Vikas:10}
V.~C. Raykar, S.~Yu, L.~H. Zhao, G.~H. Valadez, C.~Florin, L.~Bogoni, and
  L.~Moy.
\newblock Learning from crowds.
\newblock \emph{Journal of Machine Learning Research}, 11\penalty0
  (4):\penalty0 1297--1322, 2010.

\bibitem[Ryzhov et~al.(2012)Ryzhov, Powell, and Frazier]{Ryzhov:2012}
Ilya~O. Ryzhov, Warren~B. Powell, and Peter~I. Frazier.
\newblock The knowledge gradient algorithm for a general class of online
  learning problems.
\newblock \emph{Operations Research}, 60\penalty0 (1):\penalty0 180--195, 2012.

\bibitem[Shah et~al.(2016{\natexlab{a}})Shah, Balakrishnan, Bradley, Parekh,
  Ramchandran, and Wainwright]{Shah:15a}
Nihar~B. Shah, Sivaraman Balakrishnan, Joseph Bradley, Abhay Parekh, Kannan
  Ramchandran, and Martin~J. Wainwright.
\newblock Estimation from pairwise comparisons: Sharp minimax bounds with
  topology dependence.
\newblock \emph{Journal of Machine Learning Research}, 17, 2016{\natexlab{a}}.

\bibitem[Shah et~al.(2016{\natexlab{b}})Shah, Balakrishnan, Guntuboyina, and
  Wainwright]{Shah15b}
Nihar~B. Shah, Sivaraman Balakrishnan, Adityanand Guntuboyina, and Martin~J.
  Wainwright.
\newblock Stochastically transitive models for pairwise comparisons:
  Statistical and computational issues.
\newblock In \emph{International Conference on Machine Learning (ICML)},
  2016{\natexlab{b}}.

\bibitem[Taylor et~al.(2008)Taylor, Guiver, Robertson, and Minka]{export:63585}
Michael Taylor, John Guiver, Stephen Robertson, and Tom Minka.
\newblock Softrank: Optimising non-smooth rank metrics.
\newblock In \emph{ACM International Conference on Web Search and Data Mining
  (WSDM)}, 2008.

\bibitem[Thurstone(1927)]{Thurstone:27}
L.~L. Thurstone.
\newblock The method of paired comparisons for social values.
\newblock \emph{Journal of Abnormal and Social Psychology}, 21:\penalty0
  384--400, 1927.

\bibitem[Volkovs and Zemel(2014)]{JMLR:v15:volkovs14a}
Maksims~N. Volkovs and Richard~S. Zemel.
\newblock New learning methods for supervised and unsupervised preference
  aggregation.
\newblock \emph{Journal of Machine Learning Research}, 15\penalty0
  (1):\penalty0 1135--1176, 2014.

\bibitem[Wauthier et~al.(2013)Wauthier, Jordan, and Jojic]{Wauthier13}
Fabian Wauthier, Michael Jordan, and Nebojsa Jojic.
\newblock Efficient ranking from pairwise comparisons.
\newblock In \emph{International Conference on Machine Learning (ICML)}, 2013.

\bibitem[Welinder et~al.(2010)Welinder, Branson, Belongie, and
  Perona]{Peter:10}
P.~Welinder, S.~Branson, S.~Belongie, and P.~Perona.
\newblock The multidimensional wisdom of crowds.
\newblock In \emph{Advances in Neural Information Processing Systems (NIPS)},
  2010.

\bibitem[Whitehill et~al.(2009)Whitehill, Ruvolo, Wu, Bergsma, and
  Movellan]{Whitehill:09}
J.~Whitehill, P.~Ruvolo, T.~Wu, J.~Bergsma, and J.~R. Movellan.
\newblock Whose vote should count more: Optimal integration of labels from
  labelers of unknown expertise.
\newblock In \emph{Advances in Neural Information Processing Systems (NIPS)},
  2009.

\bibitem[Wu and Frazier(2016)]{wu2016parallel}
Jian Wu and Peter~I Frazier.
\newblock The parallel knowledge gradient method for batch bayesian
  optimization.
\newblock arXiv:1606.04414, 2016.

\bibitem[Xu and Li(2007)]{Xu:2007}
Jun Xu and Hang Li.
\newblock Ada{R}ank: A boosting algorithm for information retrieval.
\newblock In \emph{Annual International ACM SIGIR Conference on Research and
  Development in Information Retrieval}, 2007.

\bibitem[Yi et~al.(2013)Yi, Jin, Jain, and Jain]{Yi:13}
Jinfeng Yi, Rong Jin, Shaili Jain, and Anil~K. Jain.
\newblock Inferring users' preferences from crowdsourced pairwise comparisons:
  A matrix completion approach.
\newblock In \emph{Conference on Human Computation and Crowdsourcing (HCOMP)},
  2013.

\bibitem[Zhang et~al.(2014)Zhang, Chen, Zhou, and Jordan]{Zhang:14a}
Yuchen Zhang, Xi~Chen, Dengyong Zhou, and Michael~I. Jordan.
\newblock Spectral methods meet em: A provably optimal algorithm for
  crowdsourcing.
\newblock In \emph{Advances in Neural Information Processing Systems (NIPS)},
  2014.

\bibitem[Zheng et~al.(2008)Zheng, Zha, Zhang, Chapelle, Chen, and
  Sun]{Zheng:07}
Zhaohui Zheng, Hongyuan Zha, Tong Zhang, Olivier Chapelle, Keke Chen, and
  Gordon Sun.
\newblock A general boosting method and its application to learning ranking
  functions for web search.
\newblock In \emph{Advances in Neural Information Processing Systems (NIPS)}.
  2008.

\end{thebibliography}

\end{document}